\documentclass{article}
\usepackage[numbers,sort&compress]{natbib}
\usepackage[preprint]{neurips_2026}

\usepackage[utf8]{inputenc}
\usepackage[T1]{fontenc}
\usepackage{pdfrender}
\usepackage{url}
\usepackage{booktabs}
\usepackage{amsfonts}
\usepackage{nicefrac}
\usepackage{microtype}
\usepackage{makecell}
\usepackage{subfiles}
\usepackage{subcaption}
\usepackage{comment}
\usepackage{mathtools}
\usepackage{mdframed}
\usepackage{amsthm, amsfonts, amssymb, mathrsfs}
\usepackage{thmtools}
\usepackage{algorithm}
\usepackage{algpseudocode}
\usepackage{enumitem}
\usepackage{aliascnt}
\usepackage{etoc}
\usepackage{wrapfig}
\usepackage{caption}
\usepackage{tabularx}

\usepackage[dvipsnames,table]{xcolor}
\usepackage{tikz}
\usetikzlibrary{arrows.meta, positioning, shapes.geometric}

\usepackage[final,colorlinks,linktoc=all,pagebackref=true]{hyperref}
\usepackage[all]{hypcap}
\usepackage[nameinlink,capitalise, noabbrev]{cleveref}

\makeatletter
\@ifundefined{theHALG@line}
  {\newcommand{\theHALG@line}{\thealgorithm.\arabic{ALG@line}}}
  {\renewcommand{\theHALG@line}{\thealgorithm.\arabic{ALG@line}}}
\makeatother

\definecolor{darkred}{RGB}{100,0,0}
\definecolor{darkblue}{RGB}{0,0,100}
\definecolor{darkgreen}{RGB}{0,75,0}
\definecolor{remarkgreen}{RGB}{238,248,240}

\hypersetup{citecolor=darkgreen}
\hypersetup{linkcolor=darkgreen}
\hypersetup{urlcolor=darkgreen}

\declaretheoremstyle[
    spaceabove    = \parsep,
    spacebelow    = \parsep,
    bodyfont      = \normalfont\itshape,
]{theoremsty}

\declaretheoremstyle[
    spaceabove    = \parsep,
    spacebelow    = \parsep,
    bodyfont      = \normalfont,
]{normalsty}

\mdfdefinestyle{coloredstyle}{%
    leftline          = false,%
    rightline         = false,%
    topline           = false,%
    bottomline        = false,%
    backgroundcolor   = gray!10,%
    innertopmargin    = 1.5ex,%
    innerbottommargin = 0.7ex,%
    innerleftmargin   = 1.ex,%
    innerrightmargin  = 1.ex,%
    skipabove         = \topskip,%
    skipbelow         = \parskip%
}

\mdfdefinestyle{remarkstyle}{%
    leftline          = false,%
    rightline         = false,%
    topline           = false,%
    bottomline        = false,%
    backgroundcolor   = remarkgreen,%
    innertopmargin    = 1.5ex,%
    innerbottommargin = 0.7ex,%
    innerleftmargin   = 1.ex,%
    innerrightmargin  = 1.ex,%
    skipabove         = \topskip,%
    skipbelow         = \parskip%
}

\declaretheorem[name=Theorem, style=theoremsty, mdframed={style = coloredstyle}]{theorem}
\declaretheorem[name=Proposition, style=theoremsty, mdframed={style = coloredstyle}]{proposition}
\declaretheorem[name=Corollary, style=theoremsty, mdframed={style = coloredstyle}]{corollary}
\declaretheorem[name=Lemma, style=theoremsty, mdframed={style = coloredstyle}]{lemma}

\declaretheorem[name=Theorem, style=theoremsty, mdframed={style = coloredstyle}]{theorem*}
\declaretheorem[name=Proposition, style=theoremsty, mdframed={style = coloredstyle}]{proposition*}
\declaretheorem[name=Corollary, style=theoremsty, mdframed={style = coloredstyle}]{corollary*}
\declaretheorem[name=Lemma, style=theoremsty, mdframed={style = coloredstyle}]{lemma*}
\declaretheorem[name=Conjecture, style=theoremsty, mdframed={style = coloredstyle}]{conjecture*}

\declaretheorem[name=Remark, style=normalsty, mdframed={style = remarkstyle}]{remark}
\declaretheorem[name=Definition, style=normalsty, mdframed={style = coloredstyle}]{definition}
\declaretheorem[name=Assumption, style=normalsty, mdframed={style = coloredstyle}]{assumption}

\declaretheorem[name=Remark, style=normalsty, mdframed={style = remarkstyle}, numbered=no]{remark*}
\declaretheorem[name=Definition, style=normalsty, mdframed={style = coloredstyle}, numbered=no]{definition*}
\declaretheorem[name=Assumption, style=normalsty, mdframed={style = coloredstyle}, numbered=no]{assumption*}

\numberwithin{theorem}{section}
\numberwithin{proposition}{section}
\numberwithin{corollary}{section}
\numberwithin{lemma}{section}
\numberwithin{conjecture}{section}
\numberwithin{remark}{section}
\numberwithin{definition}{section}
\numberwithin{assumption}{section}
\numberwithin{example}{section}

\etocframedstyle[1]{\textbf{\textsc{Table of Contents}}}
\etocsettocdepth{3}

\usepackage{bbm}

\newcommand{\calA}{{\mathcal{A}}}
\newcommand{\calB}{{\mathcal{B}}}
\newcommand{\calC}{{\mathcal{C}}}
\newcommand{\calD}{{\mathcal{D}}}

\newcommand{\calH}{{\mathcal{H}}}

\newcommand{\calL}{{\mathcal{L}}}
\newcommand{\calM}{{\mathcal{M}}}

\newcommand{\calV}{{\mathcal{V}}}

\newcommand{\calY}{{\mathcal{Y}}}
\newcommand{\calZ}{{\mathcal{Z}}}

\newcommand{\bbE}{\mathbb{E}}

\newcommand{\bbP}{\mathbb{P}}

\newcommand{\lbra}[1]{\left[#1\right]}
\newcommand{\mbra}[1]{\left\{#1\right\}}
\newcommand{\sbra}[1]{\left(#1\right)}
\newcommand{\norm}[1]{\lVert#1\rVert}
\newcommand{\abs}[1]{\left| #1\right|}

\providecommand{\mask}{\texttt{[mask]}}
\providecommand{\base}{\pi_{\mathrm{ref}}}

\DeclareMathOperator*{\argmax}{argmax}
\DeclareMathOperator*{\argmin}{argmin}

\newcommand{\tabbest}[1]{\textpdfrender{TextRenderingMode=FillStroke,LineWidth=0.35pt}{#1}}

\definecolor{blendedblue}{rgb}{0.2,0.2,0.7}

\newif\ifshowbibliography
\showbibliographytrue

\title{VGB for Masked Diffusion Model: Efficient Test-time Scaling for Reward Satisfaction and Sample Editing}

\author{
Kijung Jeon\\
Georgia Tech\\
\texttt{kjeon@gatech.edu}
\And
Thuy-Duong Vuong$^{*}$\\
UCSD\\
\texttt{thvuong@ucsd.edu}
\And
Molei Tao$^{*}$\\
Georgia Tech\\
\texttt{mtao@gatech.edu}
}

\begin{document}

\maketitle
\begingroup
\renewcommand{\thefootnote}{\fnsymbol{footnote}}
\footnotetext[1]{Joint mentorship on this project.}
\endgroup
\begin{abstract}
    Inference-time scaling is a promising paradigm to improve generative models, especially when outputs must satisfy structural constraints or optimize downstream rewards. We consider Masked Diffusion Model (MDM) and introduce MDM-VGB, a discrete diffusion sampler that augments unmasking generation with theoretically principled reward-guided remasking. Inspired by the recent success of the classical Jerrum-Sinclair backtracking Markov chain \cite{sinclair1989approximate} in reward-tilted generation \cite{rohatgi2025taming}, 
    MDM-VGB extends the backtracking random walk from a fixed prefix tree to a masked-state graph,  allowing tokens to be unmasked and remasked at arbitrary positions. 
    The resulting sampler favors unmasking and remasking moves that lead to higher-value partial configurations, enabling both effective high-reward generation and efficient repair of low-reward samples. We prove that MDM-VGB is robust to process-verifier noise and achieves quadratic complexity, while popular test-time heuristics such as best-of-$N$ can incur exponential complexity due to error accumulation. Our theoretical findings are corroborated by strong empirical performance, particularly on popular constraint-satisfaction and scientific benchmarks such as Sudoku and QM9 \cite{ramakrishnan2014quantum}.\footnote{All code is provided in the following repository: \url{https://github.com/KraitGit/MDM-VGB}.}
\end{abstract}

\section{Introduction}
\label{sec:introduction}

\vspace{-8pt}
Many sequence-generation tasks require outputting configurations that satisfy global structural properties. Such properties are often easy to verify for complete sequences, but difficult to reliably extrapolate from partially generated configurations. Examples include mathematical proofs, computer programs, Sudoku puzzles, protein sequences for designing drug candidates: a completed configuration can often be accurately evaluated by a theorem prover, compiler, (Sudoku) constraint checker, or drug-property oracle, while evaluating whether a partial configuration can be extended to a valid full configuration is much harder.

Verifier-guided inference-time algorithms are a promising direction for improving structured generation \citep{brown2024large,wang2023selfconsistency,wang2024mathshepherd,rohatgi2025taming,wang2025remasking,Lee2025EffectiveTS,wang2025valueguided,misaki2026unmaskfork}.
In many realistic settings, one has access to a reliable, possibly perfect, \emph{outcome verifier} for complete configurations, together with weaker and noisier \emph{process} verifiers that estimate the \emph{potential value} of partial configurations \citep{yang2021fudge,lightman2024lets,wang2024mathshepherd}. Empirically, non-trivial gains have been shown even with simple methods that only use an outcome verifier. A canonical example is Best-of-$N$, which selects the highest-reward full configuration among $n$ sampled candidates \citep{brown2024large}. Access to a process verifier may unlock additional gains \citep{lightman2024lets,wang2024mathshepherd,wang2025valueguided};
however, a recent work by \citet{rohatgi2025taming} points out that these methods suffer from the curse of horizon
as even mild errors in process verifiers can lead to outcome-level errors that grow exponentially with sequence length.

They further observed that the classical Jerrum--Sinclair backtracking Markov chain \citep{sinclair1989approximate}
can avoid this issue, yielding a theoretically grounded inference-time algorithm with promising performance.  However, this chain is \emph{fixed-order}: it reveals tokens from left to right and can backtrack only by deleting the most recently generated token. 
This restriction limits its ability to revise earlier decisions, which is especially problematic in editing, and in general constraint-generation tasks, where early mistakes can often be fatal (see \cref{fig:prefix_vs_aoar_repair}).

\begin{figure}[t]
    \centering
    \begin{subfigure}[t]{0.49\linewidth}
        \centering
        \includegraphics[width=\linewidth]{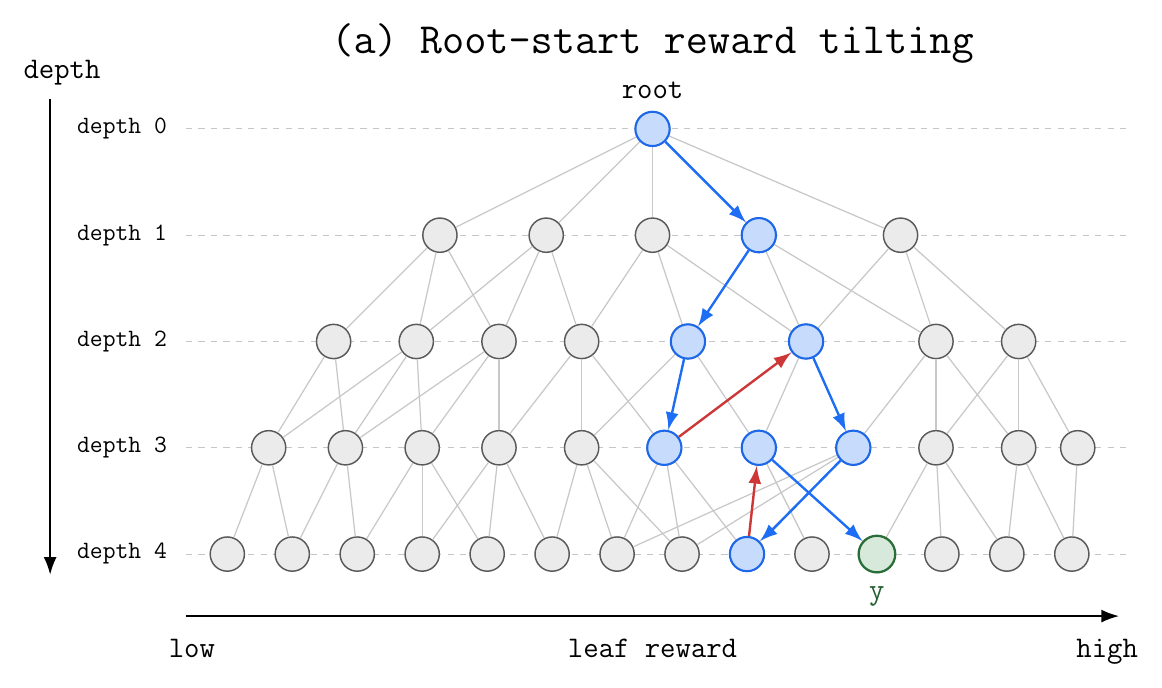}
        \caption{Generation (Root-start)}
        \label{fig:vgb_generating}
    \end{subfigure}
    \hfill
    \begin{subfigure}[t]{0.49\linewidth}
        \centering
        \includegraphics[width=\linewidth]{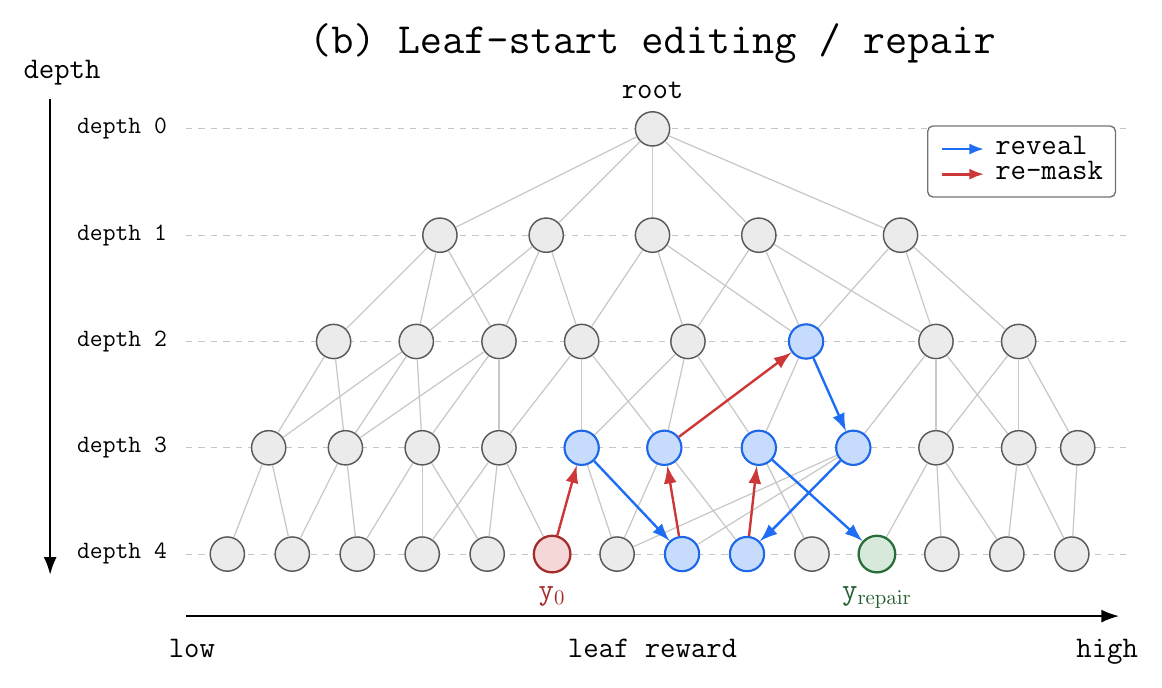}
        \caption{Editing (Leaf-start)}
        \label{fig:vgb_editing}
    \end{subfigure}
    \caption{Visualization of MDM-VGB for generation and editing. The depth of a state is the number of unmasked tokens, and same-depth states are ordered from left to right by estimated value. MDM-VGB is a random walk on the masked state graph: root-start generation begins at the fully masked state, while leaf-start editing begins at a fully revealed configuration to repair. Since value estimates are noisy, greedy unmasking can be suboptimal; re-masking can correct earlier mistakes.
    }
    \label{fig:vgb_generation_editing}
    \vspace{-10pt}
\end{figure}

In this work, we introduce \emph{Masked-Diffusion-Model Value-Guided Backtracking} (MDM-VGB), a verifier-guided discrete diffusion sampler with remasking. MDM-VGB generalizes the classical Jerrum-Sinclair backtracking Markov chain \citep{sinclair1989approximate} from fixed-order generation to any-order generation.
The Markov chain underlying MDM-VGB operates on partial configurations, i.e., masked states, where some coordinates are revealed, and the remaining coordinates are masked. Given an outcome reward on fully revealed configurations, we define the value of a partial configuration as the expected reward of compatible full configurations. 
At each step, the Markov chain either reveals masked token(s) or remasks already revealed token(s), with transition probabilities biased toward configurations with larger estimated value. This rule gives a principled randomized-greedy way to reveal and remask tokens: tokens whose removal increases the value of the partial configuration are treated as unreliable and are more likely to be revised; similarly, unmasking decisions that lead to higher-value configurations are prioritized (see \cref{def:main_geometric_balanced_mdm}).
 
This Markov chain supports both generation and editing. For generation, the chain starts at the fully masked root state (see \cref{fig:vgb_generating}); for editing, it starts from a low-reward fully revealed configuration (see \cref{fig:vgb_editing}). Unlike the fixed-order autoregressive backtracking from \cite{rohatgi2025taming},
MDM-VGB can unmask and remask tokens at arbitrary locations, allowing it to efficiently fix fatal early mistakes (see \cref{fig:prefix_vs_aoar_repair}); this flexibility yields significant performance gains in practice.

While MDM-VGB also supports flexible block updates in the sense that each transition may reveal or re-mask
an arbitrary block of coordinates, we will place
particular emphasis on the
singleton-block case as in the original masked diffusion model, which we call Any-Order-Auto-Regressive VGB (AOAR-VGB). This simplest
variant already yields strong empirical gains while admitting the cleanest theoretical guarantees. This Markov chain provably targets the reward-tilted target distribution in the asymptotic time limit. In particular, in hard-constraint settings with a perfect outcome verifier, where invalid full configurations are certifiably rewardless, the sampler is guaranteed to never produce an invalid full configuration. Moreover, as in the fixed-order case, the chain mixes rapidly, even when subjected to substantial multiplicative error in the intermediate value estimates provided by the process verifiers. In particular, we show that AOAR-VGB converges to the target reward-tilted distribution within a number of Markov steps\footnote{In our implementation, the cost per step matches a standard forward generation step, up to a small constant factor} that is quadratic in the sequence length $n$, whereas best-of-$N$ incurs an exponential cost.
In addition, motivated by non-reversible liftings of Markov chains
\citep{hayes2010liftings} and the momentum lift of the AR-VGB chain considered by
\citet{rohatgi2025taming}, we also study MDM-VGB-Momentum, the momentum lift of MDM-VGB. By augmenting each state with a direction variable, the momentum lift mitigates the diffusive and oscillatory behavior of MDM-VGB and improves first-leaf hitting time, a practical quantity of interest under finite-budget sampling.

Empirically, MDM-VGB and MDM-VGB-Momentum deliver consistent gains on scientific-design benchmarks and standard combinatorial constraint generation tasks such as QM9 and Sudoku.
For generation tasks, MDM-VGB-Momentum achieves higher reward at matched
inference cost than best-of-$N$, and reward-guided MDM  (MDM-VGR, see \cref{main:sec:preliminaries}), with especially impressive gains for harder tasks where unguided MDM has lower accuracy. For editing tasks, MDM-VGB and MDM-VGB-Momentum outperform fixed-order backtracking algorithms from \cite{sinclair1989approximate,rohatgi2025taming}: on Dyck grammar editing, they reach near-perfect accuracy using substantially fewer moves compared to AR-VGB and AR-VGB-Momentum. Moreover, MDM-VGB and MDM-VGB-Momentum improve over non-editing baselines in both output quality and cost across scientific-design benchmarks, e.g., DNA and Protein synthesis.

\vspace{-10pt}
\paragraph{Contributions.} 
\begin{itemize}[leftmargin=30pt]
\item 
We introduce MDM-VGB, a reward-guided masked discrete diffusion sampler with a \emph{principled} remasking strategy that supports both generation and editing, with especially strong gains in editing.
\item 
Theoretically, we prove that MDM-VGB efficiently targets the reward-tilted distribution over complete configurations, even in the presence of substantial process-verifier error.
\item
Across structured-generation benchmarks, MDM-VGB outperforms standard baselines under matched inference budgets. We also observe substantial gains for editing compared to the fixed-order backtracking baselines (AR-VGB and AR-VGB-Momentum) from \citep{rohatgi2025taming}.
\end{itemize}

\section{Other related works}
\vspace{-8pt}

Recent works, such as \citet{wang2025remasking,Lee2025EffectiveTS},
have considered alternative formulations of masked diffusion models (MDM) with remasking. Even
without an additional reward signal, remasking can empirically improve sample
quality by correcting errors introduced by MDM-style parallel decoding \citep{wang2025remasking}. However, it is less clear how to remask in a principled way with quantitatively significant benefits. In fact,
the Markov-chain formulation of
\citet{Lee2025EffectiveTS} also targets a reward-tilted stationary distribution, but unlike our method, it does not provide mixing-time or error-tolerance guarantees. Moreover, \citet{Lee2025EffectiveTS} remasks coordinates uniformly at random rather than by confidence, which reduces its ability to correct earlier mistakes. Remasking diffusion (ReMDM)
\citep{wang2025remasking} can use confidence scores to guide remasking decisions, but its updates follow a
factorized parallel decoding structure: all coordinates are resampled in
parallel from coordinate-wise conditionals. This can
be limiting for structured outputs with strong cross-coordinate dependencies.
For example, changing \(1+2=4\) to \(1+2=3\) gives a valid expression, whereas
simultaneously changing \(4\to3\) and \(2\to1\) does not. ReMDM does not target the reward-tilted distribution, and as with \citet{Lee2025EffectiveTS}, does not provide mixing-time or robustness guarantees.

Fixed-order backtracking (AR-VGB) \citep{sinclair1989approximate,rohatgi2025taming} comes with strong theoretical guarantees, as MDM-VGB does, but it does not allow earlier mistakes to be edited without undoing a long suffix. Likewise, sampling-and-search methods like Sequential Monte Carlo \citep{ou2025inference} and Monte-Carlo Tree-Search
\citep{ou2025inference,misaki2026unmaskfork}, as well as standard masked diffusion \citep{austin2021structured,sahoo2024simple,arriola2025block,lee2025test}, do not allow backtracking, and thus cannot repair earlier mistakes in the generative process. This is particularly problematic for structured languages, where an early error can make every subsequent extension invalid (see \cref{fig:prefix_vs_aoar_repair}).

\section{Preliminaries}

\label{main:sec:preliminaries}
\vspace{-10pt}

Let $\calY=\calV^n$ be the set of fully
revealed sequences of length $n$ over vocabulary $\calV.$ Let \(\calZ=(\calV\cup\{\mask\})^n\) be the set of masked states, where $\varnothing = \sbra{(\mathtt{[mask]},\ldots,\mathtt{[mask]})}$ denotes the fully masked state. 
For \(z\in\calZ\), let \(R(z)=\{i:z_i\neq\mask\}\) be the revealed set and \(k(z)=|R(z)|\) be its depth.
Let \(\calC(z)=\{y\in\calY: y_i=z_i \text{ for all } i\in R(z)\}\) denote the compatible completion set.

Throughout, we fix a conditioning context $x$. 
Given a reference distribution $\base(\cdot | x)$ over $\calY,$ we define the marginal at $z\in \calZ$ as $\base(z\mid x) := \sum_{y\in\calC(z) } \base(y\mid x).$ For \(B\subseteq[n]\setminus R(z)\) and \(a_B\in\calV^B\), let
\(z^{B\rightarrow a_B}\) be the state obtained by setting the tokens in \(B\) to $a_B.$ We define the conditional marginal at $B$ conditioned on $z$, i.e., $\base(Y_B=a_B\mid x,z)$, by: 
\begin{align*}
   \base(Y_B=a_B\mid x,z)
    &:= \frac{\base(z^{B\to a_B} \mid x)}{\base(z\mid x)} =
    \mathbb{P}_{\base}(Y_B=a_B\mid Y\in\calC(z)).
\end{align*}

\paragraph{Masked diffusion model (MDM).}
Following \cite{sahoo2024simple}, if a pretrained MDM is unavailable for a task, we train a discrete masked diffusion model $\base(\cdot | x)\equiv \pi_\theta$ over $\calY $ using a continuous masking process by minimizing the following objective function: \begin{equation*}
    \mathcal L_{\mathrm{MDM}}(\theta) =\mathbb E_{\substack{(x,y)\sim\mathcal D,  t\sim\mathrm{Unif}[0,1],  z_t\sim q_t(\cdot\mid y)}}\lbra{\frac{\alpha_t'}{1-\alpha_t}\sum_{i:\,z_{t,i}=\mathtt{[mask]}}\log \pi_\theta(Y_i=y_i\mid x,z_t)}.
\end{equation*}
where $\mathcal{D}$ is the training data, $\alpha_t \in [0,1]$ is a decreasing masking schedule, $\alpha'_t$ is its time derivative, and $z_t$ is sampled from the distribution $q_t(\cdot | y)$ over $\mathcal{Z}$ defined by:
\begin{equation*}
    q_t(z_t\mid y) = \prod_{i=1}^n \sbra{\alpha_t \mathbf 1\{z_{t,i}=y_i\} + (1-\alpha_t)\mathbf 1\{z_{t,i}=\mathtt{[mask]}\}}.
\end{equation*}

At inference time, MDM samples from $\base(\cdot |x)$ by starting from $\varnothing$ and iteratively revealing a new block of tokens according to the marginal of $\base(\cdot |x)$ conditioned on the revealed tokens. In other words, at a  current masked configuration
$z$,  the process reveals a new block $B \subset [n]\setminus R(z)$ by sampling from  $\base(Y_B=\cdot \mid x,z)$.

\paragraph{Reward-tilted sampling.}

Given a base reference model $\base(\cdot\mid x)$ and a nonnegative terminal score, i.e., reward function, $\tau(x, \cdot): \calY\to \mathbb{R}_{\geq 0}, $ 
define the
reward-tilted target distribution $ \pi^\star( \cdot\mid x) $ over $\calY$ where:
\begin{equation*}
    \pi^\star(y\mid x)
    :=
    \frac{\base(y\mid x)\tau(x,y)}
    {Z(x)},
    \qquad
    Z(x):=\sum_{y\in\calY}\base(y\mid x)\tau(x,y).
\end{equation*}
We also define the marginal of $\pi^\star$ at $z\in \calZ$ by $\pi^\star(z\mid x) := \sum_{y\in\calC(z) } \pi^\star(y\mid x).$ In particular, binary reward $\tau $ (i.e., $\tau(x,y) \in \{0,1\}, \ \forall y \in \calY$) corresponds to the hard-constraint case.

\textbf{Our goal} \emph{is to sample from $\pi^*$ using a given base model $\base$ and terminal score $\tau$ at inference time, without accessing or changing the base model's parameters, i.e., via test-time scaling rather than RL fine-tuning.}

\paragraph{Value of partial configurations.}
For a given terminal score $\tau$, the ideal value of a partial configuration or masked state $z\in (\calV\cup\{\mask\})^n$ is
\begin{equation*}\label{eq:ideal value masked state}
     V^\star(x,z)
    :=
    \mathbb E_{Y\sim \pi_{\mathrm{ref}}(\cdot \mid x)}\!\left[\tau(x,Y)\mid Y\in\calC(z)\right].
\end{equation*}
In practice, $V^*$ is not available, and therefore we train a process verifier $\widehat V$ to approximate $V^*$ by fitting $V_\theta$ with \textit{mean squared error}, as in \citet{rohatgi2025taming}:
\begin{equation*}
     \calL_{\mathrm{reg}}(\theta) := \bbE_{x,z,y}\lbra{\sbra{V_\theta(x,z)-\tau(x,y)}^2},
\end{equation*}
where $z$ is a partial configuration sampled from the reference model $\base$ and $y$ is a completion of $z$. After training, we use $\widehat V=V_\theta$. In some settings, we instead use a lightweight heuristic to reduce computational cost. For example, in constraint-satisfaction tasks such as Sudoku, where the terminal reward is the indicator that a full configuration satisfies a certain constraint, we take $\widehat V(x,z)$ to be the indicator that the partial assignment $z$ has not yet violated the constraint. We note that this approximation can be far from the true value $V^\star(x,z)$, since there exist partial assignments that do not violate the constraint but cannot be extended to valid full configurations. Nevertheless, such heuristic verifiers can still perform strongly on some constraint-satisfaction tasks; see \cref{main:sec:Experiments}.

\paragraph{Value-guided Masked Diffusion Model (MDM-VGR).} Given a base MDM $\base$ and  the value function $ V^\star$ with respect to a terminal reward $\tau$, MDM-VGR samples from $\pi^\star(\cdot |x)$ by starting from $\varnothing$ and iteratively revealing a new block of tokens according to the reward-tilted marginal conditioned on the current state, i.e., $\pi^\star(Y_B=a_B\mid x,z)
    =
    \base(Y_B=a_B\mid x,z)\cdot
    \frac{V^\star(x,z^{B\rightarrow a_B})}{V^\star(x,z)}.$ In practice, MDM-VGR uses a  process verifier $\hat{V}$ that approximates $ V^\star.$ 

\paragraph{Masked states graph.}

Consider a masked state $z\in \calZ$. For \(B\subseteq[n]\setminus R(z)\) and \(a_B\in\calV^B\), 
we call \(z^{B\rightarrow a_B}\) a \emph{forward child} of $z$; let $C(z)$ be the set of children of $z$. For
\(B\subseteq R(z)\), write \(z^{-B}\) for the state obtained by re-masking
\(B\), which we call a \emph{backward parent} of $z;$ we note that a state can have multiple parents and let $P(z)$ be the set of parents of $z$.
The parent-child relation between states defines the \emph{masked states graph}, an undirected graph whose vertices are masked states and edges are $ \{z,w\}$, where $w$ is a backward parent of $z.$ We write $z\sim w$ and say $z$ and $w$ are neighbors when  $\{z,w\}$ is an edge in the graph. We call the fully masked state $\varnothing$ the root, and the set of full configurations $\calY$ the leaves of this graph.

\begin{figure}[t]
    \centering
    \includegraphics[width=0.91\linewidth]{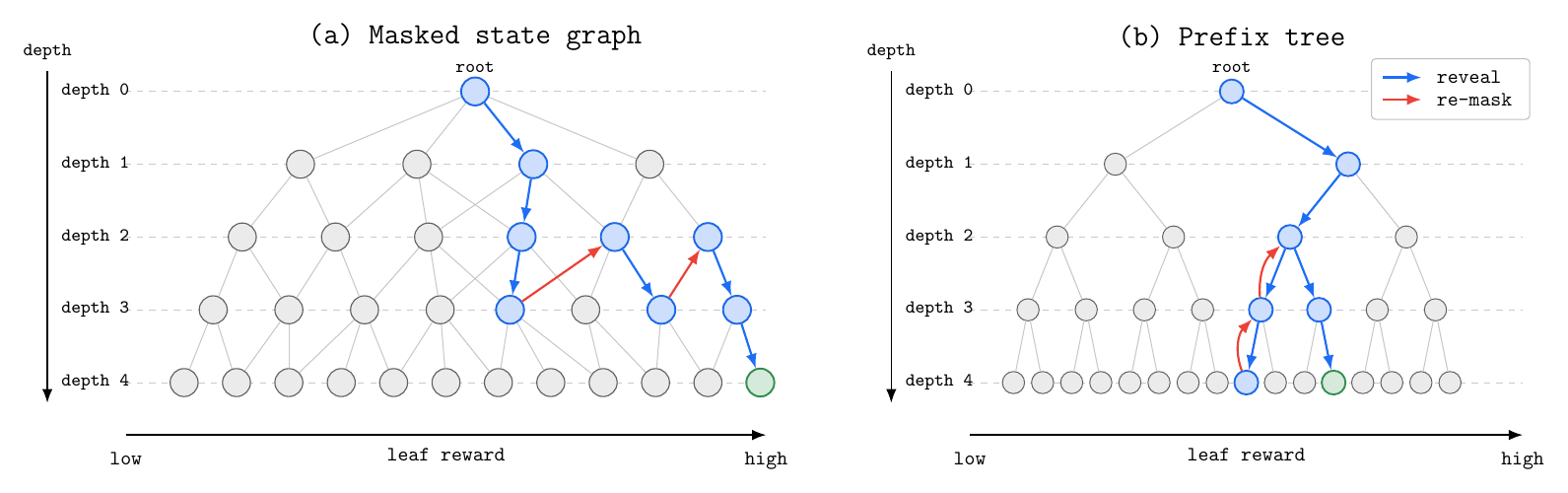}
    \caption{
    \textbf{Masked-state graph vs. prefix tree.}
    MDM-VGB operates on the any-order masked-state graph, which allows arbitrary-coordinate re-masking. AR-VGB operates on a fixed-order prefix tree, where backtracking only removes the latest revealed token.
    }
    \label{fig:masked_state_vs_prefix}
    \vspace{-5pt}
\end{figure}

\paragraph{Prefix tree.} The AR-VGB Markov chain from \cite{rohatgi2025taming,sinclair1989approximate} is a random walk on the following prefix tree graph. Fix an ordering $\sigma$ of $\{1, 2,\cdots, n\}.$ Let $\mathcal Z_{\sigma}$ denote the set of masked states in which the revealed coordinates form a prefix in this order. The prefix tree on $\mathcal Z_{\sigma}$ is the graph where $\{x,y\}$ is an edge whenever $x$ is the back-parent of $y$ obtained by remasking its last unmasked coordinate in the order $\sigma$. See \cref{fig:masked_state_vs_prefix}  for an illustration.

\section{Methods and Theoretical Guarantees}
\label{main:sec:main_results}

\subsection{Formulation of MDM-VGB} 
To address the goal of \cref{main:sec:preliminaries}, we now describe the MDM-VGB algorithm. At a high level, MDM-VGB is a random walk on the weighted masked state graph with appropriately chosen edge weights.
The edge weights depend geometrically on the estimated values of endpoint states, so that remasking or unmasking moves that lead to higher-value states are prioritized. We also rescale each edge weight by a scalar depending on the depths of
its endpoint states. This depth-dependent normalization is designed to balance forward moves that unmask new coordinates and backward moves that remask coordinates at each state and is crucial for ensuring the chain's fast convergence.

The MDM-VGB Markov chain is defined by either of the following equivalent definitions. 
\begin{definition}
\label{def:bg_aoar_edge_weights main}
Fix $\lambda\ge0$ and depth coefficients $s_{k,r}:=\binom{n-r}{k}^{-1}$, which satisfy $\binom{n-k}{r}s_{k,r} =    \binom{k}{r}s_{k-r,r}$.
The MDM-VGB Markov chain is the random walk on the weighted masked state graph, where the edge weights are as follows: for neighboring states $u \sim v$, let $p$ denote the lower-depth endpoint and $c$ denote the higher-depth endpoint among $u,v$. Let the symmetric edge weight between $u,v$ be 
\begin{equation*}
    f(u,v)=f(v,u):=\begin{cases}
    s_{k(p), k(c)-k(p)} \base(c\mid x) \widehat V(x,c)\widehat V(x,p)^\lambda  & \text{ if } c \notin\calY \\s_{k(p), k(c)-k(p)}\base(c\mid x)\widehat V(x,c) & \text{ if } c \in\calY  \end{cases}.
\end{equation*}

\end{definition}

Equivalently, the same chain can be specified through its local forward and backward move weights. In both views, \(s_{k,r}\) is the depth-dependent normalization factor that balances forward moves revealing \(r\) new coordinates with backward moves re-masking \(r\) revealed coordinates.

\begin{definition}\label{def:main_geometric_balanced_mdm}
For a state $z$ at depth $k,$ a child \(c=z^{B\rightarrow a_B}\in C(z)\) where \(B \subseteq [n] \setminus R(z)\) with $|B|= r,$ and a backward parent $p = z^{-B'}\in P(z)$ where $B'\subseteq R(z)$ and $|B'|=r',$ set:
\begin{equation*}
\begin{aligned}
    w_{\rm fwd}(z\to c)
    &:=
    \begin{cases}
    s_{k,r}\widehat V(x,z)^\lambda
    \widehat V(x,c)\base(Y_B=a_B\mid x,z), & c\notin\calY,\\
    s_{k,r}\base(Y_B=a_B\mid x,z)\widehat{V}(x,c), & c\in\calY,
    \end{cases}\\
    w_{\rm bwd}(z\to p)
    &:=
    \begin{cases}
    s_{k-r',r'}\widehat V(x,z)\widehat V(x,p)^\lambda, & z\notin\calY,\\
    s_{k-r',r'}\widehat{V}(x,z), & z\in\calY,
    \end{cases}
\end{aligned}
\end{equation*}
Let $
    W(z):= \sum_{c\in C(z)} w_{\rm fwd}(z\to c) + \sum_{p\in P(z)} w_{\rm bwd} (z\to p) $  be the normalization factor at $z$. The MDM-VGB Markov chain's transition probabilities at $z$ are given by:
\begin{align*}
  \forall c\in C(z): \mathbb{P}[z \to c] = \frac{ w_{\rm fwd}(z\to c)}{W(z)} \text{ and  } \forall p\in P(z): \quad \mathbb{P}[z \to p] = \frac{ w_{\rm bwd}(z\to p)}{W(z)}
\end{align*}
\end{definition}

The first definition is more helpful for proving the theoretical properties of the Markov chain, while the second is more helpful for implementation.
\begin{remark}[Role of $\lambda$] \label{remark:parameter lambda}
The parameter \(\lambda\ge 0\) controls how strongly re-masking decisions depend on the value of the resulting parent state: larger \(\lambda\) makes the chain more likely to erase a block that yields a higher-value parent state. Our ablation study found that a moderate $\lambda$ yields the best performance, as the chain can effectively edit problematic tokens while still retaining sufficient flexibility to explore alternative edits (see \cref{sec:exp_ablation}).
\end{remark}

In the appendix, we further distinguish between the case $\lambda = 0,$ which we call \emph{Balanced MDM-VGB} (see \cref{app:subsec:balanced_aoar_vgb}), and the general case $\lambda \geq 0,$ which we call \emph{Geometric Balanced MDM-VGB} (see \cref{app:subsec:balanced_geometric_aoar_vgb,app:subsec:balanced_geometric_mdm_vgb}). The $\lambda =0$ chain is closer to existing backtracking walks \cite{sinclair1989approximate,rohatgi2025taming}, but it selects a uniformly random block for re-masking and thus is less effective at editing problematic tokens. In \cref{app:subsec:primitive_aoar_vgb}, we also show that the direct generalization of AR-VGB to the diffusion setting, which corresponds to setting $\lambda=0$ and having no depth-normalization factor $s_{k,r}$, has exponentially slow convergence.

\begin{remark}[Implementation details: shortlisting rule]
To improve efficiency, instead of exploring the entire neighborhood of a masked state, we follow the practical shortlisting convention of \citet{rohatgi2025taming}: each MDM-VGB (or MDM-VGB-Momentum) step considers at most \(L_f\)
choices of blocks to unmask, \(L_b\) choices of blocks to re-mask, and \(K\) token vocabulary candidates for each coordinate to unmask. Unless otherwise stated, we use the
singleton block size \(|B|=1\). 
\end{remark}

\vspace{-8pt}
\subsection{Asymptotic behavior: properties of the stationary distribution}

\vspace{-5pt}
We now state the assumptions on the estimated masked-state values used in our theoretical analysis. Following
\citet{rohatgi2025taming}, we assume that the process verifier
approximates the true expected value of every masked state within
a multiplicative factor \(\kappa\). We also assume that the estimated values are
uniformly bounded between \(m\) and \(M\). This boundedness condition is mild in
practice, since these estimated values are typically clipped during implementation.

\vspace{-6pt}
\begin{assumption}
\label{ass:main_geometric_aoar_regular}
Assume terminal anchoring at leaves ($\widehat{V}(x,y) =\tau(x,y), \ \forall y\in \calY$) and 
the following
\emph{$\kappa$-accurate verifier} condition on every reachable non-leaf state
$z$:
\begin{equation*}
    \kappa^{-1}V^\star(x,z)
    \le
    \widehat V(x,z)
    \le
    \kappa V^\star(x,z)
\end{equation*}
Define the positive target-support component by \(\calZ_+(x):=\{z\in\calZ: \pi^\star(z\mid x)>0\}\), and assume additionally that every non-leaf $z\in\calZ_+(x)$ satisfies $0<m\le \widehat V(x,z)\le M<\infty$.
\end{assumption}
It is easy to see that MDM-VGB is an irreducible Markov chain. Thus, its lazified version\footnotemark{}
is ergodic and converges to a unique stationary distribution $\mu.$
Next, we show that the stationary distribution $\mu$, conditioned on the leaves $\calY$, is exactly the target distribution $\pi^\star.$

\begin{theorem}
\label{thm:main_exact_leaf_law}
Assume
terminal anchoring
\(\widehat V(x,y)=\tau(x,y)\) for all \(y\in\calY\). For any
\(\lambda\ge0\), the stationary law \(\mu\) of the lazy MDM-VGB chain satisfies
\[
    \mu(\cdot\mid\calY)=\pi^\star(\cdot\mid x).
\]
In particular, in hard-constraint settings with \(\tau\in\{0,1\}\), $\mu(\cdot\mid\calY)$ assigns zero probability to invalid configurations.
\end{theorem}
\footnotetext{For a Markov chain with transition matrix $P$ and a parameter $\gamma\geq 0$, its $\gamma$-held lazy version has transition matrix $P^{\rm lazy, \gamma} =\frac{\gamma}{1+\gamma}I + \frac{1}{1+\gamma}P.$ Each step of $P^{\rm lazy, \gamma} $ stays in-place with probability $\frac{\gamma}{1+\gamma},$ and makes moves according to $P$ with probability $\frac{1}{1+\gamma}.$ When $\gamma =1$, we refer to $P^{\rm lazy, \gamma}$ as the lazy version of $P$. For our theoretical analysis, it is convenient to work with the lazy MDM-VGB.}
Our analysis focuses on the simplest case, MDM-VGB with singleton-block updates, i.e., AOAR-VGB. The next theorem lower bounds the probability mass of the full configurations under the stationary distribution.

\begin{theorem}\label{thm:leaf mass probability}
    Under \cref{ass:main_geometric_aoar_regular}, the stationary distribution $\mu$ of the lazy AOAR-VGB chain places substantial mass on the set of full configurations $\calY$: for some constant $C>0,$
\[
    \mu(\calY)
    =
   \frac{1}{ C M^\lambda\kappa n}.
\]
\end{theorem}

\paragraph{Mixing time guarantee.}

We study AOAR-VGB's speed of convergence to the stationary distribution, commonly known as the mixing time.
The following theorem lower bounds the spectral gap of the AOAR-VGB Markov chain (see \cite{levin2017markov} for the connection between spectral gap and mixing time).

\begin{theorem}
\label{thm:main_aoar_sampler_guarantee}
Under \cref{ass:main_geometric_aoar_regular}, let \(P_{\mathrm{AOAR}}\) denote the transition matrix of the lazy AOAR-VGB Markov chain. Define the \emph{spectral gap} \citep{levin2017markov} by
\begin{equation*}
    \gamma(P_{\mathrm{AOAR}})
    :=
    1-\max\{|\lambda|:\lambda\in\Lambda(P_{\mathrm{AOAR}}),\ \lambda\ne1\},
\end{equation*}
where \(\Lambda(P_{\mathrm{AOAR}})\) denotes the set of eigenvalues of \(P_{\mathrm{AOAR}}\). Then
\begin{equation*}
    \gamma(P_{\mathrm{AOAR}})
    =
    \Omega \left(
        \left(\frac{m}{M}\right)^{2\lambda}
        \frac{1}{\kappa^4 n^2}
    \right).
\end{equation*}

\end{theorem} 
By standard Markov chain theory, this together with \cref{thm:leaf mass probability} gives an upper bound on the mixing time to the target reward-tilted distribution $\pi^\star(\cdot |x)$ over $\calY$, starting from any initial state.

\begin{corollary}\label{cor:mixing time main}
Assume \cref{ass:main_geometric_aoar_regular}.
For any target accuracy level \(\delta > 0\), the distribution $Z_T$ obtained after $T$ steps of the lazy AOAR-VGB chain initialized at a masked state $z$ satisfies  $
    d_{\rm TV}
    \left(
        \mathcal L(Z_T\mid Z_T\in\calY),
        \pi^\star(\cdot\mid x)
    \right)
    \le
    \delta  $ and $\mathbb{P}[Z_T \in \calY] \geq  \Omega\left(\frac{1}{ M^\lambda\kappa n} \right),$
    when
\[
    T
    \geq
    \Theta\!\left(
        \left(\frac{M}{m}\right)^{2\lambda}
        \kappa^4 n^2
        \log\left(\frac{M^\lambda\kappa n}{\delta \mu(z)}
   \right) \right)\footnote{where $\mu(z)$ is the probability mass of $z$ under the stationary distribution $\mu$ of the lazy chain AOAR-VGB.}
\]
In particular, when $ z$ is the root $\varnothing,$ it suffices to take $ T = \Theta\!\left(
        \left(\frac{M}{m}\right)^{2\lambda}
        \kappa^4 n^2
        \log\left(\frac{(M/m)^\lambda\kappa n}{\delta }
   \right) \right)$
\end{corollary}

\vspace{-5pt}
\begin{remark}
    The bound above controls the target law conditioned on reaching a leaf, but the the stationary distribution $\mu$ of the current chain's mass on the leaves, i.e. $\mu(\calY),$  is only of order \(1/n\). Thus, obtaining a leaf with constant probability would require running the chain \(O(n)\) times. By using the modification from \cite{hayes2010liftings,rohatgi2025taming}, i.e. increasing the holding probabilities at the leaves, we can instead make $\mu(\calY)$ at least a constant, and obtain a quadratic-time sampler for $\pi^\star(\cdot|x)$. This modification is mainly a technical step to improve the theoretical analysis; in practice, the sampler simply stops after seeing a few leaves, so we keep the simpler construction of the Markov chain.
\end{remark}

\paragraph{Proof sketch.} We discuss the analysis of AOAR-VGB. The proof is by relating the properties of AOAR-VGB to those of AR-VGB.
The key observation is that the weighted masked-state graph defining AOAR-VGB can be decomposed as a sum of weighted prefix-tree graphs that define AR-VGB. Consequently, the properties of AOAR-VGB can be compared with those of the corresponding AR-VGB walks. The stationary distribution probability at a state $z$ is proportional to the total weight of the edges incident to $z.$ First, the stationary distribution of AOAR-VGB is the average of the stationary distributions of the AR-VGB walks, and thus induces the same target distribution $\pi^\star$ on the leaves as the AR-VGB walks. Second, the weight on the leaves under the stationary distribution of AOAR-VGB can be lower bounded by those of the AR-VGB chains. Third, the conductance of the masked-state graph, which corresponds to the relaxation time and mixing time of AOAR-VGB, can be lower bounded by the conductance of the prefix-tree graphs, which in turn can be lower bounded via the existing analysis in \cite{sinclair1989approximate,rohatgi2025taming}. For details, see \cref{thm:canonical_geometric_balanced_mixing}.

\subsection{Formulation of MDM-VGB-Momentum}
We introduce MDM-VGB-Momentum, a
momentum lift of MDM-VGB in the spirit of \cite{hayes2010liftings} that targets the same distribution over the full configurations, while reducing oscillatory behaviors. The lift augments each masked state \(z\) with a direction variable, yielding
two lifted states \((z,\downarrow)\) and \((z,\uparrow)\). In reveal mode
\((z,\downarrow)\), the chain can only reveal tokens, while in re-mask mode
\((z,\uparrow)\), it can only re-mask tokens. At each state, the chain switches between these
two modes with a probability governed
by a parameter \(\chi\) chosen to reduce reveal/re-mask oscillations. See \cref{fig:main_geometric_mdm_vgb} for an illustration.

For AR-VGB, momentum has been shown
empirically to improve performance \citep{rohatgi2025taming}; theoretical work by  \citep{hayes2010liftings}
 also shows mixing-time improvements in the low-error regime, where the approximation parameter
\(\kappa\) is sufficiently close to \(1\). On the theory side, we prove that at a sufficiently low error--when $\kappa= 1+ O(1/n)$--MDM-VGB-Momentum reaches a leaf in expected $O(n)$ time, achieving a quadratic speedup over the non-momentum version, which requires $\Omega(n^2)$ time\footnote{We suppress dependency on other parameters $M,m,\lambda.$}.  We expect analogous mixing-time improvements in this low error regime, but leave a full mixing time analysis to future work. In our experiments, MDM-VGB-Momentum tends to perform better on generation tasks,
whereas the non-momentum version performs slightly better on editing tasks.

\begin{definition}[MDM-VGB-Momentum]
\label{def:main_flow_cancelled_momentum}
For a masked state $z$, a child $c\in C(z)$, and a parent $p\in P(z)$, let $w_{\rm fwd} (z\to c)$ and $ w_{\rm bwd}(z\to p)$ be as in \cref{def:main_geometric_balanced_mdm}. Let
\begin{align*}
    F(z)&:=\sum_{c\in C(z)}w_{\rm fwd} (z\to c),
    \quad
    B(z):=\sum_{p\in P(z)}w_{\rm bwd}(z\to p),\\
    M(z)&:=\min\mbra{F(z),B(z)},\quad W(z): = F(z) + B(z).
\end{align*}
Fix a parameter $\chi\in [0,1]$. The MDM-VGB-Momentum with parameter $\chi$ is a Markov chain
on the lifted state space $\widetilde\calZ=\calZ\times\{\textcolor{blue!70!black}{\downarrow},\textcolor{red!70!black}{\uparrow}\}$\footnote{Blue $\textcolor{blue!70!black}{\downarrow}$ denotes forward/reveal momentum, and red $\textcolor{red!70!black}{\uparrow}$ denotes backward/re-mask momentum.}. 

At $(z,\textcolor{blue!70!black}{\downarrow})\in \widetilde\calZ,$ the Markov chain can either transition to $ (c, \textcolor{blue!70!black}{\downarrow}) $ for $c\in C(z)$, or $ (z,\textcolor{red!70!black}{\uparrow}),$ or stays at $(z,\textcolor{blue!70!black}{\downarrow}).$ The transition probabilities are given by:
{\small
\begin{align*}
 \mathbb{P}[(z,\textcolor{blue!70!black}{\downarrow})\to (c, \textcolor{blue!70!black}{\downarrow}) ] = \frac{w_{\rm fwd} (z\to c)}{W(z)}, \, \mathbb{P}[(z,\textcolor{blue!70!black}{\downarrow})\to (z, \textcolor{blue!70!black}{\downarrow}) ] = \frac{\chi M(z)}{W(z)},\, \mathbb{P}[(z,\textcolor{blue!70!black}{\downarrow})\to (z, \textcolor{red!70!black}{\uparrow}) ] =\frac{B(z) - \chi M(z)}{W(z)}
\end{align*}
}
At $(z,\textcolor{red!70!black}{\uparrow})\in \widetilde\calZ$, the Markov chain can either transition to $ (p, \textcolor{red!70!black}{\uparrow}) $ for $p\in P(z)$, or $ (z,\textcolor{blue!70!black}{\downarrow})$ or stays at $(z,\textcolor{red!70!black}{\uparrow}).$ The transition probabilities are given by:
{\small
\begin{align*}
\mathbb{P}[(z,\textcolor{red!70!black}{\uparrow})\to (p, \textcolor{red!70!black}{\uparrow}) ] = \frac{w_{\rm bwd} (z\to p)}{W(z)} , \, \mathbb{P}[(z,\textcolor{red!70!black}{\uparrow})\to (z, \textcolor{red!70!black}{\uparrow}) ] = \frac{\chi M(z)}{W(z)},\, \mathbb{P}[(z,\textcolor{red!70!black}{\uparrow})\to (z, \textcolor{blue!70!black}{\downarrow}) ] = \frac{F(z)-\chi M(z)}{W(z)}
\end{align*}
}
\end{definition}

 The following theorem characterizes asymptotic behavior of AOAR-VGB-Momentum.
\begin{theorem}[Stationary distribution of AOAR-VGB-Momentum]
\label{thm:main_momentum_stationarity}
AOAR-VGB-Momentum with parameter $\chi\in[0,1]$ as defined in \cref{def:main_flow_cancelled_momentum} is an ergodic Markov chain, with stationary distribution $\widetilde\mu $ defined by 
\begin{equation*}
    \widetilde\mu(z,\downarrow)
    =
    \widetilde\mu(z,\uparrow)
    =
    \frac12\mu(z)
\end{equation*}
where $\mu$ is the stationary distribution of the AOAR-VGB chain. Hence the projection $\mathrm{proj}_{\#}\widetilde\mu$ of $\tilde{\mu}$ on $\calZ$  satisfies $\mathrm{proj}_{\#}\widetilde\mu=\mu$. In particular, if $ \mu(\cdot\mid\calY) = \pi^\star(\cdot\mid x)$ then
\begin{equation*}
    \mathrm{proj}_{\#}\widetilde\mu(\cdot\mid\calY)
    =
    \pi^\star(\cdot\mid x).
\end{equation*}
\end{theorem}
The following theorem shows that in the low error regime $\kappa = 1+O(1/n)$, AOAR-VGB-Momentum takes only $O(n)$ steps to hit a leaf, whereas AOAR-VGB requires $\Theta(n^2)$ steps.
\begin{theorem}
\label{thm:main_momentum_hitting}
Consider the AOAR-VGB-Momentum with parameter $\chi$ as defined in \cref{def:main_flow_cancelled_momentum}.
Let
\begin{equation*}
    p_\downarrow(z)=\frac{F(z)}{F(z)+B(z)},
    \qquad
    p_\uparrow(z)=\frac{B(z)}{F(z)+B(z)},
    \qquad
    \Delta:=\sup_{z \in \tilde{\calZ}_+^\circ(x)} \abs{p_\downarrow(z)-p_\uparrow(z)},
\end{equation*}
where the supremum is over reachable interior states $ \tilde{Z}_+^\circ (x) := \mbra{z\in \calZ_+(x) \mid 0<k(z)<n} \times \mbra{\downarrow, \uparrow}.$
Suppose $\Delta=O(1/n)$ and $1-\chi=O(1/n)$. Let $T_n$ and $T_n^{\rm Mo}$ be the time of first hitting the set of leaves $\calY$ for AOAR-VGB and AOAR-VGB-Momentum initialized at $ \varnothing$ and $(\varnothing,\downarrow)$ respectively, then
\begin{equation*}
    \mathbb E[T_n]=\Theta(n^2),
    \qquad
    \mathbb E[T_n^{\rm Mo}]=O(n),
\end{equation*}
In particular, under the multiplicative value-error bound $\kappa^{-1}V^\star\le\widehat V\le\kappa V^\star$, the condition $\Delta=O(1/n)$ is satisfied when $\kappa=1+O(1/n)$.
\end{theorem}

\begin{figure*}[t!]
    \centering
    \includegraphics[width=0.90\textwidth,trim={0 0 0 48pt},clip]{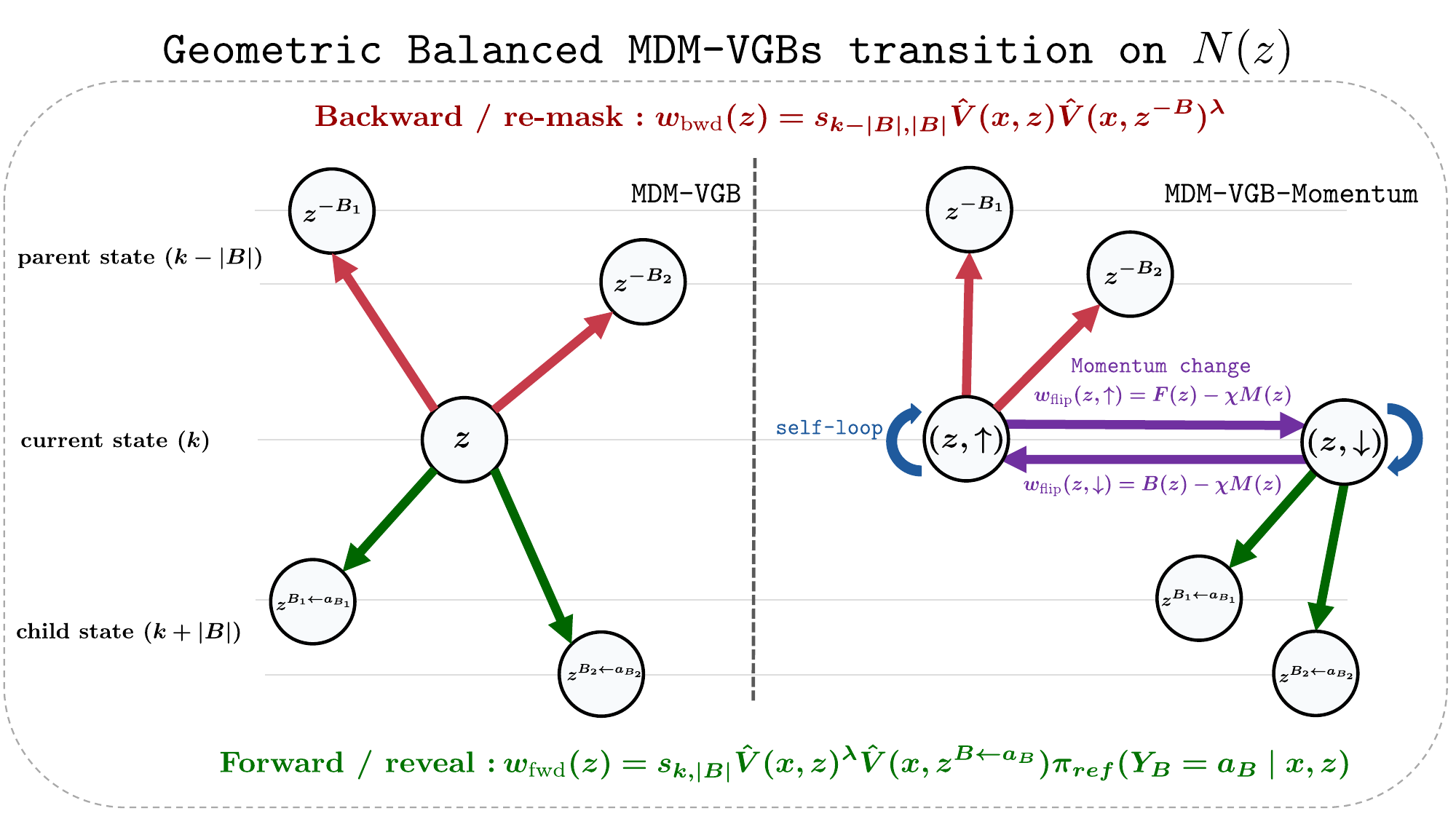}
    \caption{MDM-VGB and MDM-VGB-Momentum. MDM-VGB alternates value-guided reveal and re-mask moves, while the momentum variant augments the state with a direction variable to reduce immediate reversals.
    }    \label{fig:main_geometric_mdm_vgb}
    \vspace{-15pt}
\end{figure*}

\vspace{-10pt}
\section{Experiments}
\label{main:sec:Experiments}

We evaluate our methods \textsc{MDM-VGB}  and \textsc{MDM-VGB-Momentum} on reward-guided generation and editing tasks across several domains.

\medskip
\noindent\textbf{Task descriptions.} We briefly summarize each task, its terminal score \(R\), the base model, and the learned verifier used by our method. When applicable, the Pass@95 terminal reward is defined from the terminal score \(R\) as described in the Evaluation metrics paragraph.
\par\nobreak\smallskip

\textbf{Molecular design} aims to design drug-like molecules using the QM9 benchmark
\citep{ramakrishnan2014quantum}. For the base model, we train a 92.4M-parameter QM9 MDLM model with the MDLM objective using the UDLM-QM9 architecture.  We use the quantitative estimate of drug-likeness (QED) \citep{bickerton2012quantifying} as the terminal score \(R\). The terminal reward is the Pass@95 indicator defined from QED as described in the Evaluation metrics paragraph. To approximate the values of partial configurations, 
we train a process verifier using a token Transformer over SMILES tokens with $3.24$M parameters.

\textbf{DNA enhancer design} evaluates regulatory-sequence generation. We use
D3LM \citep{yang2026d3lmdiscretednadiffusion}, a DNA diffusion language model initialized from Nucleotide Transformer \citep{dallatorre2025nucleotide}, as the base model, and DeepSTARR developmental enhancer activity \citep{dealmeida2022deepstarr} as the terminal score \(R\). The terminal reward is the Pass@95 indicator defined from DeepSTARR developmental activity. The goal is to generate DNA sequences with high
predicted developmental enhancer activity. We learn the values of partial configurations using a small Transformer over D3LM tokens. We test multiple verifier model sizes (see \cref{sec:exp_ablation}) and report results for the verifier with $0.97$M parameters.

\textbf{Protein motif scaffolding} evaluates whether generated protein
sequences can scaffold a prescribed structural motif. We use EvoDiff OADM as
the base model \citep{alamdari2023protein}, and fold generated
sequences with OmegaFold \citep{wu2022omegafold}. For terminal scoring, we measure the motif \(C_\alpha\) root-mean-square deviation (RMSD) against the reference structure \citep{MAIOROV1994625}. For evaluation, we report motif RMSD and Success@1\AA{}, the percentage of generated configurations whose motif RMSD is below \(1.0\text{\AA}\).

We learn the values of partial configurations using a small Transformer over amino-acid tokens with $1.94$M parameters. To avoid training the verifier on a sparse Success@1\AA{} signal, we use a smoothed RMSD-derived target for verifier training.

\textbf{Sudoku} is a typical constraint-satisfaction generation task. The goal is to
generate a complete \(9\times9\) grid filled with digits \(1,\ldots,9\) so that every row,
column, and \(3\times3\) subgrid contains each digit exactly once. The terminal reward is the binary indicator that all constraints are satisfied. For partial configurations, we use a heuristic process verifier that returns the indicator that no constraint has yet been violated. For the base model, we use DiT-MDM, i.e., a masked discrete diffusion model
\citep{sahoo2024simple} with a Diffusion Transformer backbone
\citep{peebles2023scalable}.

\textbf{Letter avoidance} is a synthetic constrained natural-language generation task:
the model must generate a one-sentence story without using the letter
``e'', a task also considered in AR-VGB \citep{rohatgi2025taming}. We use Qwen3-0.6B-diffusion-mdlm-v0.1 as the base model
\cite{qwen3_06b_diffusion_mdlm}. The terminal reward is the binary indicator that the sequence does not contain
the forbidden character. For partial configurations, we use a heuristic process verifier that returns this same indicator.

\textbf{Dyck grammar} is a formal language over bracket symbols
\(\{\texttt{(},\texttt{)},\texttt{[},\texttt{]}\}\), commonly used as a sandbox test (see e.g. \cite{rohatgi2025taming}). The terminal score is the binary indicator that the (full) configuration is a valid well-balanced expression: every opening
bracket must be closed by the matching bracket type in the correct nested
order. For example, \(\texttt{()}\), \(\texttt{[]}\), and \(\texttt{[()]}\) are
valid, whereas \(\texttt{(}\), \(\texttt{]}\), \(\texttt{[(]}\),
\(\texttt{(()}\), and \(\texttt{[)}\) are not. We focus on the editing task, which highlights the benefit
of arbitrary-coordinate repair, since due to long-range dependencies, an early mistake can invalidate an otherwise plausible subsequent generation. For the base model, we use a BERT-style masked language model with $12.9$M parameters. For the process verifier, we train a small token-state Transformer with $3.24$M parameters.

\vspace{-25pt}
\paragraph{Evaluation metrics.}
For a given task, all test-time inference methods use the same reference model \(\base\) and terminal reward \(\tau\).
We compare \textsc{MDM-VGB} and \textsc{MDM-VGB-Momentum} against the baselines \textsc{Base}, \textsc{BoN}, \textsc{MDM-VGR}, \textsc{AR-VGB}, and \textsc{AR-VGB-Momentum}. \textsc{Base} is the unguided MDM
sampler. \textsc{BoN} (best-of-$N$) generates \(N\) full configurations using unguided MDM and returns the one
with the highest terminal reward
\citep{brown2024large,huang2025bestofn}. To match the complexity of \textsc{BoN}, we allow up to \(n\times N\) MDM-VGB steps, and upon reaching this budget, perform forward-only sampling using
\(\base\). For editing, all methods edit the same pool of full
configurations of lower reward.
See \cref{app:subsec:baseline_algorithms} for further details on the baselines.

For tasks with hard constraints, i.e., binary terminal score/reward, we report the success rate (percentage of configurations satisfying the constraint), and, when appropriate, a secondary quality metric: for Sudoku, we report the number of violated constraints and for Letter, we report the overall quality score (Ovl.) of the generated sentence provided by Qwen3-32B model \citep{qwen2025qwen3}; see the Letter prompt details in \cref{app:task_descriptions}.

For tasks with a real-valued terminal score (QM9, DNA, Protein), we report a calibrated metric that measures how often generated samples reach the high-quality region of the base model's distribution. Given an
independent \textsc{Base} calibration set and task-specific terminal score \(R(y)\), let
\(q_\alpha^{\rm Base}(R)\) be the \(\alpha\)-quantile of the
distribution over terminal scores produced by the base model. We define
\begin{equation*}
    \mathrm{Pass@\alpha}(y)
    =
    \begin{cases}
        \mathbf 1\{R(y)\ge q_{\alpha}^{\rm Base}(R)\},
        & \text{if higher values are better},\\
        \mathbf 1\{R(y)\le q_{1-\alpha}^{\rm Base}(R)\},
        & \text{if lower values are better}.
    \end{cases}
\end{equation*}
Across QM9 and DNA we use \(\alpha=0.95\): QM9 reports Pass@95 based on QED and DNA reports Pass@95 based on developmental activity. Protein is evaluated with Success@1\AA{} and motif RMSD.

We report additional metrics as appropriate: for QM9, the average QED of valid generated molecules and validity (percentage of valid molecules); for DNA, the average developmental activity score (Dev); and for Protein, mean motif RMSD and Success@1\AA{}, the percentage of generated configurations whose RMSD is below \(1.0\text{\AA}\). For editing, we report the reward of the edited configuration and the total computational cost of making the edits. For tasks with learned verifiers, we also report adjusted NFE, which accounts for the total cost of verifier and base-model calls. Further details about the experiments are given in \cref{app:exp_details}.

\begin{figure*}[t]
    \centering
    \includegraphics[width=\textwidth]{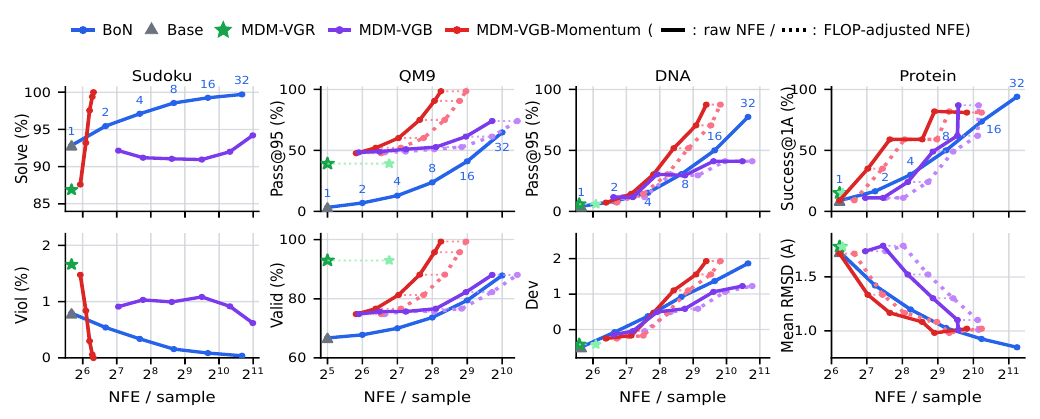}
    \caption{
    Generation quality--cost frontiers for Sudoku, QM9, DNA, and Protein.
    Solid curves count raw reference-model forward calls, and dotted curves show
    FLOP-adjusted NFE including learned-verifier overhead. Numeric labels on the
    \textsc{BoN} curve denote the number of independent rollouts \(N\).}
    \label{fig:generation_quality_cost}
    \vspace{-5pt}
\end{figure*}

\begin{table*}[t]
    \centering
    \ttfamily
    \setlength{\tabcolsep}{4.0pt}
    \renewcommand{\arraystretch}{1.08}
    \resizebox{\textwidth}{!}{
    \begin{tabular}{lcccccccccc}
    \toprule
    &
    \multicolumn{4}{c}{Heuristic verifier}
    &
    \multicolumn{6}{c}{Learned verifier} \\
    \cmidrule(lr){2-5}
    \cmidrule(lr){6-11}
    Method
    & \multicolumn{2}{c}{\cellcolor{green!8}Letter}
    & \multicolumn{2}{c}{\cellcolor{green!8}Sudoku}
    & \multicolumn{2}{c}{\cellcolor{blue!6}QM9}
    & \multicolumn{2}{c}{\cellcolor{blue!6}DNA}
    & \multicolumn{2}{c}{\cellcolor{blue!6}Protein} \\
    \cmidrule(lr){2-3}
    \cmidrule(lr){4-5}
    \cmidrule(lr){6-7}
    \cmidrule(lr){8-9}
    \cmidrule(lr){10-11}
    & Avoid \(\uparrow\)
    & Ovl.\(\uparrow\)
    & Solve \(\uparrow\)
    & Viol.\(\downarrow\)
    & Pass@95 \(\uparrow\)
    & Valid \(\uparrow\)
    & Pass@95 \(\uparrow\)
    & Dev.\(\uparrow\)
    & Success@1\AA{} \(\uparrow\)
    & RMSD \(\downarrow\) \\
    \midrule
    Base
        & 0.0 & --
        & 92.9 & 0.791
        & 3.2 & 66.8
        & 0.041 & -0.50
        & 8.7 & 1.732 \\
    BoN
        & \underline{0.1} & \tabbest{7.00}
        & \underline{99.3} & \underline{0.087}
        & 41.0 & 79.5
        & \underline{0.500} & \underline{1.37}
        & \underline{73.7} & \tabbest{0.926} \\
    MDM-VGR
        & \tabbest{100.0} & 5.72
        & 86.9 & 1.659
        & 39.1 & \underline{92.9}
        & 0.058 & -0.42
        & 15.0 & 1.780 \\
    \rowcolor{black!6}
    MDM-VGB
        & \tabbest{100.0} & \underline{5.86}
        & 92.0 & 0.917
        & \underline{61.3} & 82.2
        & 0.411 & 1.06
        & 62.0 & 1.101 \\
    \rowcolor{black!6}
    MDM-VGB-Momentum
        & \tabbest{100.0} & 5.72
        & \tabbest{100.0} & \tabbest{0.000}
        & \tabbest{90.6} & \tabbest{95.7}
        & \tabbest{0.705} & \tabbest{1.55}
        & \tabbest{82.0} & \underline{0.982} \\
    \bottomrule
    \end{tabular}
    }
    \caption{
    Generation results at a representative budget \(N=16\). 
    Letter and Sudoku use heuristic verifiers; QM9, DNA, and Protein use learned verifiers. 
    Columns report primary and secondary metrics for each task. \textbf{Bold}
    and \underline{underline} mark best and second-best.
    }
    \vspace{-20pt}
    \label{tab:generation_matched_compute}
\end{table*}

\subsection{Generation}
\label{sec:exp_generation}
\begin{wrapfigure}[10]{r}{0.35\linewidth}
    \vspace{-40pt}
    \centering
    \captionsetup{belowskip=15pt}
    \includegraphics[width=\linewidth]{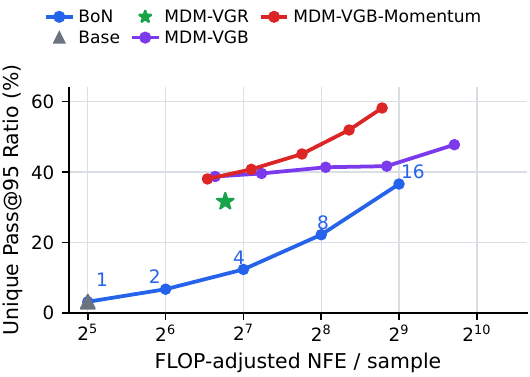}
    \caption{QM9 unique Pass@95: fraction of distinct generated molecules satisfying Pass@95 versus per-sample compute.}
    \label{fig:qm9_unique_pass_rate}
\end{wrapfigure}
\cref{fig:generation_quality_cost,tab:generation_matched_compute} evaluate MDM-VGB and MDM-VGB-Momentum as generative mechanisms under matched compute. 
Overall, the main trend is that MDM-VGB and MDM-VGB-Momentum
improve the quality--cost frontier, especially on tasks with learned verifiers. On QM9 and DNA, MDM-VGB-Momentum reaches high Pass@95 at
substantially smaller NFE than \textsc{BoN}; this trend largely holds even after FLOP-adjusting for verifier calls. 
Protein has a more competitive \textsc{BoN} frontier at large rollout budgets, but MDM-VGB-Momentum still gives the best success rate (Success@1\AA{}) under matched budget (see~\cref{tab:generation_matched_compute}).

For Sudoku, a popular constraint-satisfaction benchmark, MDM-VGB-Momentum reaches perfect accuracy with much lower compute than the baselines, whereas MDM-VGB performs slightly worse.
For QM9, \cref{fig:qm9_unique_pass_rate} plots the fraction of generated samples that are both unique and satisfy Pass@95 versus compute. The MDM-VGB variants still achieve a higher fraction of high-reward molecules at similar FLOP-adjusted NFE, suggesting that their gains are not simply due to repeatedly sampling the same few high-reward molecules.

\subsection{Editing}
\label{sec:exp_editing}
Next, we evaluate \textsc{MDM-VGB} and \textsc{MDM-VGB-Momentum} as editing procedures. 
We initialize the Markov chain at a completed low-reward output
\(y_0\in\calY\) sampled from the reference model. 
All editing methods start with the same $y_0 \in\calY$. \textsc{NoEdit} returns \(y_0\), \textsc{BoN} discards
\(y_0\) and samples fresh completions from \(\base\), and root-start \textsc{MDM-VGB}
variants restart from the fully masked configuration under the same maximum budget.
Leaf-start \textsc{MDM-VGB} variants instead re-mask selected coordinates of \(y_0\) and return the best full-configuration observed within the allowed edit budget.

\begin{figure*}[t]
    \centering
    \includegraphics[width=0.8\textwidth]{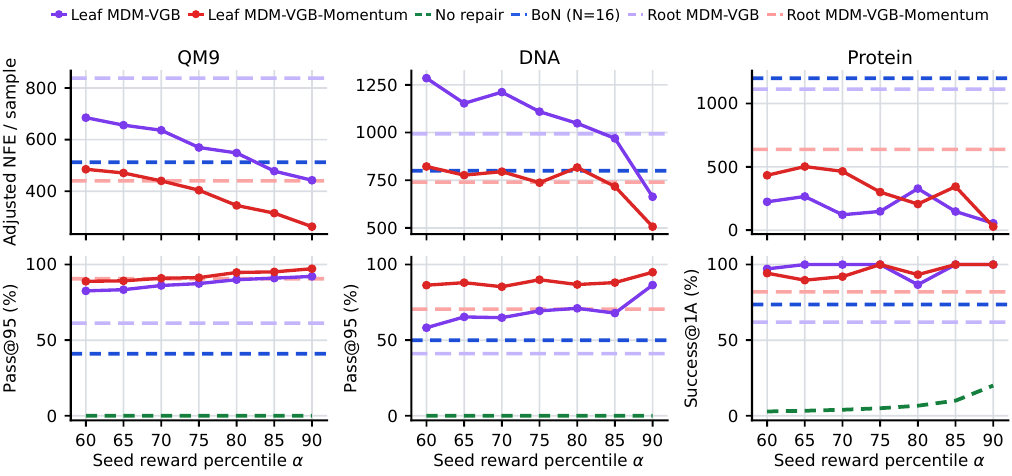}
    \caption{
    Editing on QM9, DNA, and Protein. Initial configurations are grouped by seed reward
    percentile \(\alpha\), i.e., base samples that fail Pass@95 but pass
    Pass@\(\alpha\). Larger \(\alpha\) indicates higher-reward initial outputs.
    Top: adjusted NFE per sample \((\downarrow)\). Bottom: edited quality
    \((\uparrow)\). Dashed lines show the reward of the initial configuration (no-edit) and of baselines that discard the initial configurations and generate from scratch: BoN, root-start MDM-VGB, and root-start MDM-VGB-Momentum.
    }
    \label{fig:leaf_start_repair}
\end{figure*}

\begin{table*}[t]
    \centering
    \scriptsize\ttfamily
    \setlength{\tabcolsep}{2.5pt}
    \renewcommand{\arraystretch}{1.08}
    \resizebox{\textwidth}{!}{
    \begin{tabular}{llccccccc}
        \toprule
        Task & Metric & NoEdit
        & \multicolumn{3}{c}{\cellcolor{red!6}Generation (Root start)}
        & \multicolumn{2}{c}{\cellcolor{green!8}Editing (Leaf start)}
        & Gain \\
        \cmidrule(lr){4-6}
        \cmidrule(lr){7-8}
        & & & BoN & MDM-VGB & MDM-VGB-Momentum & MDM-VGB & MDM-VGB-Momentum & \\
        \midrule
        QM9
        & Pass@95
        & 0.0
        & 41.0
        & 61.3
        & \underline{90.6}
        & \cellcolor{black!6}87.5
        & \cellcolor{black!6}\tabbest{92.5}
        & 1.9 pp \\
        DNA
        & Pass@95
        & 0.0
        & 50.0
        & 41.1
        & \underline{70.5}
        & \cellcolor{black!6}69.0
        & \cellcolor{black!6}\tabbest{88.5}
        & 18.0 pp \\
        Protein
        & Success@1\AA{}
        & 7.4
        & 73.7
        & 62.0
        & 82.0
        & \cellcolor{black!6}\tabbest{97.7}
        & \cellcolor{black!6}\underline{95.6}
        & 15.7 pp \\
        \bottomrule
    \end{tabular}
    }
    \caption{
        Comparison of generation and editing methods on QM9, DNA, and
        protein-design tasks. Editing methods are initialized from the same
        low-reward configurations, whereas generation methods discard these
        initial configurations and generate new samples from scratch.
        Leaf-start entries are averaged over seed-leaf reward percentile groups \(\alpha\in\{60,65,\ldots,90\}\).
        BoN uses \(N=16\) independent rollouts, and the MDM-VGB variants use up to the corresponding matched budget. \textnormal{Gain} denotes the absolute
        percentage-point improvement of the best editing method over the best
        non-editing baseline.
    }
    \label{tab:editing_repair}
    \vspace{-10pt}
\end{table*}

\cref{fig:leaf_start_repair,tab:editing_repair} show that using \textsc{MDM-VGB} and \textsc{MDM-VGB-Momentum} as editing procedures improves over non-editing procedures that generate samples from scratch. Across the three
scientific tasks, \textsc{MDM-VGB} and \textsc{MDM-VGB-Momentum} as editing procedures match or
improve output quality relative to the best non-editing baseline, with gains of
\(1.9\), \(18.0\), and \(15.7\) percentage points on QM9, DNA, and Protein,
respectively. The required adjusted NFE generally decreases as the initial
reward percentile \(\alpha\) increases, indicating that the editing methods make fewer edits when starting closer to the targeted high-reward region.

\begin{figure}[t]
\centering

\begin{minipage}[t]{0.49\linewidth}
\vspace{0pt}
\centering
\resizebox{\linewidth}{!}{
\begin{tikzpicture}[
    token/.style={draw, rounded corners=1.2pt, minimum width=0.44cm, minimum height=0.28cm, inner sep=1pt, font=\scriptsize\ttfamily},
    bad/.style={token, fill=red!16, draw=red!65!black},
    erase/.style={token, fill=gray!18, draw=gray!60!black},
    keep/.style={token, fill=green!12, draw=green!45!black},
    arrow/.style={-{Latex[length=1.7mm]}, thick}
]
\node[font=\scriptsize\ttfamily\bfseries, text=blue!70!black, anchor=east] at (-0.28,0.8) {AR};
\node[keep] at (0.00,0.8) {(}; \node[keep] at (0.48,0.8) {)}; \node[bad] at (0.96,0.8) {)}; \node[keep] at (1.44,0.8) {)}; \node[keep] at (1.92,0.8) {[}; \node[keep] at (2.40,0.8) {]};
\draw[arrow] (2.76,0.8) -- (3.10,0.8);
\node[keep] at (3.42,0.8) {(}; \node[keep] at (3.90,0.8) {)}; \node[erase] at (4.38,0.8) {M}; \node[erase] at (4.86,0.8) {M}; \node[erase] at (5.34,0.8) {M}; \node[erase] at (5.82,0.8) {M};
\node[font=\scriptsize\ttfamily, align=center] at (4.86,1.22) {re-mask suffix};

\node[font=\scriptsize\ttfamily\bfseries, text=blue!70!black, anchor=east] at (-0.28,0.15) {MDM};
\node[keep] at (0.00,0.15) {(}; \node[keep] at (0.48,0.15) {)}; \node[bad] at (0.96,0.15) {)}; \node[keep] at (1.44,0.15) {)}; \node[keep] at (1.92,0.15) {[}; \node[keep] at (2.40,0.15) {]};
\draw[arrow] (2.76,0.15) -- (3.10,0.15);
\node[keep] at (3.42,0.15) {(}; \node[keep] at (3.90,0.15) {)}; \node[erase] at (4.38,0.15) {M}; \node[keep] at (4.86,0.15) {)}; \node[keep] at (5.34,0.15) {[}; \node[keep] at (5.82,0.15) {]};
\node[font=\scriptsize\ttfamily, align=center] at (4.76,-0.27) {re-mask coordinate};
\end{tikzpicture}
}
\captionof{figure}{MDM-VGB can edit an early mistake directly, while AR-VGB must erase the suffix. The red token marks the local error, and gray tokens denote positions selected for re-masking.}
\label{fig:prefix_vs_aoar_repair}
\end{minipage}
\hfill
\begin{minipage}[t]{0.47\linewidth}
\vspace{0pt}
\centering
{\scriptsize
\setlength{\tabcolsep}{3.5pt}
\renewcommand{\arraystretch}{1.05}
\begin{tabular}{lcccc}
\toprule
& \multicolumn{2}{c}{\textsc{AR}} & \multicolumn{2}{c}{\textsc{MDM}} \\
\cmidrule(lr){2-3}\cmidrule(lr){4-5}
Method
& Acc. $\uparrow$
& Cost $\downarrow$
& Acc. $\uparrow$
& Cost $\downarrow$ \\
\midrule
VGB
& 25.64\%
& 127.63
& \underline{99.41}\%
& 29.98 \\
+Momentum
& 80.92\%
& 75.50
& 97.54\%
& \underline{26.24} \\
\bottomrule
\end{tabular}
}
\vspace{3pt}
\captionof{table}{Dyck grammar editing. Cost is the average number of forward and backward moves. MDM-VGB variants achieve higher accuracy with fewer edit moves than AR-VGB variants.}
\label{tab:prefix_vs_aoar_repair}
\vspace{-15pt}
\end{minipage}

\end{figure}

\cref{fig:prefix_vs_aoar_repair,tab:prefix_vs_aoar_repair} show the advantage of \textsc{MDM-VGB} and \textsc{MDM-VGB-Momentum} over the fixed-order backtracking baselines \textsc{AR-VGB} and \textsc{AR-VGB-Momentum} as editing procedures. On Dyck
grammar editing, \textsc{MDM-VGB} reaches \(99.41\%\) accuracy with an average
cost of \(29.98\) moves, compared with \(25.64\%\) accuracy and \(127.63\)
moves for \textsc{AR-VGB}. Momentum reduces edit cost for both AR and MDM variants: \textsc{MDM-VGB-Momentum}
achieves \(97.54\%\) accuracy with \(26.24\) moves, while
\textsc{AR-VGB-Momentum} reaches \(80.92\%\) accuracy with \(75.50\) moves. For MDM-VGB, momentum trades a small drop in accuracy for fewer moves, while the any-order variant remains substantially more efficient overall.

\subsection{Component Ablations}
\label{sec:exp_ablation}
\begin{wrapfigure}{r}{0.40\linewidth}
\vspace{-45pt}
\centering
\includegraphics[width=\linewidth]{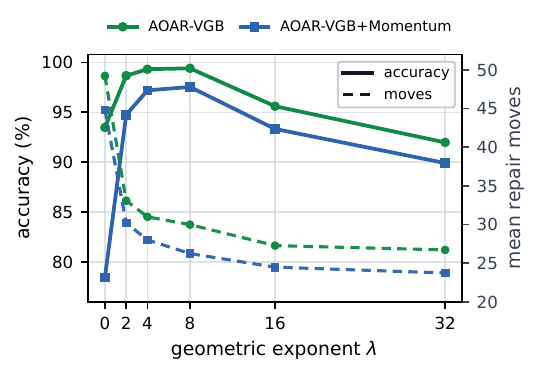}
\vspace{-15pt}
\caption{Dyck grammar editing with varying re-masking strength \(\lambda\). Moderate re-masking strength improves editing accuracy while reducing the average number of moves.}
\label{fig:dyck_lambda_sweep}
\vspace{-15pt}
\end{wrapfigure}
We highlight two main ablations: one on the re-masking parameter $\lambda$, which controls how strongly re-masking decisions are guided by reward values, and one on verifier model size.
Additional ablations, including block size and shortlisting rule, are provided in \cref{app:ablation_details}.

\paragraph{Re-masking parameter $\lambda$.}
We ablate the re-masking parameter \(\lambda\) on Dyck grammar editing. We observe that moderate
values of \(\lambda\) give the best accuracy--efficiency tradeoff, because the algorithm can focus on editing problematic tokens while retaining enough flexibility to explore alternative moves. Very large
\(\lambda\) makes the algorithm more deterministic, producing shorter trajectories
but lowering accuracy; see
\cref{fig:dyck_lambda_sweep}.

\paragraph{Verifier size and amortized cost.}
\begin{wrapfigure}{r}{0.40\linewidth}
\vspace{-25pt}
\centering
\includegraphics[width=0.85\linewidth]{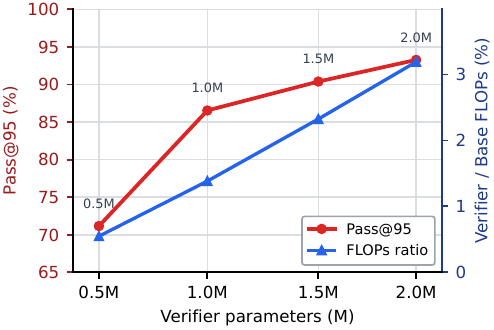}
\vspace{-2pt}
\caption{DNA verifier size tradeoff. Larger verifiers improve Pass@95 at the cost of additional verifier inference.}
\label{fig:dna_verifier_size_quality_flops}
\vspace{-20pt}
\end{wrapfigure}
We ablate the size of the learned DNA verifier in
\cref{fig:dna_verifier_size_quality_flops}. Increasing verifier size from 0.5M
to 2.0M parameters improves Pass@95, while adding less than \(3\%\) overhead
relative to the base generator FLOPs. This suggests that a stronger verifier can
provide more useful guidance for deciding which tokens to reveal or re-mask. To
keep this guidance inexpensive, we evaluate only a small shortlisted set of
candidate edits: for each selected coordinate, we sample \(K\) candidate tokens
from the reference model, remove duplicates, and batch the verifier calls.

\section{Conclusion and Future Work}
\label{main:sec:Conclusion and Limitations}
We introduced MDM-VGB, a reward-guided backtracking sampler for masked diffusion and any-order discrete generation. By extending value-guided backtracking from fixed prefix-tree to masked-state graph, MDM-VGB can reveal and re-mask arbitrary coordinates, allowing it to effectively repair faulty tokens. Theoretically, MDM-VGB efficiently target the reward-tilted law on full configurations and tolerate process verifiers noise. Empirically, MDM-VGB and MDM-VGB-Momentum improve the quality--cost frontier over baselines such as best-of-$N$ and action level sampling (MDM-VGR) for reward-tilting generation and editing across scientific-design and constraint-satisfaction tasks, with especially large gains when the base model performs poorly.

A key limitation of our work lies in the cost of training high-quality process verifiers, especially in domains with long sequence length and large vocabulary. We plan to improve verifier training through distillation, or to directly learn the MDM-VGB transition weights via fine-tuning. On the inference-time side, we plan to incorporate hierarchical update and combine our method with parallel Monte Carlo sampling \citet{lee2025test}. We hope these improvements will make MDM-VGB applicable on larger reasoning and coding benchmarks such as GSM8K, MATH500, HumanEval, and SWE-bench.

\ifshowbibliography
    \bibliography{references}

@article{rohatgi2025taming,
  title={Taming Imperfect Process Verifiers: A Sampling Perspective on Backtracking},
  author={Rohatgi, Dhruv and Shetty, Abhishek and Saless, Donya and Li, Yuchen and Moitra, Ankur and Risteski, Andrej and Foster, Dylan J},
  journal={arXiv preprint arXiv:2510.03149},
  year={2025}
}

@book{levin2017markov,
  title={Markov Chains and Mixing Times},
  author={Levin, David A. and Peres, Yuval},
  volume={107},
  year={2017},
  publisher={American Mathematical Society}
}

@inproceedings{hayes2010liftings,
  title={Liftings of Tree-Structured Markov Chains},
  author={Hayes, Thomas P. and Sinclair, Alistair},
  booktitle={Approximation, Randomization, and Combinatorial Optimization. Algorithms and Techniques},
  pages={602--616},
  year={2010},
  publisher={Springer}
}

@article{sinclair1989approximate,
  title={Approximate counting, uniform generation and rapidly mixing Markov chains},
  author={Sinclair, Alistair and Jerrum, Mark},
  journal={Information and Computation},
  volume={82},
  number={1},
  pages={93--133},
  year={1989},
  publisher={Elsevier}
}

@incollection{chomsky1963algebraic,
  title={The Algebraic Theory of Context-Free Languages},
  author={Chomsky, Noam and Sch{\"u}tzenberger, Marcel P.},
  booktitle={Computer Programming and Formal Systems},
  pages={118--161},
  year={1963},
  publisher={North-Holland}
}

@inproceedings{austin2021structured,
  title={Structured Denoising Diffusion Models in Discrete State-Spaces},
  author={Austin, Jacob and Johnson, Daniel D. and Ho, Jonathan and Tarlow, Daniel and van den Berg, Rianne},
  booktitle={Advances in Neural Information Processing Systems},
  volume={34},
  pages={17981--17993},
  year={2021}
}

@inproceedings{sahoo2024simple,
  title={Simple and Effective Masked Diffusion Language Models},
  author={Sahoo, Subham Sekhar and Arriola, Marianne and Schiff, Yair and Gokaslan, Aaron and Marroquin, Edgar and Chiu, Justin T and Rush, Alexander and Kuleshov, Volodymyr},
  booktitle={Advances in Neural Information Processing Systems},
  volume={37},
  year={2024}
}

@inproceedings{arriola2025block,
  title={Block Diffusion: Interpolating Between Autoregressive and Diffusion Language Models},
  author={Arriola, Marianne and Gokaslan, Aaron and Chiu, Justin T. and Yang, Zhihan and Qi, Zhixuan and Han, Jiaqi and Sahoo, Subham Sekhar and Kuleshov, Volodymyr},
  booktitle={The Thirteenth International Conference on Learning Representations},
  year={2025}
}

@article{cobbe2021training,
  title={Training Verifiers to Solve Math Word Problems},
  author={Cobbe, Karl and Kosaraju, Vineet and Bavarian, Mohammad and Chen, Mark and Jun, Heewoo and Kaiser, Lukasz and Plappert, Matthias and Tworek, Jerry and Hilton, Jacob and Nakano, Reiichiro and Hesse, Christopher and Schulman, John},
  journal={arXiv preprint arXiv:2110.14168},
  year={2021}
}

@inproceedings{wang2023selfconsistency,
  title={Self-Consistency Improves Chain of Thought Reasoning in Language Models},
  author={Wang, Xuezhi and Wei, Jason and Schuurmans, Dale and Le, Quoc and Chi, Ed and Narang, Sharan and Chowdhery, Aakanksha and Zhou, Denny},
  booktitle={The Eleventh International Conference on Learning Representations},
  year={2023}
}

@article{brown2024large,
  title={Large Language Monkeys: Scaling Inference Compute with Repeated Sampling},
  author={Brown, Bradley and Juravsky, Jordan and Ehrlich, Ryan and Clark, Ronald and Le, Quoc V. and R{\'e}, Christopher and Mirhoseini, Azalia},
  journal={arXiv preprint arXiv:2407.21787},
  year={2024}
}

@article{huang2025bestofn,
  title={Is Best-of-{N} the Best of Them? Coverage, Scaling, and Optimality in Inference-Time Alignment},
  author={Huang, Audrey and Block, Adam and Liu, Qinghua and Jiang, Nan and Krishnamurthy, Akshay and Foster, Dylan J.},
  journal={arXiv preprint arXiv:2503.21878},
  year={2025}
}

@inproceedings{lightman2024lets,
  title={Let's Verify Step by Step},
  author={Lightman, Hunter and Kosaraju, Vineet and Burda, Yura and Edwards, Harri and Baker, Bowen and Lee, Teddy and Leike, Jan and Schulman, John and Sutskever, Ilya and Cobbe, Karl},
  booktitle={The Twelfth International Conference on Learning Representations},
  year={2024}
}

@inproceedings{wang2024mathshepherd,
  title={Math-Shepherd: Verify and Reinforce {LLM}s Step-by-step without Human Annotations},
  author={Wang, Peiyi and Li, Lei and Shao, Zhihong and Xu, Runxin and Dai, Damai and Li, Yifei and Chen, Deli and Wu, Yu and Sui, Zhifang},
  booktitle={Proceedings of the 62nd Annual Meeting of the Association for Computational Linguistics (Volume 1: Long Papers)},
  pages={9426--9439},
  year={2024},
  address={Bangkok, Thailand},
  publisher={Association for Computational Linguistics}
}

@inproceedings{wang2025valueguided,
  title={Value-Guided Search for Efficient Chain-of-Thought Reasoning},
  author={Wang, Kaiwen and Zhou, Jin Peng and Chang, Jonathan and Gao, Zhaolin and Kallus, Nathan and Brantley, Kiant{\'e} and Sun, Wen},
  booktitle={Advances in Neural Information Processing Systems},
  year={2025}
}

@inproceedings{yang2021fudge,
  title={{FUDGE}: Controlled Text Generation With Future Discriminators},
  author={Yang, Kevin and Klein, Dan},
  booktitle={Proceedings of the 2021 Conference of the North American Chapter of the Association for Computational Linguistics: Human Language Technologies},
  pages={3511--3535},
  year={2021},
  address={Online},
  publisher={Association for Computational Linguistics}
}

@article{ramakrishnan2014quantum,
  title={Quantum chemistry structures and properties of 134 kilo molecules},
  author={Ramakrishnan, Raghunathan and Dral, Pavlo O. and Rupp, Matthias and von Lilienfeld, O. Anatole},
  journal={Scientific Data},
  volume={1},
  pages={140022},
  year={2014}
}

@inproceedings{wang2025remasking,
  title={Remasking Discrete Diffusion Models with Inference-Time Scaling},
  author={Wang, Guanghan and Schiff, Yair and Sahoo, Subham Sekhar and Kuleshov, Volodymyr},
  booktitle={The Thirty-ninth Annual Conference on Neural Information Processing Systems},
  year={2025}
}

@article{qwen2025qwen3,
  title={Qwen3 Technical Report},
  author={Yang, An and Li, Anfeng and Yang, Baosong and Zhang, Beichen and Hui, Binyuan and Zheng, Bo and Yu, Bowen and Gao, Chang and Huang, Chengen and Lv, Chenxu and Zheng, Chujie and Liu, Dayiheng and Zhou, Fan and Huang, Fei and Hu, Feng and Ge, Hao and Wei, Haoran and Lin, Huan and Tang, Jialong and Yang, Jian and Tu, Jianhong and Zhang, Jianwei and Yang, Jianxin and Yang, Jiaxi and Zhou, Jing and Zhou, Jingren and Lin, Junyang and Dang, Kai and Bao, Keqin and Yang, Kexin and Yu, Le and Deng, Lianghao and Li, Mei and Xue, Mingfeng and Li, Mingze and Zhang, Pei and Wang, Peng and Zhu, Qin and Men, Rui and Gao, Ruize and Liu, Shixuan and Luo, Shuang and Li, Tianhao and Tang, Tianyi and Yin, Wenbiao and Ren, Xingzhang and Wang, Xinyu and Zhang, Xinyu and Ren, Xuancheng and Fan, Yang and Su, Yang and Zhang, Yichang and Zhang, Yinger and Wan, Yu and Liu, Yuqiong and Wang, Zekun and Cui, Zeyu and Zhang, Zhenru and Zhou, Zhipeng and Qiu, Zihan},
  journal={arXiv preprint arXiv:2505.09388},
  year={2025}
}

@article{misaki2026unmaskfork,
  title={UnMaskFork: Test-Time Scaling for Masked Diffusion via Deterministic Action Branching},
  author={Misaki, Kou and Akiba, Takuya},
  journal={arXiv preprint arXiv:2602.04344},
  year={2026}
}

@article{Lee2025EffectiveTS,
  title={Effective Test-Time Scaling of Discrete Diffusion through Iterative Refinement},
  author={Lee, Sanghyun and Kim, Sunwoo and Kim, Seungryong and Park, Jongho and Park, Dongmin},
  journal={arXiv preprint arXiv:2511.05562},
  year={2025}
}

@article{lee2025test,
  title={Test-Time Scaling in Diffusion LLMs via Hidden Semi-Autoregressive Experts},
  author={Lee, Jihoon and Moon, Hoyeon and Zhai, Kevin and Chithanar, Arun Kumar and Sahu, Anit Kumar and Kar, Soummya and Lee, Chul and Chakraborty, Souradip and Bedi, Amrit Singh},
  journal={arXiv preprint arXiv:2510.05040},
  year={2025}
}

@article{ou2025inference,
  title={Inference-Time Scaling of Discrete Diffusion Models via Importance Weighting and Optimal Proposal Design},
  author={Ou, Zijing and Pani, Chinmay and Li, Yingzhen},
  journal={arXiv preprint arXiv:2505.22524},
  year={2025}
}

@article{kim2025fine,
  title={Fine-Tuning Masked Diffusion for Provable Self-Correction},
  author={Kim, Jaeyeon and Kim, Seunggeun and Lee, Taekyun and Pan, David Z. and Kim, Hyeji and Kakade, Sham and Chen, Sitan},
  journal={arXiv preprint arXiv:2510.01384},
  year={2025}
}

@article{dealmeida2022deepstarr,
  title={{DeepSTARR} predicts enhancer activity from {DNA} sequence and enables the de novo design of synthetic enhancers},
  author={de Almeida, Bernardo P. and Reiter, Franziska and Pagani, Michaela and Stark, Alexander},
  journal={Nature Genetics},
  volume={54},
  pages={613--624},
  year={2022}
}

@article{dallatorre2025nucleotide,
  title={Nucleotide Transformer: building and evaluating robust foundation models for human genomics},
  author={Dalla-Torre, Hugo and Gonzalez, Liam and Mendoza-Revilla, Javier and Lopez Carranza, Nicolas and Grzywaczewski, Adam Henryk and Oteri, Francesco and Dallago, Christian and Trop, Evan and de Almeida, Bernardo P. and Sirelkhatim, Hassan and Richard, Guillaume and Skwark, Marcin and Beguir, Karim and Lopez, Marie and Pierrot, Thomas},
  journal={Nature Methods},
  volume={22},
  pages={287--297},
  year={2025},
  doi={10.1038/s41592-024-02523-z}
}

@misc{yang2026d3lmdiscretednadiffusion,
  title={{D3LM}: A Discrete {DNA} Diffusion Language Model for Bidirectional {DNA} Understanding and Generation},
  author={Yang, Zhao and Liu, Hengchang and Cao, Chuan and Su, Bing},
  year={2026},
  eprint={2603.01780},
  archivePrefix={arXiv},
  primaryClass={cs.LG},
  url={https://arxiv.org/abs/2603.01780}
}

@inproceedings{alamdari2023protein,
  title={Protein generation with evolutionary diffusion: sequence is all you need},
  author={Alamdari, Sarah and Thakkar, Nitya and van den Berg, Rianne and Lu, Alex and Fusi, Nicolo and Amini, Ava and Yang, Kevin},
  booktitle={NeurIPS 2023 Generative AI and Biology (GenBio) Workshop},
  year={2023}
}

@article{wu2022omegafold,
  title={High-resolution de novo structure prediction from primary sequence},
  author={Wu, Ruidong and Ding, Fan and Wang, Rui and Shen, Rui and Zhang, Xiwen and Luo, Shitong and Su, Chenpeng and Wu, Zuofan and Xie, Qi and Berger, Bonnie and Ma, Jianzhu and Peng, Jian},
  journal={bioRxiv preprint},
  year={2022}
}

@misc{qwen3_06b_diffusion_mdlm,  
title        = {Qwen3-0.6B-diffusion-mdlm-v0.1},  
author       = {{dLLM Hub}},  
year         = {2026},  
howpublished = {\url{https://huggingface.co/dllm-hub/Qwen3-0.6B-diffusion-mdlm-v0.1}},}

@inproceedings{peebles2023scalable,
  title     = {Scalable Diffusion Models with Transformers},
  author    = {Peebles, William and Xie, Saining},
  booktitle = {Proceedings of the IEEE/CVF International Conference on Computer Vision},
  year      = {2023}
}

@article{bickerton2012quantifying,
author = {Bickerton, Richard and Paolini, Gaia and Besnard, Jérémy and Muresan, Sorel and Hopkins, Andrew},
year = {2012},
month = {02},
pages = {90-8},
title = {Quantifying the chemical beauty of drugs},
volume = {4},
journal = {Nature chemistry},
doi = {10.1038/nchem.1243}
}

@article{MAIOROV1994625,
title = {Significance of Root-Mean-Square Deviation in Comparing Three-dimensional Structures of Globular Proteins},
journal = {Journal of Molecular Biology},
volume = {235},
number = {2},
pages = {625-634},
year = {1994},
issn = {0022-2836},
doi = {https://doi.org/10.1006/jmbi.1994.1017},
url = {https://www.sciencedirect.com/science/article/pii/S0022283684710175},
author = {Vladimir N. Maiorov and Gordon M. Crippen}
}
    \bibliographystyle{plainnat}
\fi

\clearpage
\appendix
\clearpage
\tableofcontents
\clearpage
\section{Table of Key Notation and Baseline Algorithms}
\label{app:sec:key_notation_and_remarks}
\renewcommand{\arraystretch}{1.3}
\begin{table}[ht]
\caption{Table of Key Notations}
\label{tab:notation}
\vspace{5pt}
\centering
\renewcommand{\arraystretch}{1.3}
\begin{tabular}{c l l}
\toprule
\textbf{Symbol} & \textbf{Definition} & \textbf{Descriptions} \\
\midrule

$x$
  & conditioning context
  & Prompt or source input \\

$n$
  & sequence length
  & Number of token positions \\

$\calV$
  & finite vocabulary
  & Token alphabet \\

$\calY$
  & $\calV^n$
  & Fully revealed sequences \\

$\calZ$
  & $\sbra{\calV \cup \mbra{\texttt{[mask]}}}^n$
  & Masked state space \\

$R\sbra{z}$
  & $\mbra{i \in \lbra{n} : z_i \neq \texttt{[mask]}}$
  & Observed coordinates of $z$ \\

$k\sbra{z}$
  & $\abs{R\sbra{z}}$
  & Depth of masked state $z$ \\

$\calC\sbra{z}$
  & $\mbra{y \in \calY : y_i = z_i,\ i \in R\sbra{z}}$
  & Compatible full completions of $z$ \\

$\pi_{\mathrm{ref}}\sbra{y \mid x}$
  & base sequence model
  & Reference distribution \\

$\tau\sbra{x,y}$
  & nonnegative reward
  & Leaf-level verifier or reward \\

$\pi^\star\sbra{y \mid x}$
  & $\propto \pi_{\mathrm{ref}}\sbra{y \mid x}\tau\sbra{x,y}$
  & Tilted target distribution \\

$M_{\mathrm{ref}}\sbra{x,z}$
  & $\sum_{y \in \calC\sbra{z}} \pi_{\mathrm{ref}}\sbra{y \mid x}$
  & Base completion mass \\

$V^\star\sbra{x,z}$
  & $\bbE_{\pi_{\mathrm{ref}}}\lbra{\tau\sbra{x,Y}\mid x,Y \in \calC\sbra{z}}$
  & Conditional-expectation value \\

$U^\star\sbra{x,z}$
  & $\sum_{y \in \calC\sbra{z}} \pi_{\mathrm{ref}}\sbra{y \mid x}\tau\sbra{x,y}$
  & Tilted completion mass \\

$\widehat V\sbra{x,z}$
  & learned approximation of $V^\star\sbra{x,z}$
  & Practical plug-in verifier \\

$z^{j \leftarrow a}$
  & single-site reveal update
  & Fill masked coordinate $j$ with token $a$ \\

$z^{-i}$
  & single-site re-mask update
  & Re-mask observed coordinate $i$ \\

$s_k$
  & depth-dependent edge coefficient
  & Balanced AOAR-VGB coefficient \\

$\lambda$
  & geometric value-gating exponent
  & Strength of geometric remasking gate \\

$d$
  & $d \in \mbra{\uparrow,\downarrow}$
  & Momentum direction state \\

$\chi$
  & cancellation strength
  & Flow-cancelled momentum parameter \\

\midrule

$B$
  & subset of coordinates with $r=|B|$
  & Reveal or re-mask block \\

$s_{k,r}$
  & $\binom{n-r}{k}^{-1}$
  & Block-update depth coefficient \\

$z^{B \leftarrow a_B}$
  & block reveal update
  & Fill block $B$ with assignment $a_B$ \\

$z^{-B}$
  & block re-mask update
  & Mask every coordinate in block $B$ \\

\bottomrule
\end{tabular}
\vspace{1mm}
\end{table}

\begin{figure}[h!]
\centering
\resizebox{0.98\textwidth}{!}{
\begin{tikzpicture}[
    node distance=1.7cm and 1.95cm,
    method/.style={
        ellipse,
        draw=darkgreen!80!black,
        fill=remarkgreen,
        very thick,
        align=center,
        minimum width=3.05cm,
        minimum height=1.12cm,
        inner sep=3pt
    },
    prior/.style={
        method,
        draw=gray!65!black,
        fill=gray!8
    },
    problem/.style={
        font=\footnotesize\itshape,
        align=center,
        text=darkred
    },
    extension/.style={
        font=\footnotesize,
        align=center,
        text=darkgreen
    },
    flow/.style={
        -{Latex[length=2.3mm]},
        very thick,
        draw=darkgreen!70!black
    },
    issue/.style={
        -{Latex[length=2.3mm]},
        very thick,
        draw=darkred!75!black
    }
]
\node[prior] (ar) {AR-VGB\\[-1pt]{\scriptsize \citep{rohatgi2025taming}}};
\node[prior, right=of ar] (armom) {AR-VGB\\Momentum\\[-1pt]{\scriptsize \citep{rohatgi2025taming}}};
\node[method, below=of ar] (primitive) {Primitive\\AOAR-VGB\\[-1pt]{\scriptsize (\cref{app:subsec:primitive_aoar_vgb})}};
\node[method, right=2.25cm of primitive] (balanced) {Balanced\\AOAR-VGB\\[-1pt]{\scriptsize (\cref{app:subsec:balanced_aoar_vgb})}};
\node[method, right=2.25cm of balanced] (geobal) {Geometric\\Balanced\\AOAR-VGB\\[-1pt]{\scriptsize (\cref{app:subsec:balanced_geometric_aoar_vgb})}};
\node[method, right=2.25cm of geobal] (balmom) {Momentum-Balanced\\Geometric\\AOAR-VGB\\[-1pt]{\scriptsize (\cref{app:subsec:momentum_bg_aoar_vgb})}};
\node[method, below=1.75cm of geobal] (mdm) {Geometric\\Balanced\\MDM-VGB\\[-1pt]{\scriptsize (\cref{app:subsec:balanced_geometric_mdm_vgb})}};
\node[method, right=2.25cm of mdm] (mdmmom) {Momentum-Balanced\\Geometric\\MDM-VGB\\[-1pt]{\scriptsize (\cref{app:subsec:momentum_bg_mdm_vgb})}};

\draw[issue] (ar) -- node[problem, above] {leaf\\hitting} (armom);
\draw[flow] (ar) -- node[extension, left] {AOAR\\extension} (primitive);
\draw[issue] (primitive) -- node[problem, above] {leaf-mass\\dilution}
    node[problem, below, yshift=-0.35em] {\scriptsize
    (\hyperref[rem:primitive_vs_balanced_directional_bias]{Remark~\ref*{rem:primitive_vs_balanced_directional_bias}}\\[-1pt]
    \& \hyperref[prop:primitive_leaf_dilution]{Prop.~\ref*{prop:primitive_leaf_dilution}})} (balanced);
\draw[flow] (balanced) -- node[extension, above] {value-gated\\remasking} (geobal);
\draw[issue] (geobal) -- node[problem, above] {leaf\\hitting} (balmom);
\draw[flow] (geobal) -- node[extension, left] {MDM\\extension} (mdm);
\draw[issue] (mdm) -- node[problem, above] {leaf\\hitting} (mdmmom);
\end{tikzpicture}
}
\caption{Roadmap of the VGB variants studied in the appendix. AR-VGB and its momentum variant follow \citet{rohatgi2025taming}; Primitive AOAR-VGB gives the any-order extension; Balanced AOAR-VGB addresses leaf-mass dilution; Geometric Balanced AOAR-VGB adds value-gated remasking; momentum-balanced geometric variants target leaf hitting under finite budgets; and Geometric Balanced MDM-VGB gives the block-update extension.}
\label{fig:appendix_vgb_roadmap}
\end{figure}
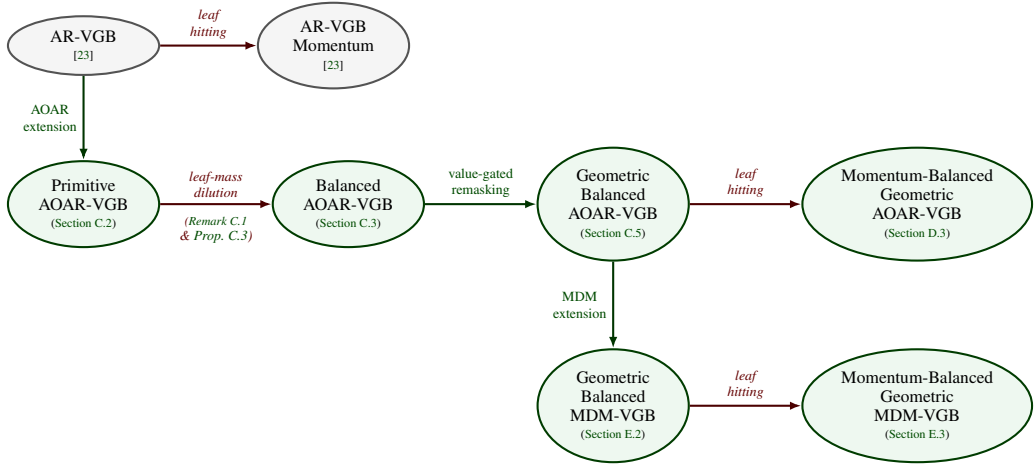
Throughout the appendix, AOAR-VGB denotes the singleton-block any-order case, while MDM-VGB denotes the block-update method used in the main text.

\begin{table*}[h!]
    \centering
    \caption{
    Overview of baseline methods. The final column marks whether the
    theoretical sampler has the exact leaf-conditioned target
    $\pi^\star(y\mid x)\propto \base(y\mid x)\tau(x,y)$. Here
    $\triangle$ denotes conditional exactness: BoN matches the
    success-conditioned target when $\tau\in\mbra{0,1}$, while VGR is exact only
    when $\widehat V=V^\star$. Exactness here refers to the ideal asymptotic sampler with full-neighborhood transitions.
    }
    \label{tab:baseline_taxonomy}
    \vspace{2pt}
    \small
    \setlength{\tabcolsep}{5.0pt}
    \renewcommand{\arraystretch}{1.08}
    \begin{tabular}{l c c c}
    \toprule
    Method
    & Uses $\widehat V$?
    & Backtracking?
    & Exact $\pi^\star$ law? \\
    \midrule
    Base
    & $\mathsf{X}$
    & $\mathsf{X}$
    & $\mathsf{X}$ \\
    BoN
    & $\mathsf{X}$
    & $\mathsf{X}$
    & $\triangle$ \\
    VGR
    & $\mathsf{O}$
    & $\mathsf{X}$
    & $\triangle$ \\
    VGB
    & $\mathsf{O}$
    & $\mathsf{O}$
    & $\mathsf{O}$ \\
    VGB-Momentum
    & $\mathsf{O}$
    & $\mathsf{O}$
    & $\mathsf{O}$ \\
    \bottomrule
    \end{tabular}
\end{table*}

\subsection{Baseline Algorithms}
\label{app:subsec:baseline_algorithms}

We compare MDM-VGB algorithms against baselines that use the same reference model $\base$. BoN is the terminal-verifier version of the standard best-of-$N$ reranking baseline. VGR keeps the verifier-weighted forward reveal rule from VGB, but disables all re-mask (backtracking) moves. In this subsection, $\widehat V(x,z)$ denotes the learned verifier score used by value-guided baselines on non-leaf states, while the true reward $\tau(x,y)$ is used at leaves.

\paragraph{Best-of-$N$ with terminal verifier.}
BoN is the usual verifier-reranking or best-of-$N$ selection baseline \citep{cobbe2021training,brown2024large,huang2025bestofn}. It samples $N$ independent complete candidates from a forward-only reference rollout of $\base(\cdot\mid x)$, and returns the candidate with the largest terminal verifier value $\tau(x,y)$. It therefore tests what can be gained by spending compute on more complete reference samples, without modifying local transitions online.

\begin{algorithm}[h!]
\caption[Best-of-$N$ with terminal verifier (BoN)]{Best-of-$N$ with terminal verifier (BoN)
(\protect\raisebox{0.15ex}{\protect\colorbox{Lavender!25}{\protect\rule{0pt}{0.8ex}\protect\hspace{0.8ex}}} = parallelizable area)}
\label{alg:bon}
\begin{algorithmic}[1]
\Require context $x$, reference model $\base$, reward $\tau$, number of samples $N$
\State Initialize $\mathcal D\gets\varnothing$
\Statex \hspace*{\algorithmicindent}\colorbox{Lavender!25}{
\parbox{0.86\linewidth}{
\textbf{Parallel candidate generation.}
For $i=1,\dots,N$, independently sample a complete candidate
$y^{(i)}$ by a forward-only reference rollout from $\base(\cdot\mid x)$,
and compute $s_i\gets\tau(x,y^{(i)})$.}}
\For{$i=1,\dots,N$}
    \State Add $(y^{(i)},s_i)$ to $\mathcal D$
\EndFor
\State $i^\star\gets \argmax_{i\in\lbra{N}}s_i$
\State \Return $y^{(i^\star)}$
\end{algorithmic}
\end{algorithm}

\paragraph{Value-guided rollout.}
\textsc{VGR} is a forward-only value-guided baseline. It is the AOAR/MDM
analogue of action-level rejection sampling in \citet{rohatgi2025taming}: the
verifier is used online to choose forward reveal actions, but all
re-mask moves are disabled. Thus \textsc{VGR} isolates online value-guided
forward selection from the stochastic backtracking mechanism that distinguishes
VGB.

Since the full forward action space can be large, \textsc{VGR} uses the same
forward shortlisting approximation as VGB. At state $z$, it constructs a
shortlisted forward candidate set $\mathcal A(z)$. For a candidate
$u=z^{B\leftarrow a_B}$, with singleton blocks giving the AOAR case, the
\textsc{VGR} weight is
\begin{equation*}
    w(z,u)
    :=
    s_{k(z),|B|}\,
    \base(Y_B=a_B\mid x,z)\,
    \widehat V(x,u).
\end{equation*}
The geometric factor $\widehat V(x,z)^\lambda$ is omitted because it is common
to all forward candidates from the same state and cancels under normalization.
We do not report a separate momentum version of \textsc{VGR}: once backward
moves are disabled, the reveal/re-mask direction variable is degenerate.
\begin{algorithm}[h!]
\caption[Value-guided rollout (VGR)]{Value-guided rollout (VGR)
(\protect\raisebox{0.15ex}{\protect\colorbox{Lavender!25}{\protect\rule{0pt}{0.8ex}\protect\hspace{0.8ex}}} = parallelizable area)}
\label{alg:vgr}
\begin{algorithmic}[1]
\Require context $x$, initial masked state $z_0$, reference model $\base$, verifier score $\widehat V$, depth coefficients $s_{k,r}$, step limit $T$
\State $z\gets z_0$
\For{$t=1,\dots,T$}
    \If{$z\in\calY$}
        \State \Return $z$
    \EndIf
    \Statex \hspace*{\algorithmicindent}\colorbox{Lavender!25}{
    \parbox{0.86\linewidth}{
    \textbf{Parallel forward construction.}
    Construct a shortlisted forward candidate set $\mathcal A(z)$. For every
    $u\in\mathcal A(z)$, write $u=z^{B\leftarrow a_B}$,
    $r=|B|$, compute
    \begin{equation*}
        w(u)
        =
        s_{k(z),r}\,
        \base(Y_B=a_B\mid x,z)\,
        \widehat V(x,u).
    \end{equation*}}}
    \If{$\sum_{u\in\mathcal A(z)}w(u)>0$}
        \State Sample $u\in\mathcal A(z)$ proportional to $w(u)$
    \Else
        \State Sample $u=z^{B\leftarrow a_B}\in\mathcal A(z)$ proportional to $\base(Y_B=a_B\mid x,z)$
    \EndIf
    \State $z\gets u$
\EndFor
\If{$z\notin\calY$}
    \State Complete remaining masked coordinates using $\base(\cdot\mid x,z)$
\EndIf
\State \Return $z$
\end{algorithmic}
\end{algorithm}

\paragraph{MDM-VGB samplers.}
We next give the block-update VGB samplers. The first
algorithm is the generic MDM-VGB with admissible block sizes
$\calM$; the second applies the lifted momentum construction to the same
local forward and backward weights. Formal definitions of these weights and the
stationarity guarantees are given in the following sections. The block sets \(\calB_r^+(z)\), \(\calB_r^-(z)\) and the full MDM neighborhoods \(C_{\mathrm{MDM}}(z)\), \(P_{\mathrm{MDM}}(z)\) are defined in \cref{app:subsec:balanced_mdm_vgb}.

\begin{algorithm}[h!]
\caption[MDM-VGB sampler with block coefficients, without a self-loop]{MDM-VGB sampler with block coefficients, without a self-loop
(\protect\raisebox{0.15ex}{\protect\colorbox{Lavender!25}{\protect\rule{0pt}{0.8ex}\protect\hspace{0.8ex}}} = parallelizable area)}
\label{alg:mdm_vgb_sampler}
\begin{algorithmic}[1]
\Statex \textbf{Input:} context $x$, initial state $z$, reference block conditionals $\base(\cdot\mid x,z)$, verifier $\widehat V$, reward $\tau$, admissible sizes $\calM$, coefficients $s_{k,r}$, geometric exponent $\lambda\ge0$, step limit \texttt{step\_limit}
\Statex \textbf{Output:} leaf sample list $\mathcal L$
\State Initialize $\mathcal L\gets\varnothing$ and $\texttt{step}\gets0$
\Statex \hspace{\algorithmicindent}$\triangleright$ Balanced MDM-VGB uses $\lambda=0$; Geometric Balanced MDM-VGB allows $\lambda>0$.
\While{$\texttt{step}<\texttt{step\_limit}$}
    \State Set $k\gets k(z)=\abs{R(z)}$
    \State Initialize an empty weighted candidate set $\mathcal A$
    \If{$k<n$}
        \Statex \hspace*{\algorithmicindent}\colorbox{Lavender!25}{
        \parbox{0.88\linewidth}{
        \textbf{Parallel forward construction.}
        For all $r\in\calM$ with $r\le n-k$, all $A\in\calB_r^+(z)$, and all
        $a_A\in\calV^A$, independently set $u_{A,a_A}\gets z^{A\leftarrow a_A}$
        and add $u_{A,a_A}$ to $\mathcal A$ with weight
        \begin{align*}
            \begin{cases}
                s_{n-r,r}\base(Y_A=a_A\mid x,z)\tau(x,u_{A,a_A}),
                    & u_{A,a_A}\in\calY,\\
                s_{k,r}\base(Y_A=a_A\mid x,z)\widehat V(x,u_{A,a_A}),
                    & u_{A,a_A}\notin\calY,\ \lambda=0,\\
                s_{k,r}\widehat V(x,z)^\lambda
                \base(Y_A=a_A\mid x,z)\widehat V(x,u_{A,a_A}),
                    & u_{A,a_A}\notin\calY,\ \lambda>0.
            \end{cases}
        \end{align*}}}
    \EndIf
    \If{$k>0$}
        \Statex \hspace*{\algorithmicindent}\colorbox{Lavender!25}{
        \parbox{0.88\linewidth}{
        \textbf{Parallel backward construction.}
        For all $r\in\calM$ with $r\le k$ and all $A\in\calB_r^-(z)$,
        independently set $u_A\gets z^{-A}$ and add $u_A$ to $\mathcal A$ with
        weight
        \begin{align*}
            \begin{cases}
                s_{n-r,r}\tau(x,z), & z\in\calY,\\
                s_{k-r,r}\widehat V(x,z),
                    & z\notin\calY,\ \lambda=0,\\
                s_{k-r,r}\widehat V(x,z)\widehat V(x,u_A)^\lambda,
                    & z\notin\calY,\ \lambda>0.
            \end{cases}
        \end{align*}}}
    \EndIf
    \State Sample $u\in\mathcal A$ proportional to its weight
    \State Set $z\gets u$
    \If{$z\in\calY$}
        \State Append $z$ to $\mathcal L$
    \EndIf
    \State $\texttt{step}\gets\texttt{step}+1$
\EndWhile
\State \Return $\mathcal L$
\end{algorithmic}
\end{algorithm}

\begin{algorithm}[t]
\caption[Momentum MDM-VGB sampler]{Momentum MDM-VGB sampler
(\protect\raisebox{0.15ex}{\protect\colorbox{Lavender!25}{\protect\rule{0pt}{0.8ex}\protect\hspace{0.8ex}}} = parallelizable area;
\protect\raisebox{0.15ex}{\protect\colorbox{SkyBlue!15}{\protect\rule{0pt}{0.8ex}\protect\hspace{0.8ex}}} = momentum area)}
\label{alg:momentum_mdm_sampler}
\label{alg:momentum_mdm_transition}
\label{alg:momentum_balanced_mdm_transition}
\begin{algorithmic}[1]
\Statex \textbf{Input:} context $x$, initial lifted state $(z,d)$ with $d\in\mbra{\downarrow,\uparrow}$, base local weights $w_{\rm fwd},w_{\rm bwd}$ for balanced or geometric MDM-VGB, cancellation strength $\chi\in[0,1]$, step limit \texttt{step\_limit}
\Statex \textbf{Output:} leaf sample list $\mathcal L$
\State Initialize $\mathcal L\gets\varnothing$ and $\texttt{step}\gets0$
\While{$\texttt{step}<\texttt{step\_limit}$}
    \State Set $k\gets k(z)=\abs{R(z)}$
    \If{$k=0$ and $d=\uparrow$}
        \State Set $d\gets\downarrow$ \hfill \textsc{Root Direction Reset}
    \EndIf
    \If{$k=n$ and $d=\downarrow$}
        \State Set $d\gets\uparrow$ \hfill \textsc{Leaf Direction Reset}
    \EndIf
    \State Initialize empty weighted candidate sets $\mathcal A_{\mathrm{fwd}}$, $\mathcal A_{\mathrm{bwd}}$, and $\widetilde{\mathcal A}$
    \If{$k<n$}
        \Statex \hspace*{\algorithmicindent}\colorbox{Lavender!25}{
        \parbox{0.88\linewidth}{
        \textbf{Parallel forward construction.}
        For all $c\in C_{\mathrm{MDM}}(z)$, independently add $c$ to
        $\mathcal A_{\mathrm{fwd}}$ with weight
        $w_{\rm fwd}(z\to c)$.}}
    \EndIf
    \If{$k>0$}
        \Statex \hspace*{\algorithmicindent}\colorbox{Lavender!25}{
        \parbox{0.88\linewidth}{
        \textbf{Parallel backward construction.}
        For all $p\in P_{\mathrm{MDM}}(z)$, independently add $p$ to
        $\mathcal A_{\mathrm{bwd}}$ with weight
        $w_{\rm bwd}(z\to p)$.}}
    \EndIf
    \Statex \hspace*{\algorithmicindent}\colorbox{SkyBlue!15}{
    \parbox{0.88\linewidth}{
    \textbf{Momentum construction.}
    Set
    \begin{align*}
    \begin{aligned}
        F(z)&=\sum_{c\in\mathcal A_{\mathrm{fwd}}}w_{\rm fwd}(z\to c),\\
        B(z)&=\sum_{p\in\mathcal A_{\mathrm{bwd}}}w_{\rm bwd}(z\to p),\\
        M(z)&=\min\mbra{F(z),B(z)}.
    \end{aligned}
    \end{align*}
    Populate $\widetilde{\mathcal A}$ with the following lifted-state weights:
    \begin{align*}
    \begin{array}{c@{\quad}c@{\quad}c@{\quad}c}
    \texttt{mode} & \texttt{lifted state} & \texttt{weight} & \texttt{transition mode}\\[0.5mm]
    \midrule
    d=\downarrow & (c,\downarrow),\ c\in\mathcal A_{\mathrm{fwd}} & w_{\rm fwd}(z\to c) & \textsc{Move}\\
    d=\downarrow & (z,\uparrow) & B(z)-\chi M(z) & \textsc{Switch Up}\\
    d=\downarrow & (z,\downarrow) & \chi M(z) & \textsc{Stay Down (optional)}\\[0.5mm]
    d=\uparrow & (p,\uparrow),\ p\in\mathcal A_{\mathrm{bwd}} & w_{\rm bwd}(z\to p) & \textsc{Move}\\
    d=\uparrow & (z,\downarrow) & F(z)-\chi M(z) & \textsc{Switch Down}\\
    d=\uparrow & (z,\uparrow) & \chi M(z) & \textsc{Stay Up (optional)}
    \end{array}
    \end{align*}}}
    \State Sample $(u,d')\in\widetilde{\mathcal A}$ proportional to its weight
    \State Set $(z,d)\gets(u,d')$
    \If{$z\in\calY$}
        \State Append $z$ to $\mathcal L$
    \EndIf
    \State $\texttt{step}\gets\texttt{step}+1$
\EndWhile
\State \Return $\mathcal L$
\end{algorithmic}
\end{algorithm}
\clearpage

\section{Shared Setup and Verifier Training}
\label{app:sec:shared_setup_verifier_training}
We first collect the notation shared by all VGB variants. Fix a conditioning context $x\sim\rho$, a sequence length $n$, and a finite vocabulary $\calV$. The fully revealed state space and the masked state space are
\begin{equation*}
    \calY := \calV^n, \qquad \calZ := \sbra{\calV\cup\mbra{\mask}}^n,
\end{equation*}
where $\mask\notin\calV$ is the masking symbol. For $z\in\calZ$, let
\begin{equation*}
    R(z) := \mbra{i\in\lbra{n}\mid z_i\neq\mask}, \qquad k(z):=\abs{R(z)}.
\end{equation*}
Here $R(z)$ is the set of revealed coordinates and $k(z)$ is the depth of $z$. The compatible completions of $z$ are
\begin{equation*}
    \calC(z) := \mbra{y\in\calY \mid y_i=z_i\ \text{for all}\ i\in R(z)}.
\end{equation*}
Equivalently, observing the partial state $z$ restricts the unknown terminal sequence $Y$ to the completion set $\calC(z)$. As more coordinates are revealed, this set becomes more informative.

\paragraph{Tilted target.}
Let $\base(y\mid x)$ be a reference sequence model and let $\tau(x,y)\ge 0$ be a terminal reward. These two objects define the \emph{tilted target distribution}
\begin{equation*}
    \pi^*(y\mid x) := \frac{\base(y\mid x)\tau(x,y)}{Z(x)}, \qquad Z(x):=\sum_{y\in\calY}\base(y\mid x)\tau(x,y).
\end{equation*}
The goal of the VGB procedures is to sample from $\pi^*(\cdot\mid x)$ while using $\base$ as the proposal backbone.

\paragraph{Completion values.}
To describe local moves on masked states, we need to know both how much reference probability remains under a partial assignment and how promising that partial assignment is under the reward. The first quantity is the \emph{reference completion mass}
\begin{equation*}
    M_{\mathrm{ref}}(x,z) := \sum_{y\in\calC(z)} \base(y\mid x) = \bbP_{Y\sim\base}(Y\in\calC(z)\mid x).
\end{equation*}
This mass determines which masked states are reachable under the reference model:
\begin{equation*}
    \calZ_{\mathrm{reach}}(x) := \mbra{z\in\calZ: M_{\mathrm{ref}}(x,z)>0}.
\end{equation*}
If $z\in\calZ_{\mathrm{reach}}(x)$, we say that $z$ is \emph{$\base$-reachable}. When the context is fixed, we write \(\calZ_{\mathrm{reach}}\) for \(\calZ_{\mathrm{reach}}(x)\).

The second quantity is the conditional-expectation value
\begin{equation*}
    V^*(x,z) := \bbE_{\base}\lbra{\tau(x,Y)\mid x,\ Y\in\calC(z)},
\end{equation*}
which is the object approximated by the verifier. For the analysis, it is also convenient to introduce the unnormalized tilted completion mass
\begin{equation*}
    U^*(x,z) := \sum_{y\in\calC(z)}\base(y\mid x)\tau(x,y).
\end{equation*}
We note that these two views are related by
\begin{equation*}
    U^*(x,z)=M_{\mathrm{ref}}(x,z)V^*(x,z).
\end{equation*}
The quantity $U^*$ is useful for stationarity proofs, but learning it directly would also require estimating the completion mass $M_{\mathrm{ref}}$. The VGB framework discussed in later sections avoids this extra estimation problem: its local transitions can be implemented using reference conditionals together with a learned approximation $\widehat V$ of $V^*$.

\paragraph{Verifier training.}
When a learned verifier is used, we train a masked-state verifier \(\widehat V(x,z)\). During training, we fit a parametric model $V_\theta(x,z)$ and then set $\widehat V$ to the trained model.

For a sampled context $x$, we run several reference rollouts, each producing a terminal completion $y$. Along the way, the sampler reveals coordinates, producing a trajectory of intermediate masked states. We subsample several snapshots from each trajectory and assign all of them the same terminal label $r=\tau(x,y)$.

The verifier is trained with the Monte Carlo (MC) rollout regression objective
\begin{equation*}
    \calL_{\mathrm{reg}}(\theta) := \bbE_{x,z,y}\lbra{\sbra{V_\theta(x,z)-\tau(x,y)}^2},
\end{equation*}
where $z$ is a sampled intermediate state from a rollout that terminates at $y$. The following proposition reveals the population target of this regression.

\begin{proposition}[MC regression recovers \(V^\star\)]
\label{prop:mc_regression_recovers_value}
Consider the population rollout law $\bbP_{\mathrm{ref}}$ over triples $(x,z,y)$. Suppose this law is reference-consistent: whenever $z\in\calZ_{\mathrm{reach}}(x)$,
\begin{equation*}
    \bbP_{\mathrm{ref}}(Y=y\mid x,z)
    =
    \base(y\mid x,\ Y\in\calC(z)).
\end{equation*}
Then the population regression objective
\begin{equation*}
    \calL_{\mathrm{reg}}(V)
    :=
    \bbE_{\bbP_{\mathrm{ref}}}
    \left[
        \bigl(V(x,z)-\tau(x,Y)\bigr)^2
    \right]
\end{equation*}
recovers \(V^\star\) as its population optimum:
\begin{equation*}
    V^\star = \argmin_V \calL_{\mathrm{reg}}(V)
    \qquad
    \bbP_{\mathrm{ref}}\text{-a.e.}
\end{equation*}
\end{proposition}

\begin{proof}
For fixed $(x,z)$, the squared loss is minimized by the conditional mean
\begin{equation*}
    \bbE_{\bbP_{\mathrm{ref}}}\left[\tau(x,Y)\mid x,z\right].
\end{equation*}
Therefore, reference consistency gives
\begin{equation*}
    \bbE_{\bbP_{\mathrm{ref}}}\left[\tau(x,Y)\mid x,z\right]
    =
    \bbE_{\base}
    \left[
        \tau(x,Y)\mid x,\ Y\in\calC(z)
    \right]
    =
    V^\star(x,z),
\end{equation*}
which proves the claim.
\end{proof}

By \cref{prop:mc_regression_recovers_value}, under a reference-consistent population rollout law $\bbP_{\mathrm{ref}}$, the population target of MC rollout regression is exactly $V^\star(x,z)$. Thus, $\widehat V$ plays the role of an intermediate-state value function, analogous to the prefix-state value function in \citet{rohatgi2025taming}, but lifted here to any-order partial states. The training procedure is summarized in \cref{alg:aoar_verifier_training}.

\begin{algorithm}[t]
\caption{Training the shared masked-state verifier}
\label{alg:aoar_verifier_training}
\begin{algorithmic}[1]
\Require context distribution $\rho$, reference model $\base$, reward $\tau$, parametric verifier $V_\theta$, snapshots per rollout $K$, number of rollouts $N$
\State Initialize snapshot dataset $\calD_{\mathrm{snap}} \gets \varnothing$
\For{$i=1,\dots,N$}
    \State Sample $x\sim\rho$
    \State Run a reference rollout using $\base(\cdot\mid x)$ to obtain a terminal completion $y$ and snapshots $\mbra{z^{(1)},\dots,z^{(K)}}$
    \State Compute $r\gets\tau(x,y)$
    \State Add $\mbra{(x,z^{(k)},r)}_{k=1}^K$ to $\calD_{\mathrm{snap}}$
\EndFor
\State Update $\theta$ by minimizing
\Statex \hspace{\algorithmicindent}  $\displaystyle
\theta \gets \argmin_{\theta'}\
\frac{1}{|\calD_{\mathrm{snap}}|}
\sum_{(x,z,r)\in\calD_{\mathrm{snap}}}
\sbra{V_{\theta'}(x,z)-r}^2$
\State Define $\widehat V(x,z)\gets V_{\theta}(x,z)$
\State \Return trained verifier $\widehat V$
\end{algorithmic}
\end{algorithm}

\section{AOAR-VGB}
\label{app:sec:aoar_vgb}

We now extend the VGB perspective of \citet{rohatgi2025taming} from left-to-right prefix trees to any-order masked states. The key difference is that the state is no longer a prefix. Instead, a state is a masked partial sequence $z \in \calZ$, and a forward move reveals one currently masked coordinate. A backward move re-masks one currently revealed coordinate. This gives a natural random walk on the graph of masked states.

\subsection{Setup and Notation}
\label{app:subsec:aoar_setup}

We use the shared notation from the previous section. In particular, for a masked state $z \in \calZ$, $R(z)$ denotes the set of revealed coordinates, $k(z)=\abs{R(z)}$ denotes its depth, and $\calC(z)$ denotes the set of full sequences compatible with $z$.

We write $\varnothing$ for the fully masked state, so $R(\varnothing)=\varnothing$. The fully revealed states are exactly the leaves $\calY=\calV^n$.

For an unrevealed coordinate $j\notin R(z)$ and token $a\in\calV$, define the single-site reveal update $z^{j\leftarrow a}$ by
\begin{equation*}
    \sbra{z^{j\leftarrow a}}_i = \begin{cases}
        a, & i=j,\\
        z_i, & i\neq j.
    \end{cases}
\end{equation*}
For a revealed coordinate $i\in R(z)$, define the single-site re-masking update $z^{-i}$ by
\begin{equation*}
    \sbra{z^{-i}}_\ell = \begin{cases}
        \mask, & \ell=i,\\
        z_\ell, & \ell\neq i.
    \end{cases}
\end{equation*}

The AOAR neighbor sets are
\begin{align*}
    C_{\mathrm{AOAR}}(z) &:= \mbra{ z^{j\leftarrow a} \mid j\notin R(z),\ a\in\calV,\ M_{\text{ref}}(x,z^{j\leftarrow a})>0 }, \tag{\text{Child neighborhood}} \\
    P_{\mathrm{AOAR}}(z) &:= \mbra{ z^{-i} \mid i\in R(z) }, \tag{\text{Parent neighborhood}}
\end{align*}
and
\begin{equation*}
    N_{\mathrm{AOAR}}(z) := C_{\mathrm{AOAR}}(z)\cup P_{\mathrm{AOAR}}(z). \tag{\text{AOAR neighborhood}}
\end{equation*}
Thus, $C_{\mathrm{AOAR}}(z)$ contains one-step reveal states, and $P_{\mathrm{AOAR}}(z)$ contains one-step re-masked states.

The AOAR oracle assumption is that for every reachable state $z$ and every $j\notin R(z)$, we can evaluate the single-coordinate reference conditional
\begin{equation*}
    \base(Y_j=a\mid x,z) := \bbP_{\base} \sbra{ Y_j=a \mid x,\ Y\in\calC(z) }, \qquad a\in\calV.
\end{equation*}
Equivalently,
\begin{equation*}
    \base(Y_j=a\mid x,z) = \frac{M_{\text{ref}}(x,z^{j\leftarrow a})}{M_{\text{ref}}(x,z)}.
\end{equation*}

We first record the basic identities that will be used in the weighted-graph construction and the stationary analysis below.

\begin{lemma}[Single-site partition identity]
\label{lem:aoar_partition}
Fix a reachable state $z$ and coordinate $j\notin R(z)$. Then
\begin{equation*}
    \calC(z) = \bigsqcup_{a\in\calV} \calC(z^{j\leftarrow a}).
\end{equation*}
Consequently,
\begin{equation*}
    \sum_{a\in\calV} U^*(x,z^{j\leftarrow a}) = U^*(x,z), \qquad \sum_{a\in\calV} M_{\text{ref}}(x,z^{j\leftarrow a}) = M_{\text{ref}}(x,z).
\end{equation*}
\end{lemma}

\begin{proof}
Every full sequence $y\in\calC(z)$ has a unique value $y_j=a$ at coordinate $j$. Hence $y$ belongs to exactly one set $\calC(z^{j\leftarrow a})$. This proves the disjoint union. Summing the masses $\base(y\mid x)\tau(x,y)$ over this partition gives the identity for $U^*$. Setting $\tau\equiv 1$ gives the identity for $M_{\text{ref}}$.
\end{proof}

\begin{proposition}[AOAR Bellman identity]
\label{prop:aoar_bellman}
For every reachable state $z$ and every $j\notin R(z)$,
\begin{equation*}
    V^*(x,z) = \sum_{a\in\calV} \base(Y_j=a\mid x,z) V^*(x,z^{j\leftarrow a}).
\end{equation*}
\end{proposition}

\begin{proof}
By the law of total expectation,
\begin{equation*}
    V^*(x,z) = \bbE_{\base} \lbra{ \tau(x,Y) \mid x,\ Y\in\calC(z) }.
\end{equation*}
Conditioning further on the value of $Y_j$, we obtain
\begin{equation*}
    V^*(x,z) = \sum_{a\in\calV} \bbP_{\base}(Y_j=a\mid x, Y \in \calC(z)) \bbE_{\base} \lbra{ \tau(x,Y) \mid x,\ Y\in\calC(z),\ Y_j=a }.
\end{equation*}
The event $\mbra{Y\in\calC(z),\ Y_j=a}$ is exactly $\mbra{Y\in\calC(z^{j\leftarrow a})}$. Therefore the inner conditional expectation is $V^*(x,z^{j\leftarrow a})$.
\end{proof}

\begin{proposition}[Exact target conditional]
\label{prop:aoar_exact_target_conditional}
For every reachable state $z$ with $U^* (x,z)>0$, every $j\notin R(z)$, and every $a\in\calV$,
\begin{equation*}
    \pi^*(Y_j=a\mid x,z) = \frac{U^*(x,z^{j\leftarrow a})}{U^*(x,z)} = \base(Y_j=a\mid x,z) \frac{V^*(x,z^{j\leftarrow a})}{V^*(x,z)}.
\end{equation*}
\end{proposition}

\begin{proof}
By definition of conditional probability under the tilted target,
\begin{equation*}
    \pi^*(Y_j=a\mid x,z) = \frac{ \bbP_{\pi^*}(Y\in\calC(z^{j\leftarrow a})\mid x) }{ \bbP_{\pi^*}(Y\in\calC(z)\mid x) }.
\end{equation*}
Since
\begin{equation*}
    \bbP_{\pi^*}(Y\in\calC(z^{j\leftarrow a})\mid x) = \frac{U^*(x,z^{j\leftarrow a})}{Z(x)}
\end{equation*}
and
\begin{equation*}
    \bbP_{\pi^*}(Y\in\calC(z)\mid x) = \frac{U^*(x,z)}{Z(x)},
\end{equation*}
the first equality follows. For the second equality, use
\begin{equation*}
    U^*(x,z)=M_{\text{ref}}(x,z)V^*(x,z)
\end{equation*}
and
\begin{equation*}
    M_{\text{ref}}(x,z^{j\leftarrow a}) = M_{\text{ref}}(x,z) \base(Y_j=a\mid x,z).
\end{equation*}
\end{proof}

\subsection{Primitive AOAR-VGB}
\label{app:subsec:primitive_aoar_vgb}

The most direct AOAR extension of AR-VGB uses the same local rule as the prefix-tree AR-VGB walk: forward moves are weighted by the reference conditional times the verifier value of the child, and backward moves are weighted by the verifier value of the current state.

Let $\widehat V(x,z)$ denote the trained verifier. We use \emph{terminal anchoring}: whenever the terminal reward is available at a leaf, the verifier is fixed to the true reward,
\begin{equation*}
    \widehat V(x,y):=\tau(x,y), \qquad y\in\calY.
\end{equation*}

\begin{definition}[$\gamma$-held primitive AOAR-VGB local kernel]
\label{def:primitive_aoar_weights}
Fix a holding parameter $\gamma\ge 0$. For a reachable state $z$, define the primitive AOAR-VGB local weights by
\begin{alignat*}{3}
    \textsc{Forward:}\quad
    & w(z\to z^{j\leftarrow a}) &&:= \base(Y_j=a\mid x,z)\, \widehat V(x,z^{j\leftarrow a}), \qquad && j\notin R(z),\\
    \textsc{Backward:}\quad
    & w(z\to z^{-i}) &&:= \widehat V(x,z), \qquad && i\in R(z).
\end{alignat*}
Let
\begin{equation*}
    W_{\mathrm{PM}}(z) := \sum_{v\in N_{\mathrm{AOAR}}(z)} w(z\to v).
\end{equation*}
The associated non-self move kernel is
\begin{equation*}
    K_{\mathrm{PM}}(z,u) := \frac{w(z\to u)}{W_{\mathrm{PM}}(z)}, \qquad u\in N_{\mathrm{AOAR}}(z).
\end{equation*}
The $\gamma$-held primitive AOAR-VGB kernel is
\begin{equation*}
    P_{\mathrm{PM}}^{(\gamma)}(z,z)=\frac{\gamma}{1+\gamma},
\end{equation*}
and, for $u\in N_{\mathrm{AOAR}}(z)$,
\begin{equation*}
    P_{\mathrm{PM}}^{(\gamma)}(z,u) = \frac{1}{1+\gamma}K_{\mathrm{PM}}(z,u) = \frac{1}{1+\gamma}\frac{w(z\to u)}{\sum_{v\in N_{\mathrm{AOAR}}(z)}w(z\to v)}.
\end{equation*}
When $\gamma=1$, this recovers the usual $1/2$-lazy kernel. When $\gamma=0$, it is the non-lazy local random walk.
\end{definition}

Although $M_{\text{ref}}$ is not needed to implement the transition, it is useful to temporarily reintroduce it in the analysis. Multiplying the learned value by the reference completion mass $M_{\text{ref}}$  gives the plug-in tilted mass
\begin{equation*}
    \widehat U(x,z) := M_{\text{ref}}(x,z)\widehat V(x,z).
\end{equation*}

\begin{theorem}[Weighted-graph representation and exact leaf law]
\label{thm:primitive_weighted_graph}
Construct an undirected graph on reachable masked states by connecting $z$ and $z^{j\leftarrow a}$ whenever $j\notin R(z)$ and $z^{j\leftarrow a}$ is reachable. Assign the symmetric edge weight
\begin{equation*}
    f_{\mathrm{PM}}(z,z^{j\leftarrow a}) := \widehat U(x,z^{j\leftarrow a}) = M_{\text{ref}}(x,z^{j\leftarrow a})\widehat V(x,z^{j\leftarrow a}).
\end{equation*}
For a state $z$, let its weighted degree be
\begin{equation*}
    D_{\mathrm{PM}}(z) := \sum_{u\sim z} f_{\mathrm{PM}}(z,u),
\end{equation*}
where $u\sim z$ means that $u$ and $z$ are connected by an edge in the AOAR weighted graph.

Then, for every $\gamma\ge 0$, the $\gamma$-held primitive AOAR-VGB kernel $P_{\mathrm{PM}}^{(\gamma)}$ from \cref{def:primitive_aoar_weights} is exactly the $\gamma$-held weighted random walk on this graph. Also, its stationary distribution is
\begin{equation*}
    \mu_{\mathrm{PM}}(z) = \frac{D_{\mathrm{PM}}(z)}{\sum_{z'}D_{\mathrm{PM}}(z')}.
\end{equation*}
Moreover, if $\widehat V(x,y)=\tau(x,y)$ on all leaves $y\in\calY$, then the stationary law conditioned on the leaves is exactly the tilted target:
\begin{equation*}
    \mu_{\mathrm{PM}}(y\mid y\in\calY) = \pi^*(y\mid x).
\end{equation*}
\end{theorem}

\begin{proof}
Fix a reachable state $z$. For a forward neighbor $z^{j\leftarrow a}$, we have
\begin{equation*}
    \frac{f_{\mathrm{PM}}(z,z^{j\leftarrow a})}{M_{\text{ref}}(x,z)} = \frac{M_{\text{ref}}(x,z^{j\leftarrow a})}{M_{\text{ref}}(x,z)}\widehat V(x,z^{j\leftarrow a}) = \base(Y_j=a\mid x,z)\widehat V(x,z^{j\leftarrow a}).
\end{equation*}
This is exactly the forward local weight.

For a backward neighbor $z^{-i}$, the edge between $z^{-i}$ and $z$ has weight
\begin{equation*}
    f_{\mathrm{PM}}(z^{-i},z) = \widehat U(x,z) = M_{\text{ref}}(x,z)\widehat V(x,z).
\end{equation*}
Dividing by the same state-dependent factor $M_{\text{ref}}(x,z)$ gives the backward local weight
\begin{equation*}
    \widehat V(x,z).
\end{equation*}
Thus, from every state $z$, the local weights in \cref{def:primitive_aoar_weights} are exactly the incident edge weights divided by the common positive factor $M_{\text{ref}}(x,z)$. This common factor cancels under normalization. Therefore the non-self move kernel is the weighted random walk on the graph with edge weights $f_{\mathrm{PM}}$:
\begin{equation*}
    K_{\mathrm{PM}}(z,u) = \frac{f_{\mathrm{PM}}(z,u)}{D_{\mathrm{PM}}(z)}, \qquad u\sim z.
\end{equation*}
This is the standard random walk on a network with conductances $f_{\mathrm{PM}}(z,u)$ in the sense of \citet{levin2017markov}.

We now verify stationarity directly. Let
\begin{equation*}
    D_{\mathrm{tot}} := \sum_{z'}D_{\mathrm{PM}}(z'), \qquad \mu_{\mathrm{PM}}(z):=\frac{D_{\mathrm{PM}}(z)}{D_{\mathrm{tot}}}.
\end{equation*}
For two distinct neighboring states $z\sim u$, the $\gamma$-held kernel satisfies
\begin{equation*}
    P_{\mathrm{PM}}^{(\gamma)}(z,u) = \frac{1}{1+\gamma}\frac{f_{\mathrm{PM}}(z,u)}{D_{\mathrm{PM}}(z)}.
\end{equation*}
Therefore
\begin{align*}
    \mu_{\mathrm{PM}}(z)P_{\mathrm{PM}}^{(\gamma)}(z,u)
    &= \frac{D_{\mathrm{PM}}(z)}{D_{\mathrm{tot}}}\cdot \frac{1}{1+\gamma}\frac{f_{\mathrm{PM}}(z,u)}{D_{\mathrm{PM}}(z)} \\
    &= \frac{f_{\mathrm{PM}}(z,u)}{(1+\gamma)D_{\mathrm{tot}}}.
\end{align*}
Since $f_{\mathrm{PM}}(z,u)=f_{\mathrm{PM}}(u,z)$, the same calculation gives
\begin{equation*}
    \mu_{\mathrm{PM}}(u)P_{\mathrm{PM}}^{(\gamma)}(u,z) = \frac{f_{\mathrm{PM}}(z,u)}{(1+\gamma)D_{\mathrm{tot}}}.
\end{equation*}
Thus detailed balance holds on every non-self edge. For the self-loop terms,
\begin{equation*}
    \mu_{\mathrm{PM}}(z)P_{\mathrm{PM}}^{(\gamma)}(z,z) = \mu_{\mathrm{PM}}(z)\frac{\gamma}{1+\gamma},
\end{equation*}
so detailed balance is immediate. Therefore $P_{\mathrm{PM}}^{(\gamma)}$ is reversible with respect to $\mu_{\mathrm{PM}}$, and hence $\mu_{\mathrm{PM}}$ is stationary.

Now consider a leaf $y\in\calY$. It has exactly $n$ parents, one for each coordinate that can be re-masked. For each parent $y^{-i}$, the edge weight is
\begin{equation*}
    f_{\mathrm{PM}}(y^{-i},y) = M_{\text{ref}}(x,y)\widehat V(x,y) = \base(y\mid x)\tau(x,y).
\end{equation*}
Here we used terminal anchoring. Hence
\begin{equation*}
    D_{\mathrm{PM}}(y) = n\,\base(y\mid x)\tau(x,y).
\end{equation*}
The factor $n$ is constant across leaves. Therefore
\begin{equation*}
    \mu_{\mathrm{PM}}(y\mid y\in\calY) = \frac{\base(y\mid x)\tau(x,y)}{\sum_{y'\in\calY}\base(y'\mid x)\tau(x,y')} = \pi^*(y\mid x).
\end{equation*}
\end{proof}

When the verifier is exact, every $\gamma$-held primitive AOAR-VGB chain has a particularly simple stationary law.

\begin{corollary}[Exact primitive stationary law]
\label{cor:primitive_exact_stationary}
For any $\gamma\ge 0$, if $\widehat V=V^*$, then the stationary distribution of $P_{\mathrm{PM}}^{(\gamma)}$ satisfies
\begin{equation*}
    \mu_{\mathrm{PM}}^*(z) \propto U^*(x,z).
\end{equation*}
\end{corollary}

\begin{proof}
When $\widehat V=V^*$, we have $\widehat U=U^*$. For a state $z$ of depth $k=k(z)$, every backward edge contributes $U^*(x,z)$, and there are $k$ such edges. For each unrevealed coordinate $j\notin R(z)$, \cref{lem:aoar_partition} gives
\begin{equation*}
    \sum_{a\in\calV} U^*(x,z^{j\leftarrow a}) = U^*(x,z).
\end{equation*}
There are $n-k$ unrevealed coordinates, so the total forward incident weight is $(n-k)U^*(x,z)$. Hence
\begin{equation*}
    D_{\mathrm{PM}}(z) = kU^*(x,z)+(n-k)U^*(x,z) = nU^*(x,z).
\end{equation*}
By \cref{thm:primitive_weighted_graph}, for every $\gamma\ge 0$,
\begin{equation*}
    \mu_{\mathrm{PM}}^*(z)\propto D_{\mathrm{PM}}(z).
\end{equation*}
Therefore
\begin{equation*}
    \mu_{\mathrm{PM}}^*(z)\propto U^*(x,z).
\end{equation*}
\end{proof}

However, we note that the same stationary law also exposes a depth-allocation problem: although the conditional law on leaves is correct, the primitive chain assigns exponentially small marginal mass to the leaf layer, as follows:

\Needspace{8\baselineskip}
\begin{proposition}[Exponential leaf-mass dilution of primitive AOAR-VGB]
\label{prop:primitive_leaf_dilution}
For any $\gamma\ge 0$, assume $\widehat V=V^*$. Then the stationary mass of depth $k$ under $P_{\mathrm{PM}}^{(\gamma)}$ is
\begin{equation*}
    \mu_{\mathrm{PM}}^*(k(z)=k) = \frac{\binom{n}{k}}{2^n}.
\end{equation*}
In particular,
\begin{equation*}
    \mu_{\mathrm{PM}}^*(\calY) = 2^{-n}.
\end{equation*}
\end{proposition}

\begin{proof}
For any fixed depth $k$, double-counting identity gives
\begin{align*}
    \sum_{z:\,k(z)=k} U^*(x,z) &= \sum_{z:k(z)=k} \sum_{y \in \calC(z)} \base(y \mid x) \tau(x,y) \\
    &= \sum_{y \in \calY} \base(y\mid x) \tau(x,y) \sum_{z : k(z) =k} \mathbf{1}\mbra{y \in \calC(z)}= \binom{n}{k} Z(x).
\end{align*}
Indeed, each full sequence $y\in\calY$ contributes the mass $\base(y\mid x)\tau(x,y)$ to exactly $\binom{n}{k}$ masked states of depth $k$, one for each choice of $k$ revealed coordinates. Therefore
\begin{equation*}
    \sum_{z\in\calZ} U^*(x,z) = \sum_{k=0}^n \sum_{z:k(z)=k} U^*(x,z)= \sum_{k=0}^n \binom{n}{k}Z(x) = 2^n Z(x).
\end{equation*}
By \cref{cor:primitive_exact_stationary}, the primitive stationary law is proportional to $U^*$, so
\begin{equation*}
    \mu_{\mathrm{PM}}^*(k(z)=k) = \frac{\binom{n}{k}Z(x)}{2^n Z(x)} = \frac{\binom{n}{k}}{2^n}.
\end{equation*}
Taking $k=n$ gives $\mu_{\mathrm{PM}}^*(\calY)=2^{-n}$.
\end{proof}

Thus, primitive AOAR-VGB has the correct leaf-conditioned distribution, but the stationary chain spends exponentially little mass on leaves. We therefore introduce Balanced AOAR-VGB, a depth-balanced variant designed to resolve this leaf-mass dilution problem.

\subsection{Balanced AOAR-VGB}
\label{app:subsec:balanced_aoar_vgb}

Primitive AOAR-VGB treats every single-site edge symmetrically. In the AOAR graph, however, the number of states at each depth grows as $\binom{n}{k}$. As a result, the primitive stationary law puts most of its mass around the middle depths. Balanced AOAR-VGB rescales forward and backward moves as a function of depth to suppress this combinatorial depth bias.

\begin{definition}[$\gamma$-held depth-rescaled AOAR-VGB local kernel]
\label{def:balanced_aoar_weights}
Fix a holding parameter $\gamma\ge 0$ and positive coefficients
\begin{equation*}
    \beta_0,\beta_1,\dots,\beta_{n-1}>0, \qquad \alpha_1,\alpha_2,\dots,\alpha_n>0.
\end{equation*}
For a reachable state $z$ with $k=k(z)$, define the depth-rescaled local weights by
\begin{alignat*}{3}
    \textsc{Forward:}\quad
    & w^{(k)}(z\to z^{j\leftarrow a}) &&:= \beta_k\, \base(Y_j=a\mid x,z)\widehat V(x,z^{j\leftarrow a}), \qquad && j\notin R(z),\\
    \textsc{Backward:}\quad
    & w^{(k)}(z\to z^{-i}) &&:= \alpha_k\, \widehat V(x,z), \qquad && i\in R(z).
\end{alignat*}
Let
\begin{equation*}
    W_{\mathrm{BAL}}(z):=\sum_{u\in N_{\mathrm{AOAR}}(z)} w^{(k)}(z\to u).
\end{equation*}
The associated non-self move kernel is
\begin{equation*}
    K_{\mathrm{BAL}}(z,u) := \frac{w^{(k)}(z\to u)}{W_{\mathrm{BAL}}(z)}, \qquad u\in N_{\mathrm{AOAR}}(z).
\end{equation*}
The $\gamma$-held balanced AOAR-VGB kernel is
\begin{align*}
    P_{\mathrm{BAL}}^{(\gamma)}(z,z) &= \frac{\gamma}{1+\gamma},\\
    P_{\mathrm{BAL}}^{(\gamma)}(z,u)
    &= \frac{1}{1+\gamma}K_{\mathrm{BAL}}(z,u)
    = \frac{1}{1+\gamma}\frac{w^{(k)}(z\to u)}{\sum_{v\in N_{\mathrm{AOAR}}(z)}w^{(k)}(z\to v)},
    \quad u\in N_{\mathrm{AOAR}}(z).
\end{align*}
When $\gamma=1$, this recovers the usual $1/2$-lazy kernel. When $\gamma=0$, it is the non-lazy local random walk.
\end{definition}

In the \cref{prop:balanced_directional_masses} below, we use the word ``move'' to mean a non-self transition. Thus, conditioning on $\mathrm{move}$ means conditioning on the event that the kernel does not take the holding self-loop and instead chooses a neighbor in $N_{\mathrm{AOAR}}(z)$. The events $\mathrm{forward}$ and $\mathrm{backward}$ mean that the selected neighbor lies in $C_{\mathrm{AOAR}}(z)$ or $P_{\mathrm{AOAR}}(z)$, respectively. Since the holding probability is common to all non-self moves, the conditional directional probabilities below do not depend on $\gamma$.

\begin{proposition}[Exact directional masses and balanced condition]
\label{prop:balanced_directional_masses}
Assume $\widehat V=V^*$. Let $z$ be a reachable state of interior depth $1\le k\le n-1$ with $V^*(x,z)>0$. Then the exact total forward and backward local masses are
\begin{equation*}
    W_{\mathrm{fwd}}^{(k),*}(x,z) = (n-k)\beta_k V^*(x,z), \qquad
    W_{\mathrm{bwd}}^{(k),*}(x,z) = k\alpha_k V^*(x,z).
\end{equation*}
Therefore, conditional on making a non-self move,
\begin{align*}
    \bbP(\mathrm{forward}\mid x,z,\mathrm{move}) &= \frac{(n-k)\beta_k}{(n-k)\beta_k+k\alpha_k},\\
    \bbP(\mathrm{backward}\mid x,z,\mathrm{move}) &= \frac{k\alpha_k}{(n-k)\beta_k+k\alpha_k}.
\end{align*}
In particular, the forward and backward directions are exactly balanced at depth $k$ if and only if
\begin{equation*}
    (n-k)\beta_k = k\alpha_k.
\end{equation*}
\end{proposition}

\begin{proof}
For each unrevealed coordinate $j\notin R(z)$, \cref{prop:aoar_bellman} gives
\begin{equation*}
    \sum_{a\in\calV} \base(Y_j=a\mid x,z) V^*(x,z^{j\leftarrow a}) = V^*(x,z).
\end{equation*}
There are $n-k$ unrevealed coordinates, so the total forward mass is
\begin{align*}
    W_{\mathrm{fwd}}^{(k),*}(x,z)
    &=
    \sum_{j\notin R(z)}\sum_{a\in\calV}
    \beta_k\base(Y_j=a\mid x,z)V^*(x,z^{j\leftarrow a})
    =
    (n-k)\beta_k V^*(x,z).
\end{align*}
There are $k$ revealed coordinates, and each backward move has weight $\alpha_k V^*(x,z)$. Hence
\begin{equation*}
    W_{\mathrm{bwd}}^{(k),*}(x,z) = \sum_{i\in R(z)}\alpha_k V^*(x,z) = k\alpha_k V^*(x,z).
\end{equation*}
The conditional directional probabilities follow by normalizing the forward and backward masses among non-self moves. The common holding factor $1/(1+\gamma)$ cancels after conditioning on $\mathrm{move}$. Exact balance is therefore equivalent to
\begin{equation*}
    (n-k)\beta_k V^*(x,z)=k\alpha_k V^*(x,z),
\end{equation*}
or simply $(n-k)\beta_k=k\alpha_k$.
\end{proof}

\begin{remark}[Primitive AOAR-VGB as the unrescaled case]
\label{rem:primitive_vs_balanced_directional_bias}
Primitive AOAR-VGB corresponds to the unrescaled choice
\begin{equation*}
    \beta_k=1,\qquad \alpha_k=1
\end{equation*}
at every interior depth. Under this choice, \cref{prop:balanced_directional_masses} gives
\begin{equation*}
    \bbP(\mathrm{forward}\mid x,z,\mathrm{move})=\frac{n-k}{n}, \qquad
    \bbP(\mathrm{backward}\mid x,z,\mathrm{move})=\frac{k}{n}.
\end{equation*}
Thus, the primitive chain strongly favors forward moves near the root and strongly favors backward moves near the leaves.

This local depth bias is consistent with the global pathology in \cref{prop:primitive_leaf_dilution}: under the exact primitive stationary law, the depth marginal is binomial and the leaf layer receives only $2^{-n}$ stationary mass. Balanced AOAR-VGB removes this local directional bias by choosing $\alpha_k,\beta_k$ so that
\begin{equation*}
    (n-k)\beta_k=k\alpha_k,
\end{equation*}
making the forward and backward directions equally likely after conditioning on a non-self move.
\end{remark}

We next give the stationary law for the balanced chain. The following edge-coefficient representation is the cleanest way to see that terminal exactness is preserved even when $\widehat V$ is imperfect.

Let $g_1>0$ be arbitrary, and define $g_2,\dots,g_n$ recursively by
\begin{equation*}
    g_{k+1} = g_k \frac{\beta_k}{\alpha_k}, \qquad k=1,\dots,n-1.
\end{equation*}
We set $\beta_0=\alpha_n=1$. These boundary choices do not affect the normalized transition, since only forward moves are available at the root and only backward moves are available at the leaves.

\begin{theorem}[Weighted-graph representation of balanced AOAR-VGB]
\label{thm:balanced_weighted_graph}
Construct an undirected graph on reachable masked states with single-site AOAR edges. For an edge between a depth-$k$ parent $z$ and its child $z^{j\leftarrow a}$, assign the symmetric edge weight
\begin{equation*}
    f_{\mathrm{BAL}}(z,z^{j\leftarrow a}) := g_{k+1}\, M_{\text{ref}}(x,z^{j\leftarrow a}) \widehat V(x,z^{j\leftarrow a}).
\end{equation*}
Then, for every $\gamma\ge 0$, the balanced AOAR-VGB kernel from \cref{def:balanced_aoar_weights} is exactly the $\gamma$-held weighted random walk on this graph. Hence its stationary distribution is
\begin{equation*}
    \mu_{\mathrm{BAL}}(z)=\frac{D_{\mathrm{BAL}}(z)}{\sum_{z'}D_{\mathrm{BAL}}(z')}, \qquad
    D_{\mathrm{BAL}}(z) := \sum_{u\sim z} f_{\mathrm{BAL}}(z,u).
\end{equation*}
If $\widehat V(x,y)=\tau(x,y)$ for all leaves $y\in\calY$, then
\begin{equation*}
    \mu_{\mathrm{BAL}}(y\mid y\in\calY) = \pi^*(y\mid x).
\end{equation*}
\end{theorem}

\begin{proof}
Let $z$ be a state of depth $k$. First consider a forward edge $z\to z^{j\leftarrow a}$. Then
\begin{align*}
    \frac{f_{\mathrm{BAL}}(z,z^{j\leftarrow a})}{M_{\text{ref}}(x,z)}
    &=
    g_{k+1}
    \frac{M_{\text{ref}}(x,z^{j\leftarrow a})}{M_{\text{ref}}(x,z)}
    \widehat V(x,z^{j\leftarrow a}) \\
    &=
    g_{k+1}\base(Y_j=a\mid x,z)\widehat V(x,z^{j\leftarrow a}).
\end{align*}
For a backward edge from $z$ to $z^{-i}$, the child in the edge is $z$, so
\begin{equation*}
    \frac{f_{\mathrm{BAL}}(z^{-i},z)}{M_{\text{ref}}(x,z)}=g_k \widehat V(x,z).
\end{equation*}
For an interior state, the recursion
\begin{equation*}
    \frac{g_{k+1}}{g_k} = \frac{\beta_k}{\alpha_k}
\end{equation*}
implies $g_{k+1}/\beta_k=g_k/\alpha_k$. Thus, at every interior depth $k$, the incident edge weights divided by $M_{\text{ref}}(x,z)$ are proportional to the local depth-rescaled weights in \cref{def:balanced_aoar_weights}. The proportionality factor is common across all outgoing neighbors of $z$, and hence cancels under normalization.

At the root and leaves, we use the boundary convention $\beta_0=\alpha_n=1$. Since only forward moves are available at the root and only backward moves are available at the leaves, these boundary choices do not change the normalized transition. Therefore $P_{\mathrm{BAL}}^{(\gamma)}$ is exactly the $\gamma$-held weighted random walk on the graph with edge weights $f_{\mathrm{BAL}}$:
\begin{equation*}
    P_{\mathrm{BAL}}^{(\gamma)}(z,z)=\frac{\gamma}{1+\gamma}, \qquad
    P_{\mathrm{BAL}}^{(\gamma)}(z,u)=\frac{1}{1+\gamma}\frac{f_{\mathrm{BAL}}(z,u)}{D_{\mathrm{BAL}}(z)}, \quad u\sim z.
\end{equation*}

We now verify the stationary distribution directly. This is the standard random walk on a network with conductances $f_{\mathrm{BAL}}(z,u)$, for which the stationary distribution is proportional to the vertex conductance $D_{\mathrm{BAL}}(z)$ \citep{levin2017markov}. For completeness, we check detailed balance. Let
\begin{equation*}
    D_{\mathrm{tot}}:=\sum_{z'}D_{\mathrm{BAL}}(z'), \qquad
    \mu_{\mathrm{BAL}}(z):=\frac{D_{\mathrm{BAL}}(z)}{D_{\mathrm{tot}}}.
\end{equation*}
For two distinct neighboring states $z\sim u$,
\begin{align*}
    \mu_{\mathrm{BAL}}(z)P_{\mathrm{BAL}}^{(\gamma)}(z,u)
    =
    \frac{D_{\mathrm{BAL}}(z)}{D_{\mathrm{tot}}}
    \cdot
    \frac{1}{1+\gamma}
    \frac{f_{\mathrm{BAL}}(z,u)}{D_{\mathrm{BAL}}(z)} =
    \frac{f_{\mathrm{BAL}}(z,u)}{(1+\gamma)D_{\mathrm{tot}}}.
\end{align*}
Since the edge weight is symmetric, $f_{\mathrm{BAL}}(z,u)=f_{\mathrm{BAL}}(u,z)$, the same calculation gives
\begin{equation*}
    \mu_{\mathrm{BAL}}(u)P_{\mathrm{BAL}}^{(\gamma)}(u,z)=\frac{f_{\mathrm{BAL}}(z,u)}{(1+\gamma)D_{\mathrm{tot}}}.
\end{equation*}
Thus, detailed balance holds on every non-self edge. For the self-loop terms,
\begin{equation*}
    \mu_{\mathrm{BAL}}(z)P_{\mathrm{BAL}}^{(\gamma)}(z,z)=\mu_{\mathrm{BAL}}(z)\frac{\gamma}{1+\gamma},
\end{equation*}
so detailed balance is trivial. Therefore $P_{\mathrm{BAL}}^{(\gamma)}$ is reversible with respect to $\mu_{\mathrm{BAL}}$, and hence $\mu_{\mathrm{BAL}}$ is stationary.

Now let $y\in\calY$ be a leaf. Each of its $n$ incident edges has weight
\begin{equation*}
    g_n M_{\text{ref}}(x,y)\widehat V(x,y)=g_n \base(y\mid x)\tau(x,y).
\end{equation*}
Thus
\begin{equation*}
    D_{\mathrm{BAL}}(y)=n g_n \base(y\mid x)\tau(x,y).
\end{equation*}
The factor $n g_n$ is constant over leaves, so conditioning on $\calY$ gives
\begin{equation*}
    \mu_{\mathrm{BAL}}(y\mid y\in\calY)
    =
    \frac{\base(y\mid x)\tau(x,y)}
    {\sum_{y'\in\calY}\base(y'\mid x)\tau(x,y')}
    =
    \pi^*(y\mid x).
\end{equation*}
\end{proof}

\begin{remark}[Zero holding in the practical sampler]
\label{rem:balanced_gamma_zero}
\quad The holding parameter $\gamma$ is useful for connecting the balanced AOAR-VGB transition to the standard lazy random-walk formalism, but explicit self-loop steps have limited algorithmic value in implementation and directly add computational cost.

Following the practical convention in \citet{rohatgi2025taming}, we therefore set $\gamma=0$ in our AOAR-VGB implementation and use the non-lazy local move kernel. The stationarity argument in \cref{thm:balanced_weighted_graph} is stated for every $\gamma\ge 0$, so this practical choice does not change the leaf-conditioned stationary law.
\end{remark}

When the verifier is exact, the weighted-degree expression from \cref{thm:balanced_weighted_graph} simplifies to the tilted mass $U^*$ multiplied by a depth-dependent coefficient.

\begin{corollary}[Exact balanced stationary law]
\label{cor:balanced_exact_stationary}
Assume $\widehat V=V^*$. Define
\begin{equation*}
    d_0:=n g_1, \qquad d_n:=n g_n, \qquad
    d_k := k g_k+(n-k)g_{k+1}\quad (1\le k\le n-1).
\end{equation*}
Then, for every $\gamma\ge 0$, the exact balanced stationary law satisfies
\begin{equation*}
    \mu_{\mathrm{BAL}}^*(z) \propto d_{k(z)}\,U^*(x,z).
\end{equation*}
\end{corollary}

\begin{proof}
Let $z$ have depth $k$. If $1\le k\le n-1$, its backward incident edges contribute
\begin{equation*}
    k g_k U^*(x,z).
\end{equation*}
For each unrevealed coordinate $j\notin R(z)$, \cref{lem:aoar_partition} gives
\begin{equation*}
    \sum_{a\in\calV} g_{k+1}U^*(x,z^{j\leftarrow a}) = g_{k+1}U^*(x,z).
\end{equation*}
There are $n-k$ unrevealed coordinates, so the forward incident weight is
\begin{equation*}
    (n-k)g_{k+1}U^*(x,z).
\end{equation*}
Thus
\begin{equation*}
    D_{\mathrm{BAL}}^*(z)=\sbra{k g_k+(n-k)g_{k+1}}U^*(x,z).
\end{equation*}
The root and leaf cases give $d_0=n g_1$ and $d_n=n g_n$, respectively. By \cref{thm:balanced_weighted_graph}, the stationary distribution is proportional to $D_{\mathrm{BAL}}^*$ for every $\gamma\ge 0$, which proves the claim.
\end{proof}

\begin{corollary}[Canonical balanced AOAR-VGB]
\label{cor:canonical_balanced_aoar}
Set
\begin{equation*}
    \beta_k=k, \qquad \alpha_k=n-k, \qquad 1\le k\le n-1.
\end{equation*}
Then, the exact forward and backward directional probabilities are both $1/2$ at every interior depth. Choosing $g_1=1$, the corresponding edge coefficients are
\begin{equation*}
    g_k = \frac{1}{\binom{n-1}{k-1}}, \qquad 1\le k\le n.
\end{equation*}
Moreover, under $\widehat V=V^*$, for every $\gamma\ge 0$,
\begin{equation*}
    \mu_{\mathrm{BAL}}^*(\calY) = \frac{1}{2n}.
\end{equation*}
\end{corollary}

\begin{proof}
The balanced condition follows immediately:
\begin{equation*}
    (n-k)\beta_k = (n-k)k = k(n-k) = k\alpha_k.
\end{equation*}
The recursion for $g_k$ becomes
\begin{equation*}
    g_{k+1} = g_k\frac{k}{n-k}.
\end{equation*}
Starting from $g_1=1$, this gives
\begin{equation*}
    g_k=\prod_{r=1}^{k-1}\frac{r}{n-r}
    =
    \frac{(k-1)!(n-k)!}{(n-1)!}
    =
    \frac{1}{\binom{n-1}{k-1}}.
\end{equation*}

For $1\le k\le n-1$,
\begin{equation*}
    d_k = k g_k+(n-k)g_{k+1} = k g_k+k g_k = 2k g_k = \frac{2n}{\binom{n}{k}}.
\end{equation*}
At the boundaries,
\begin{equation*}
    d_0=d_n=n.
\end{equation*}
Using the double-counting identity
\begin{align*}
    \sum_{z:\,k(z)=k}U^*(x,z)
    &=
    \sum_{z:\,k(z)=k}\sum_{y\in\calC(z)}\base(y\mid x)\tau(x,y) \\
    &=
    \sum_{y\in\calY}\base(y\mid x)\tau(x,y)
    \sum_{z:\,k(z)=k}\mathbf{1}\mbra{y\in\calC(z)} =
    \binom{n}{k}Z(x),
\end{align*}
we can compute the unnormalized stationary weight layer by layer. For the root layer, $\sum_{z:\,k(z)=0}U^*(x,z)=Z(x)$ and $d_0=n$, so its total stationary weight is $nZ(x)$. For the leaf layer, $\sum_{z:\,k(z)=n}U^*(x,z)=Z(x)$ and $d_n=n$, so its total stationary weight is also $nZ(x)$. For each interior depth $1\le k\le n-1$, the total stationary weight is
\begin{equation*}
    d_k\sum_{z:\,k(z)=k}U^*(x,z)=\frac{2n}{\binom{n}{k}}\binom{n}{k}Z(x)=2nZ(x).
\end{equation*}
Hence the total weight over all depths is
\begin{equation*}
    nZ(x)+nZ(x)+(n-1)2nZ(x)=2n^2Z(x).
\end{equation*}
The leaf layer has weight $nZ(x)$, so
\begin{equation*}
    \mu_{\mathrm{BAL}}^*(\calY)=\frac{nZ(x)}{2n^2Z(x)}=\frac{1}{2n}.
\end{equation*}
\end{proof}

Thus, canonical balanced AOAR-VGB keeps the exact target law on leaves while increasing the stationary leaf mass from $2^{-n}$ to order $1/n$.

\subsection{Theoretical Framework and Inference Algorithms}
\label{app:subsec:aoar_theory_algorithms}

We now state the mixing framework for canonical balanced AOAR-VGB. The argument follows the same high-level principle as VGB on a prefix tree: stochastic backtracking defines a reversible walk whose leaf-conditioned stationary law is exact, and whose conductance can be controlled under a multiplicative value-approximation assumption.

\begin{assumption}[Uniform multiplicative verifier accuracy]
\label{ass:aoar_kappa_accuracy}
There exists $\kappa\ge 1$ such that for every reachable non-leaf state $z$,
\begin{equation*}
    \kappa^{-1}V^*(x,z)\le \widehat V(x,z)\le \kappa V^*(x,z).
\end{equation*}
At leaves, we use terminal anchoring:
\begin{equation*}
\label{eq:terminal_anchoring}
    \widehat V(x,y)=\tau(x,y), \qquad y\in\calY.
\end{equation*}
\end{assumption}

The canonical balanced AOAR graph can be viewed as an average of fixed-order VGB trees. Let $\mathsf{Perm}_n$ denote the set of coordinate orders, i.e., permutations of $[n]$. For $\omega=(\omega_1,\dots,\omega_n)\in\mathsf{Perm}_n$, define the fixed-order state set
\begin{equation*}
    \calZ_\omega := \mbra{ z\in\calZ \mid R(z)=\mbra{\omega_1,\dots,\omega_k} \text{ for some } k\in\mbra{0,\dots,n} }.
\end{equation*}
This is a tree whose forward move at depth $k$ reveals coordinate $\omega_{k+1}$. On this tree, define the edge weight
\begin{equation*}
    f_\omega(z,z^{\omega_{k+1}\leftarrow a}) := M_{\text{ref}}(x,z^{\omega_{k+1}\leftarrow a}) \widehat V(x,z^{\omega_{k+1}\leftarrow a}).
\end{equation*}

\begin{lemma}[Fixed-order tree conductance]
\label{lem:fixed_order_conductance}
We call
\begin{equation*}
    D_\omega(z):=\sum_{u\sim z}f_\omega(z,u), \qquad
    \mu_\omega(z):=\frac{D_\omega(z)}{\sum_{z'}D_\omega(z')}.
\end{equation*}
the weighted degree and the degree-normalized stationary law, respectively. For any nonempty set $A\subsetneq\calZ_\omega$, define its set conductance and the tree conductance by
\begin{equation*}
    \Phi_\omega(A):=\frac{\sum_{z\in A,\ u\notin A}f_\omega(z,u)}{\sum_{z\in A}D_\omega(z)}, \qquad
    \Phi_\omega:=\inf_{\substack{A\subseteq\calZ_\omega\\0<\mu_\omega(A)\le 1/2}}\Phi_\omega(A).
\end{equation*}
Assuming \cref{ass:aoar_kappa_accuracy}, the random walk on each fixed-order tree $\calZ_\omega$ satisfies
\begin{equation*}
    \Phi_\omega \ge \frac{1}{4\kappa^2 n}.
\end{equation*}
Furthermore, its stationary leaf mass satisfies
\begin{equation*}
    \mu_\omega(\calY) = \sum_{y\in\calY}\mu_\omega(y) \ge \frac{1}{4\kappa n}.
\end{equation*}
\end{lemma}

\begin{proof}
Write
\begin{equation*}
    \widehat U(x,z) := M_{\text{ref}}(x,z)\widehat V(x,z).
\end{equation*}
For non-leaf states, \cref{ass:aoar_kappa_accuracy} gives
\begin{equation*}
    \kappa^{-1}U^*(x,z) \le \widehat U(x,z) \le \kappa U^*(x,z).
\end{equation*}
On leaves, terminal anchoring gives $\widehat U(x,y)=U^*(x,y)$.

For a fixed-order tree, the weighted degree $D_\omega(z)$ is the sum of the parent edge and the child edges incident to $z$. If $k=k(z)$, then
\begin{equation*}
    D_\omega(z)=\mathbf{1}\mbra{k>0}\widehat U(x,z)+\mathbf{1}\mbra{k<n}\sum_{a\in\calV}\widehat U(x,z^{\omega_{k+1}\leftarrow a}).
\end{equation*}
Using $\widehat U\le \kappa U^*$ and the partition identity for the next reveal coordinate,
\begin{align*}
\label{eq:fixed_order_pointwise_degree_bound}
    D_\omega(z)
    &\le \kappa \mathbf{1}\mbra{k>0}U^*(x,z)+\kappa \mathbf{1}\mbra{k<n}\sum_{a\in\calV}U^*(x,z^{\omega_{k+1}\leftarrow a}) \nonumber \\
    &= \kappa\bigl(\mathbf{1}\mbra{k>0}+\mathbf{1}\mbra{k<n}\bigr)U^*(x,z) \le 2\kappa U^*(x,z).
\end{align*}

Let $D_{\omega,\mathrm{tot}}:=\sum_{z\in\calZ_\omega}D_\omega(z)$ denote the total unnormalized stationary weight of the fixed-order tree. Then
\begin{equation*}
    D_{\omega,\mathrm{tot}}\le 2\kappa \sum_{k=0}^n \sum_{z:\,k(z)=k,\ z\in\calZ_\omega} U^*(x,z).
\end{equation*}
For each fixed depth in a fixed-order tree, the compatible states partition the leaves, so
\begin{equation*}
    \sum_{z:\,k(z)=k,\ z\in\calZ_\omega} U^*(x,z) = Z(x).
\end{equation*}
Hence $D_{\omega,\mathrm{tot}}\le 2\kappa(n+1)Z(x)$. Let
\begin{equation*}
    D_{\omega,\mathrm{leaf}}:=\sum_{y\in\calY}D_\omega(y)
\end{equation*}
denote the total degree of the leaf layer. In the fixed-order tree, each leaf $y$ has exactly one incident non-self edge, namely its parent edge, whose weight is $\widehat U(x,y)$. From the terminal anchoring in \cref{ass:aoar_kappa_accuracy},
\begin{equation*}
    D_{\omega,\mathrm{leaf}}=\sum_{y\in\calY}U^*(x,y)=Z(x).
\end{equation*}
Since the stationary distribution of the fixed-order weighted walk is degree-normalized,
\begin{equation*}
    \mu_\omega(\calY)=\frac{D_{\omega,\mathrm{leaf}}}{D_{\omega,\mathrm{tot}}}\ge \frac{Z(x)}{2\kappa(n+1)Z(x)}=\frac{1}{2\kappa(n+1)}\ge \frac{1}{4\kappa n},
\end{equation*}
where the last inequality uses $n\ge 1$.

We next lower bound the conductance. Let $A$ be a connected subset of the fixed-order tree that does not contain the root. Let $v\in A$ be the unique minimum-depth node in $A$. Since $A$ does not contain the root, $v$ has a parent $p(v)$. By the minimality of $k(v)$, we have $p(v)\notin A$, and hence the edge $\mbra{p(v),v}$ leaves $A$. By \cref{ass:aoar_kappa_accuracy}, its weight is
\begin{equation*}
    \widehat U(x,v) \ge \kappa^{-1}U^*(x,v).
\end{equation*}
For $\ell \in \mbra{0,\dots,n-k(v)}$, let
\begin{equation*}
    C_\omega^\ell(v):=\mbra{u\in\calZ_\omega \mid u \text{ is a descendant of } v,\ k(u)=k(v)+\ell}
\end{equation*}
be the set of descendants of $v$ at relative depth $\ell$. Define the total weighted degree of the descendant subtree of $v$ by
\begin{equation*}
    D_{\omega,\mathrm{sub}}(v):=\sum_{\ell=0}^{n-k(v)} \sum_{u\in C_\omega^\ell(v)} D_\omega(u).
\end{equation*}
Applying the pointwise degree bound above to the descendant subtree of $v$, we have
\begin{equation*}
    D_{\omega,\mathrm{sub}}(v)\le 2\kappa \sum_{\ell=0}^{n-k(v)} \sum_{u\in C_\omega^\ell(v)} U^*(x,u).
\end{equation*}
For every descendant level $\ell$,
\begin{align*}
    \sum_{u\in C_\omega^\ell(v)}U^*(x,u) &=
    \sum_{u\in C_\omega^\ell(v)}\sum_{y\in\calC(u)}\base(y\mid x)\tau(x,y) = \sum_{y\in\calC(v)}\base(y\mid x)\tau(x,y) = U^*(x,v),
\end{align*}
because the completion sets $\mbra{\calC(u):u\in C_\omega^\ell(v)}$ partition $\calC(v)$. Therefore
\begin{equation*}
    D_{\omega,\mathrm{sub}}(v)\le 2\kappa(n-k(v)+1)U^*(x,v)\le 2\kappa(n+1)U^*(x,v).
\end{equation*}
Since $A$ is connected in the tree and $v$ is its unique minimum-depth node, $A$ is contained in the descendant subtree of $v$. Hence $\sum_{z\in A}D_\omega(z)\le D_{\omega,\mathrm{sub}}(v)$. The cut numerator $\sum_{z\in A,\ u\notin A}f_\omega(z,u)$ in the tree conductance $\Phi_\omega(A)$ contains the single boundary edge from $v$ to its parent, because $v\in A$ and $p(v)\notin A$. Therefore
\begin{equation*}
    \sum_{z\in A,\ u\notin A}f_\omega(z,u)\ge f_\omega(v,p(v))=f_\omega(p(v),v)=\widehat U(x,v)\ge \kappa^{-1}U^*(x,v).
\end{equation*}
Combining this boundary lower bound with the descendant-subtree volume upper bound gives
\begin{equation*}
    \Phi_\omega(A):= \frac{\sum_{z\in A,\ u\notin A}f_\omega(z,u)}{\sum_{z\in A}D_\omega(z)} \ge \frac{\kappa^{-1}U^*(x,v)}{2\kappa(n+1)U^*(x,v)}=\frac{1}{2\kappa^2(n+1)}\ge \frac{1}{4\kappa^2 n}.
\end{equation*}
This proves the bound for connected sets that do not contain the root.

If $A$ is disconnected and does not contain the root, decompose it into connected components $A_1,\dots,A_m$. Let
\begin{equation*}
    B_r:=\sum_{z\in A_r,\ u\notin A_r}f_\omega(z,u), \qquad V_r:=\sum_{z\in A_r}D_\omega(z).
\end{equation*}
Since the fixed-order graph is a tree, there are no edges between distinct components of $A$, and therefore
\begin{equation*}
    \Phi_\omega(A)=\frac{\sum_{r=1}^m B_r}{\sum_{r=1}^m V_r}=\sum_{r=1}^m \frac{V_r}{\sum_{s=1}^m V_s}\Phi_\omega(A_r)\ge \min_{r\in\lbra{m}}\Phi_\omega(A_r).
\end{equation*}
If an admissible set $A$ contains the root, then $A^c$ does not. Since $\mu_\omega(A)\le 1/2$, we have $\sum_{z\in A}D_\omega(z)\le \sum_{z\in A^c}D_\omega(z)$, and therefore $\Phi_\omega(A)\ge \Phi_\omega(A^c)$. Applying the previous argument to the connected components of $A^c$ proves $\Phi_\omega\ge 1/(4\kappa^2 n)$.
\end{proof}

The next two statements show that conductance survives summing weighted graphs, and that canonical balanced AOAR is exactly such a sum over fixed coordinate orders.

\begin{lemma}[Conductance of a sum of weighted graphs]
\label{lem:sum_graph_conductance}
Let $G_1,\dots,G_m$ be weighted graphs on the same vertex set with nonnegative symmetric weights $f_{G_1},\dots,f_{G_m}$. Let $G$ be the weighted graph with edge weight
\begin{equation*}
    f_G(z,u) := \sum_{r=1}^m f_{G_r}(z,u).
\end{equation*}
Then for every set $A$,
\begin{equation*}
    \Phi_G(A) \ge \min_{r\in\lbra{m}}\Phi_{G_r}(A).
\end{equation*}
Consequently,
\begin{equation*}
    \Phi_G \ge \min_{r\in\lbra{m}}\Phi_{G_r}.
\end{equation*}
\end{lemma}

\begin{proof}
Write
\begin{equation*}
    D_{G_r}(z):=\sum_{u\sim_{G_r} z} f_{G_r}(z,u), \qquad D_G(z):=\sum_{u\sim_G z} f_G(z,u)=\sum_{r=1}^m D_{G_r}(z).
\end{equation*}
For any set $A$,
\begin{equation*}
    \Phi_G(A)=\frac{\sum_{z\in A,\ u\notin A}f_G(z,u)}{\sum_{z\in A}D_G(z)}.
\end{equation*}
Using $f_G=\sum_{r=1}^m f_{G_r}$ and $D_G=\sum_{r=1}^m D_{G_r}$, this becomes
\begin{equation*}
    \Phi_G(A)=\frac{\sum_{r=1}^m\sum_{z\in A,\ u\notin A} f_{G_r}(z,u)}{\sum_{r=1}^m\sum_{z\in A}D_{G_r}(z)}.
\end{equation*}
We use the elementary fact that for $a_r\ge 0$ and $b_r>0$,
\begin{equation*}
    \frac{\sum_{r=1}^m a_r}{\sum_{r=1}^m b_r}\ge \min_{r\in\lbra{m}}\frac{a_r}{b_r}.
\end{equation*}
Therefore
\begin{equation*}
    \Phi_G(A)\ge \min_{r\in\lbra{m}}\frac{\sum_{z\in A,\ u\notin A} f_{G_r}(z,u)}{\sum_{z\in A}D_{G_r}(z)}=\min_{r\in\lbra{m}}\Phi_{G_r}(A).
\end{equation*}
Taking the infimum over admissible $A$ gives the conductance bound.
\end{proof}

\begin{proposition}[Canonical balanced AOAR as a sum of fixed-order VGB trees]
\label{prop:canonical_as_sum_of_trees}
Consider the canonical balanced AOAR-VGB weighted graph whose depth-$k$ parent-child edge coefficient is
\begin{equation*}
    g_{k+1}=\frac{1}{\binom{n-1}{k}}, \qquad 0\le k\le n-1.
\end{equation*}
Then, the sum of the fixed-order tree graphs over all $\omega\in\mathsf{Perm}_n$ is exactly $(n-1)!$ times this canonical balanced AOAR-VGB weighted graph.
\end{proposition}

\begin{proof}
Consider an AOAR edge between a depth-$k$ state $z$ and a child $z^{j\leftarrow a}$, where $j\notin R(z)$. This edge appears in exactly those fixed orders $\omega$ for which the first $k$ coordinates are precisely $R(z)$, in any order, and the $(k+1)$-st coordinate is $j$. The number of such permutations is
\begin{equation*}
    k!(n-k-1)!.
\end{equation*}
Therefore, in the sum of fixed-order tree graphs, the coefficient multiplying $M_{\text{ref}}(x,z^{j\leftarrow a})\widehat V(x,z^{j\leftarrow a})$ is
\begin{equation*}
    k!(n-k-1)! = \frac{(n-1)!}{\binom{n-1}{k}}.
\end{equation*}
This is exactly $(n-1)!g_{k+1}$ for the canonical balanced coefficient in the proposition.
\end{proof}

Combining the fixed-order conductance bound with the sum-of-graphs representation gives the main canonical balanced mixing guarantee.

\begin{theorem}[Mixing guarantee for canonical balanced AOAR-VGB]
\label{thm:canonical_balanced_mixing}
Assume \cref{ass:aoar_kappa_accuracy}. Let $P_{\mathrm{BAL}}$ be the canonical balanced AOAR-VGB lazy kernel, and let $\mu_{\mathrm{BAL}}$ be its stationary distribution. For a set $A$, write
\begin{equation*}
    Q_{\mathrm{BAL}}(A,A^c):=\sum_{z\in A}\mu_{\mathrm{BAL}}(z)P_{\mathrm{BAL}}(z,A^c),
    \qquad
    \Phi(P_{\mathrm{BAL}})
    :=
    \min_{\substack{A\subseteq\calZ_{\mathrm{reach}}\\0<\mu_{\mathrm{BAL}}(A)\le1/2}}
    \frac{Q_{\mathrm{BAL}}(A,A^c)}{\mu_{\mathrm{BAL}}(A)}
\end{equation*}
for conductance. We write
\begin{equation*}
    \gamma(P_{\mathrm{BAL}}):=1-\max\{|\lambda|:\lambda\in \Lambda(P_{\mathrm{BAL}}),\ \lambda\ne1\}
\end{equation*}
for the spectral gap, where $\Lambda (P_{\mathrm{BAL}})$ denotes the set of eigenvalues of $P_{\mathrm{BAL}}$. Then
\begin{align*}
    \Phi(P_{\mathrm{BAL}})&=\Omega\sbra{\frac{1}{\kappa^2 n}},&
    \gamma(P_{\mathrm{BAL}})&=\Omega\sbra{\frac{1}{\kappa^4 n^2}},&
    \mu_{\mathrm{BAL}}(\calY)&=\Omega\sbra{\frac{1}{\kappa n}}.
\end{align*}
Moreover, the stationary law conditioned on the leaves is exactly
\begin{equation*}
    \mu_{\mathrm{BAL}}(\cdot\mid\calY) = \pi^*(\cdot\mid x).
\end{equation*}
Consequently, if $Z_T$ is the state after $T$ lazy balanced AOAR-VGB Markov chain steps initialized at $Z_0 \in \calZ$ with $\mu_{\mathrm{BAL}}(Z_0) > 0$, then for any $\delta>0$,
\begin{equation*}
    T
    =
    O\!\left(
        \kappa^4 n^2
        \log
        \frac{\kappa n}{\delta \mu_{\mathrm{BAL}}(Z_0)}
    \right)
\end{equation*}
suffices to guarantee
\begin{equation*}
    d_{\mathrm{TV}}
    \left(
        \mathcal L(Z_T\mid Z_T\in\calY),
        \pi^*(\cdot\mid x)
    \right)
    \le
    \delta \quad \text{and} \quad \mathbb{P}[Z_T \in \calY] =\Omega\sbra{\frac{1}{\kappa n}} \footnote{$\mathcal{L}(Z_T)$ denote the law of $Z_T$ and $d_{\mathrm{TV}}$ denote the total variation distance.}.
\end{equation*}
For example, if $Z_0 = \varnothing$, then \begin{equation*}
    T
    =
    O\!\left(
        \kappa^4 n^2
        \log
        \frac{\kappa n}{\delta}
    \right).
\end{equation*}

\end{theorem}

\begin{proof}
By \cref{prop:canonical_as_sum_of_trees}, the canonical balanced AOAR weighted graph is a sum of fixed-order VGB tree graphs. By \cref{lem:fixed_order_conductance}, each fixed-order graph has conductance $\Omega(1/(\kappa^2 n))$. \cref{lem:sum_graph_conductance} therefore implies
\begin{equation*}
    \Phi(P_{\mathrm{BAL}}) = \Omega\sbra{\frac{1}{\kappa^2 n}}.
\end{equation*}
Since the chain is lazy and reversible, Cheeger's inequality \citep{levin2017markov} gives
\begin{equation*}
    \gamma(P_{\mathrm{BAL}})\ge \frac12 \Phi(P_{\mathrm{BAL}})^2=\Omega\sbra{\frac{1}{\kappa^4 n^2}}.
\end{equation*}

The leaf-mass lower bound follows by the same ratio-of-sums argument. Indeed, let
\begin{equation*}
    D_{\omega,\mathrm{leaf}}:=\sum_{y\in\calY}D_\omega(y), \qquad
    D_{\omega,\mathrm{tot}}:=\sum_{z\in\calZ_\omega}D_\omega(z).
\end{equation*}
Then $\mu_\omega(\calY)=D_{\omega,\mathrm{leaf}}/D_{\omega,\mathrm{tot}}$, or equivalently $D_{\omega,\mathrm{leaf}}=D_{\omega,\mathrm{tot}}\mu_\omega(\calY)$.
The global factor $(n-1)!$ in \cref{prop:canonical_as_sum_of_trees} cancels after normalization, so the leaf mass of the canonical balanced AOAR graph satisfies
\begin{align*}
    \mu_{\mathrm{BAL}}(\calY)
    &=\frac{\sum_{\omega\in\mathsf{Perm}_n}D_{\omega,\mathrm{leaf}}}
    {\sum_{\omega\in\mathsf{Perm}_n}D_{\omega,\mathrm{tot}}}
    =\frac{\sum_{\omega\in\mathsf{Perm}_n}D_{\omega,\mathrm{tot}}\,\mu_\omega(\calY)}
    {\sum_{\omega\in\mathsf{Perm}_n}D_{\omega,\mathrm{tot}}} \\
    &\ge \min_{\omega\in\mathsf{Perm}_n}\mu_\omega(\calY)
    \ge \frac{1}{4\kappa n}
    =\Omega\sbra{\frac{1}{\kappa n}}.
\end{align*}
The exact leaf law follows from \cref{thm:balanced_weighted_graph}, since terminal anchoring gives $\widehat V(x,y)=\tau(x,y)$ on leaves.

Finally, since $P_{\mathrm{BAL}}$ is lazy and reversible, the standard spectral estimate for reversible chains \citep{levin2017markov} gives the mixing time bound from $Z_0\in \calZ$ as follows:
\begin{equation*}
    \norm{P_{\mathrm{BAL}}^t(Z_0,\cdot)-\mu_{\mathrm{BAL}}}_{\mathrm{TV}}
    \le \frac12 e^{-t\gamma(P_{\mathrm{BAL}})}\mu_{\mathrm{BAL}}(Z_0)^{-1/2}.
\end{equation*}

In particular, the mixing time from $Z_0$ is,
\begin{equation}\label{eq:mixing time AOARVGB from z0}
    T_{\mathrm{mix}}^{(Z_0)}(\varepsilon)=O\sbra{\frac{1}{\gamma(P_{\mathrm{BAL}})} \log\frac{1}{\varepsilon\sqrt{\mu_{\mathrm{BAL}}(Z_0)}}}.
\end{equation}

Let $\nu_T:=P_{\mathrm{BAL}}^T(Z_0,\cdot)$. If
$d_{\mathrm{TV}}(\nu_T,\mu_{\mathrm{BAL}})\le \varepsilon$ and
$\varepsilon\le \mu_{\mathrm{BAL}}(\calY)/2$, then $\nu_T(\calY)\ge \mu_{\mathrm{BAL}}(\calY)/2 = \Omega\sbra{\frac{1}{\kappa n}}$; this follows from
$|\nu_T(\calY)-\mu_{\mathrm{BAL}}(\calY)|\le d_{\mathrm{TV}}(\nu_T,\mu_{\mathrm{BAL}})$.
Then, from the definition of total variation (TV) distance,
\begin{align*}
    d_{\mathrm{TV}}\sbra{\nu_T(\cdot\mid\calY),\mu_{\mathrm{BAL}}(\cdot\mid\calY)}
    &=\frac12\sum_{y\in\calY}
    \abs{\frac{\nu_T(y)}{\nu_T(\calY)}-\frac{\mu_{\mathrm{BAL}}(y)}{\mu_{\mathrm{BAL}}(\calY)}} \\
    &=\frac12\sum_{y\in\calY}
    \abs{\frac{\nu_T(y)-\mu_{\mathrm{BAL}}(y)}{\nu_T(\calY)}
    +\mu_{\mathrm{BAL}}(y)\sbra{\frac{1}{\nu_T(\calY)}-\frac{1}{\mu_{\mathrm{BAL}}(\calY)}}} \\
    &\le \frac{1}{2\nu_T(\calY)}\sum_{y\in\calY}\abs{\nu_T(y)-\mu_{\mathrm{BAL}}(y)}
    +\frac{\abs{\nu_T(\calY)-\mu_{\mathrm{BAL}}(\calY)}}{2\nu_T(\calY)} \\
    &\le \frac{2\varepsilon}{\mu_{\mathrm{BAL}}(\calY)}
    +\frac{\varepsilon}{\mu_{\mathrm{BAL}}(\calY)}
    \le \frac{3\varepsilon}{\mu_{\mathrm{BAL}}(\calY)}.
\end{align*}
Here the partial-sum bound used above is
\begin{equation*}
    \sum_{y\in\calY}\abs{\nu_T(y)-\mu_{\mathrm{BAL}}(y)}\le \sum_{z\in\calZ}\abs{\nu_T(z)-\mu_{\mathrm{BAL}}(z)}=2d_{\mathrm{TV}}(\nu_T,\mu_{\mathrm{BAL}})\le 2\varepsilon.
\end{equation*}
Thus, taking $\varepsilon:=\delta\,\mu_{\mathrm{BAL}}(\calY)/3$ ensures that the leaf-conditioned law is within $\delta$ of the stationary leaf-conditioned law.  
Plugging this choice into \eqref{eq:mixing time AOARVGB from z0}, and using
$\mu_{\mathrm{BAL}}(\calY)=\Omega(1/(\kappa n))$, gives
\begin{equation}\label{eq:mixing of balanced AOAR VGB general}
    T =O\sbra{\kappa^4 n^2
    \log\frac{1}{\varepsilon\sqrt{\mu_{\mathrm{BAL}}(Z_0)}}} =O\sbra{\kappa^4 n^2
    \log\frac{\kappa n}{\delta \mu_{\mathrm{BAL}}(Z_0)}}
\end{equation}
Since $\mu_{\mathrm{BAL}}(\cdot\mid\calY)=\pi^*(\cdot\mid x)$, the claim follows.

Next, consider the case $Z_0 = \varnothing.$ 
It remains to lower bound its stationary mass $\mu_{\mathrm{BAL}}(\varnothing)$. For a fixed order $\omega$, the root degree satisfies
\begin{align*}
    D_\omega(\varnothing)
    &=\sum_{a\in\calV}\widehat U(x,\varnothing^{\omega_1\leftarrow a})
    \ge \kappa^{-1}\sum_{a\in\calV}U^*(x,\varnothing^{\omega_1\leftarrow a}) =\kappa^{-1}\sum_{y\in\calY}\base(y\mid x)\tau(x,y)
    =\kappa^{-1}Z(x),
\end{align*}
where the middle equality uses that the first-reveal completion sets
$\{\calC(\varnothing^{\omega_1\leftarrow a})\}_{a\in\calV}$ partition $\calY$.
The total-degree bound in \cref{lem:fixed_order_conductance} gives $D_{\omega,\mathrm{tot}}\le 2\kappa(n+1)Z(x)$. Therefore $\mu_\omega(\varnothing)\ge 1/(2\kappa^2(n+1))$. Averaging over $\omega$ as above, with the global factor $(n-1)!$ canceling after normalization, yields
\begin{equation*}
    \mu_{\mathrm{BAL}}(\varnothing) = \Omega\sbra{\frac{1}{\kappa^2 n}}.
\end{equation*}
Plugging into \eqref{eq:mixing of balanced AOAR VGB general} finishes the proof.

\end{proof}

\subsection{Geometric Balanced AOAR-VGB}
\label{app:subsec:balanced_geometric_aoar_vgb}

\paragraph{Geometric balanced AOAR-VGB formulation.}
We take one further step beyond the depth-balanced AOAR-VGB local kernel in \cref{def:balanced_aoar_weights}. In that kernel, a backward move from a non-leaf state $z$ has weight proportional to $\widehat V(x,z)$ for every revealed coordinate $i\in R(z)$. Hence, once the chain decides to move backward, the index to re-mask is selected uniformly among revealed coordinates. This is a natural depth-balancing rule, but it ignores an extra degree of freedom that is absent in AR prefix trees: in AOAR, the sampler can choose which revealed coordinate should be erased.

This motivates a value-gated remasking rule. When considering a backward move $z\to z^{-i}$, we would like to favor coordinates whose remasked parent $z^{-i}$ remains promising under the verifier. We implement this through a clean geometric edge-weight construction: the child side carries the tilted completion mass $\widehat U(x,z^{j \leftarrow a})$, while the parent side contributes through $\widehat V(x,z^{-i})^\lambda$. This preserves the usual VGB forward ranking and makes the backward index choice depend on the verifier signal of the remasked parent.

\cref{fig:geometric_aoar_vgb_transition} illustrates the resulting local transition rule. The geometric gate leaves the forward candidate ranking unchanged up to the state-wise factor $\widehat V(x,z)^\lambda$, while the backward weights distinguish revealed coordinates through the verifier values of their remasked parents.

\begin{figure}[t]
    \centering
    \vspace{-0.5em}
    \includegraphics[width=0.95\linewidth]{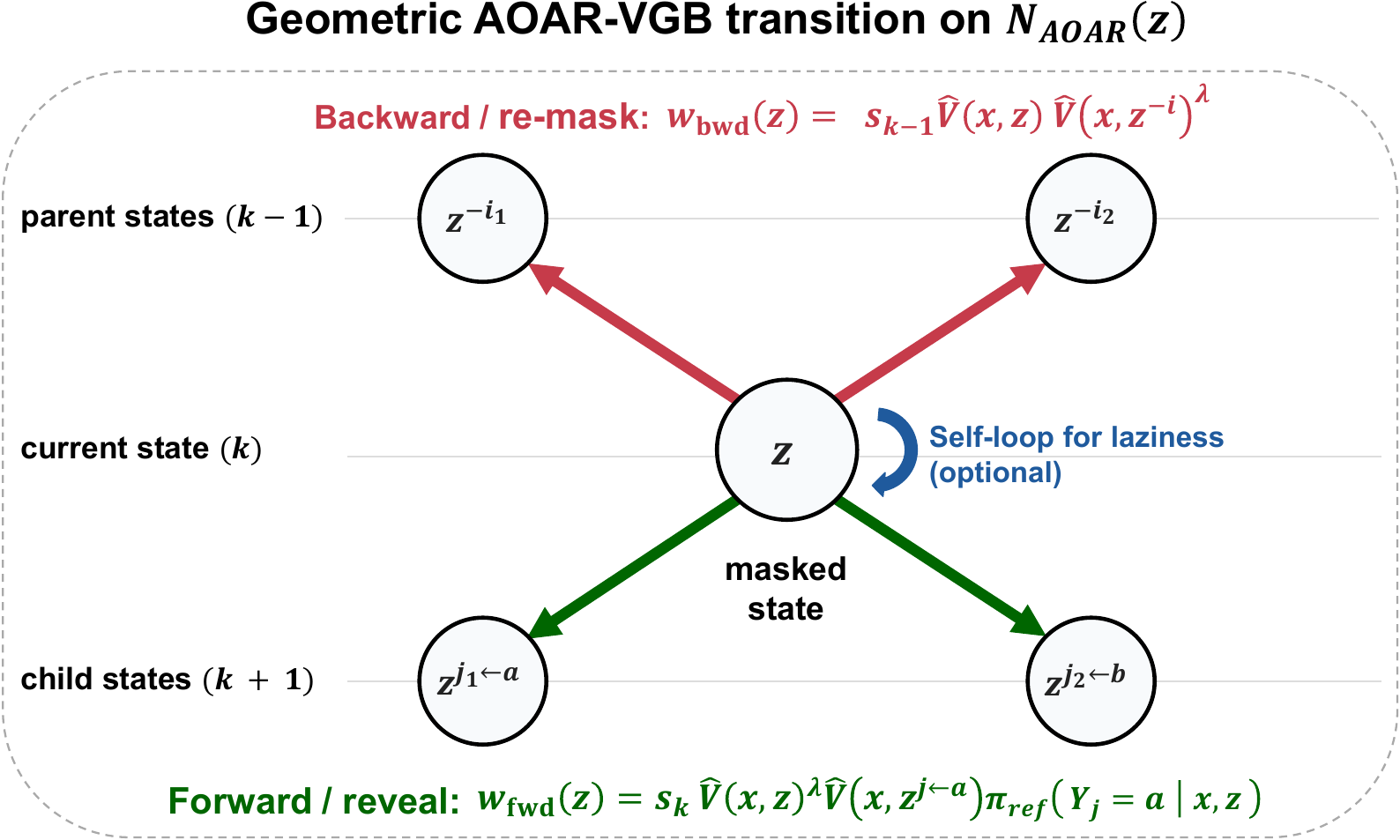}
    \caption{
        Illustration of one geometric balanced AOAR-VGB transition from a masked state $z$.
        Forward reveal moves use the reference conditional times the child verifier value, with the common gate $\widehat V(x,z)^\lambda$.
        Backward re-masking moves use the current verifier value and the parent gate $\widehat V(x,z^{-i})^\lambda$.
        Setting $\lambda=0$ removes the geometric gates and recovers balanced AOAR-VGB.
    }
    \label{fig:geometric_aoar_vgb_transition}
    \vspace{-0.5em}
\end{figure}

\begin{definition}[Geometric balanced AOAR edge weights]
\label{def:bg_aoar_edge_weights}
Fix $\lambda\ge0$ and positive depth coefficients
\begin{equation*}
    s_0,s_1,\dots,s_{n-1}>0.
\end{equation*}
For neighboring states $u \sim v$, let $p(u,v)$ denote the lower-depth endpoint and $c(u,v)$ denote the higher-depth endpoint. Thus $c(u,v)=p(u,v)^{j\leftarrow a}$ for a unique revealed coordinate and token.

If $c(u,v)\notin\calY$, define the non-leaf edge weight by
\begin{equation*}
    f_{\mathrm{BG}}(u,v)=f_{\mathrm{BG}}(v,u):=
    s_{k(p(u,v))}\widehat U(x,c(u,v))\widehat V(x,p(u,v))^\lambda.
\end{equation*}
If $c(u,v)=y\in\calY$, define the leaf-corrected edge weight by
\begin{equation*}
    f_{\mathrm{BG}}(u,v)=f_{\mathrm{BG}}(v,u):=
    s_{n-1}U^*(x,y)=s_{n-1}\base(y\mid x)\tau(x,y).
\end{equation*}
\end{definition}

\paragraph{Local implementation weights.}
Although \cref{def:bg_aoar_edge_weights} is written in terms of $\widehat U$ and $U^*$, the local weights below do not require evaluating these tilted masses directly. Let $z$ have depth $k=k(z)$.

For a non-leaf child $c=z^{j\leftarrow a}\notin\calY$, we have
\begin{equation*}
    \widehat U(x,c)=M_{\mathrm{ref}}(x,z)\base(Y_j=a\mid x,z)\widehat V(x,c).
\end{equation*}
Dividing all incident non-leaf edge weights at current state $z$ by the common factor $M_{\mathrm{ref}}(x,z)$ gives
\begin{equation*}
\label{eq:bg_aoar_clean_forward_weight}
    w_{\mathrm{BG,fwd}}(z\to c)
    =
    s_k\widehat V(x,z)^\lambda\base(Y_j=a\mid x,z)\widehat V(x,c).
\end{equation*}
The factor $\widehat V(x,z)^\lambda$ is common across all forward candidates from the same current state $z$. Hence the forward candidate ranking is the same as in ordinary VGB:
\begin{equation*}
    w_{\mathrm{BG,fwd}}(z\to c) \propto \base(Y_j=a\mid x,z)\widehat V(x,z^{j\leftarrow a}).
\end{equation*}

For a backward move from a non-leaf state $z$ to $u=z^{-i}$, the relevant edge has child endpoint $z$ and parent endpoint $u$. Thus
\begin{equation*}
    f_{\mathrm{BG}}(u,z)=s_{k-1}\widehat U(x,z)\widehat V(x,u)^\lambda=
    s_{k-1}M_{\mathrm{ref}}(x,z)\widehat V(x,z)\widehat V(x,u)^\lambda.
\end{equation*}
Dividing by the current-state common factor $M_{\mathrm{ref}}(x,z)$ gives
\begin{equation*}
\label{eq:bg_aoar_clean_backward_weight}
    w_{\mathrm{BG,bwd}}(z\to z^{-i})
    =
    s_{k-1}\widehat V(x,z)\widehat V(x,z^{-i})^\lambda.
\end{equation*}
Thus backward moves prefer revealed coordinates whose remasked parent has high verifier value.

At the leaf boundary, the leaf-corrected edge gives the local forward weight directly. If $y=z^{j\leftarrow a}\in\calY$, then
\begin{equation*}
    f_{\mathrm{BG}}(z,y)=s_{n-1}U^*(x,y)
    =s_{n-1}M_{\mathrm{ref}}(x,z)\base(Y_j=a\mid x,z)\tau(x,y),
\end{equation*}
and dividing by $M_{\mathrm{ref}}(x,z)$ gives
\begin{equation*}
\label{eq:bg_aoar_clean_leaf_forward}
    w_{\mathrm{BG,fwd}}(z\to y)=s_{n-1}\base(Y_j=a\mid x,z)\tau(x,y).
\end{equation*}
Similarly, from a leaf $y$ to a parent $y^{-i}$,
\begin{equation*}
\label{eq:bg_aoar_clean_leaf_backward}
    w_{\mathrm{BG,bwd}}(y\to y^{-i})=s_{n-1}\tau(x,y).
\end{equation*}

\begin{remark}[$\lambda=0$ recovers balanced AOAR-VGB]
When $\lambda=0$, the geometric gate disappears and the local weights reduce to the balanced AOAR-VGB weights with depth coefficients $s_k$.
\end{remark}

\begin{theorem}[Stationarity and exact leaf law of geometric balanced AOAR]
\label{thm:bg_aoar_stationarity}
Let $P_{\mathrm{BG}}$ be the random walk on the reachable AOAR graph with edge weights $f_{\mathrm{BG}}$ from \cref{def:bg_aoar_edge_weights}. That is,
\begin{equation*}
    P_{\mathrm{BG}}(z,u)=\frac{f_{\mathrm{BG}}(z,u)}{D_{\mathrm{BG}}(z)},\qquad
    D_{\mathrm{BG}}(z):=\sum_{v\sim z}f_{\mathrm{BG}}(z,v).
\end{equation*}
Then $P_{\mathrm{BG}}$ is reversible with stationary distribution
\begin{equation*}
    \mu_{\mathrm{BG}}(z)=\frac{D_{\mathrm{BG}}(z)}{\sum_{z'}D_{\mathrm{BG}}(z')}.
\end{equation*}
Moreover, under terminal anchoring, its leaf-conditioned stationary law is exact:
\begin{equation*}
    \mu_{\mathrm{BG}}(y\mid y\in\calY)=\pi^*(y\mid x).
\end{equation*}
\end{theorem}

\begin{proof}
Since $f_{\mathrm{BG}}$ is a nonnegative symmetric edge weight, the random walk is a standard weighted-graph random walk. For neighboring distinct states $z\sim u$,
\begin{equation*}
    \mu_{\mathrm{BG}}(z)P_{\mathrm{BG}}(z,u)
    =
    \frac{f_{\mathrm{BG}}(z,u)}{\sum_{z'}D_{\mathrm{BG}}(z')}
    =
    \mu_{\mathrm{BG}}(u)P_{\mathrm{BG}}(u,z),
\end{equation*}
so detailed balance holds.

Now let $y\in\calY$ be a leaf. Its AOAR parents are $y^{-i}$ for $i\in\lbra{n}$. By the leaf correction,
\begin{equation*}
    f_{\mathrm{BG}}(y^{-i},y)=s_{n-1}U^*(x,y)
\end{equation*}
for every $i$. Hence
\begin{equation*}
    D_{\mathrm{BG}}(y)=\sum_{i=1}^n f_{\mathrm{BG}}(y^{-i},y)=n s_{n-1}U^*(x,y).
\end{equation*}
The factor $ns_{n-1}$ is constant over leaves, and therefore
\begin{equation*}
    \mu_{\mathrm{BG}}(y\mid y\in\calY)
    =
    \frac{U^*(x,y)}{\sum_{y'\in\calY}U^*(x,y')}
    =
    \pi^*(y\mid x).
\end{equation*}
\end{proof}

\paragraph{Mixing analysis of geometric balanced VGB.}

We next transfer the canonical balanced mixing guarantee to the geometric
kernel. The only care point is that zero target-mass states should not create
spurious components. We therefore work on the positive target-support component,
which is stable under all positive-weight moves.

\begin{lemma}[Positive target support is closed under AOAR moves]
\label{lem:bg_positive_support_closure}
Assume $Z(x)>0$ and \cref{ass:aoar_kappa_accuracy}. Define the positive
target support
\begin{equation*}
    \calZ_+(x)
    :=
    \bigl\{
        z\in\calZ_{\mathrm{reach}}
        :
        U^\star(x,z)>0
    \bigr\}.
\end{equation*}
Then $\varnothing\in\calZ_+(x)$. Moreover, $\calZ_+(x)$ is closed under
all positive-weight balanced and geometric AOAR moves. More precisely:

\begin{enumerate}
    \item If $z\in\calZ_+(x)$ and $i\in R(z)$, then
    $z^{-i}\in\calZ_+(x)$.

    \item If $z\in\calZ_+(x)$, $j\notin R(z)$, and
    $c=z^{j\leftarrow a}\notin\calZ_+(x)$, then the AOAR edge
    $\mbra{z,c}$ has zero balanced edge weight and zero geometric edge weight.
\end{enumerate}
\end{lemma}

\begin{proof}
Since $Z(x)>0$,
\begin{equation*}
    U^\star(x,\varnothing)=Z(x)>0,
\end{equation*}
so $\varnothing\in\calZ_+(x)$.

For the first claim, let $p=z^{-i}$. Since $p$ is obtained by re-masking one
revealed coordinate of $z$, we have
\begin{equation*}
    \calC(z)\subseteq \calC(p).
\end{equation*}
Thus
\begin{equation*}
    U^\star(x,p)
    =
    \sum_{y\in\calC(p)}\base(y\mid x)\tau(x,y)
    \ge
    \sum_{y\in\calC(z)}\base(y\mid x)\tau(x,y)
    =
    U^\star(x,z)
    >
    0.
\end{equation*}
Hence $p\in\calZ_+(x)$.

For the second claim, let $c=z^{j\leftarrow a}\notin\calZ_+(x)$. If $c$ is
not $\base$-reachable, then it is not an admissible child edge. Otherwise,
$M_{\rm ref}(x,c)>0$ and $U^\star(x,c)=0$, hence
$V^\star(x,c)=0$. By \cref{ass:aoar_kappa_accuracy},
\begin{equation*}
    \widehat V(x,c)=0,
    \qquad
    \widehat U(x,c):=M_{\rm ref}(x,c)\widehat V(x,c)=0.
\end{equation*}
Therefore the balanced edge weight, which is proportional to
$\widehat U(x,c)$ for non-leaf children and to $U^\star(x,c)$ for leaf
children, is zero. The same is true for the geometric edge weight: for non-leaf
children it is proportional to
\begin{equation*}
    \widehat U(x,c)\widehat V(x,z)^\lambda,
\end{equation*}
and for leaf children it is proportional to $U^\star(x,c)$. Hence no
positive-weight edge leaves $\calZ_+(x)$.
\end{proof}

On this component, the geometric graph differs from the canonical balanced
graph only through the parent value gate $\widehat V(x,p)^\lambda$ on
non-leaf edges. The next assumption prevents this gate from creating an
additional bottleneck.

\begin{assumption}[Bounded geometric verifiers]
\label{ass:bg_gate_bounds}
Consider the canonical geometric balanced AOAR-VGB graph with
\begin{equation*}
    s_k=\frac{1}{\binom{n-1}{k}},
    \qquad 0\le k\le n-1.
\end{equation*}
Fix $\lambda>0$. Assume there exist constants
\begin{equation*}
    0<m\le 1\le M<\infty
\end{equation*}
such that, for every non-leaf state $z\in\calZ_+(x)$,
\begin{equation*}
    m
    \le
    \widehat V(x,z)
    \le
    M .
\end{equation*}
\end{assumption}

\begin{theorem}[Mixing stability for geometric balanced AOAR-VGB]
\phantomsection
\label{thm:canonical_geometric_balanced_mixing}
Assume \cref{ass:aoar_kappa_accuracy} and \cref{ass:bg_gate_bounds}. Let $P_{\mathrm{BG}}$ be the lazy canonical geometric balanced AOAR-VGB kernel on $\calZ_+(x)$ with $Z(x)>0$, and let $\mu_{\mathrm{BG}}$ be its stationary distribution. Then
\begin{equation*}
    \mu_{\mathrm{BG}}(\cdot\mid\calY)=\pi^*(\cdot\mid x).
\end{equation*}
Moreover, its conductance satisfies
\begin{equation*}
    \Phi(P_{\mathrm{BG}})
    =
    \Omega\!\left(
        \left(\frac{m}{M}\right)^\lambda
        \frac{1}{\kappa^2 n}
    \right),
\end{equation*}
and hence, by Cheeger's inequality for lazy reversible chains, its spectral gap satisfies
\begin{equation*}
    \gamma(P_{\mathrm{BG}})
    =
    \Omega\!\left(
        \left(\frac{m}{M}\right)^{2\lambda}
        \frac{1}{\kappa^4 n^2}
    \right).
\end{equation*}
The stationary leaf and root masses satisfy
\begin{equation*}
    \mu_{\mathrm{BG}}(\calY)
    =
    \Omega\!\left(
        \frac{1}{M^\lambda\kappa n}
    \right),
    \qquad
    \mu_{\mathrm{BG}}(\varnothing)
    =
    \Omega\!\left(
        \left(\frac{m}{M}\right)^\lambda
        \frac{1}{\kappa^2 n}
    \right).
\end{equation*}
Consequently, if $Z_T$ is the state after $T$ lazy geometric AOAR-VGB transitions initialized at $Z_0\in\calZ_+(x)$ with $\mu_{\mathrm{BG}}(Z_0)>0$, then for any $\delta>0$,
\begin{equation*}
    T
    =
    O\!\left(
        \left(\frac{M}{m}\right)^{2\lambda}
        \kappa^4 n^2
        \log
        \frac{M^\lambda\kappa n}{\delta \mu_{\mathrm{BG}}(Z_0)}
    \right)
\end{equation*}
suffices to guarantee
\begin{equation*}
    d_{\mathrm{TV}}
    \left(
        \mathcal L(Z_T\mid Z_T\in\calY),
        \pi^*(\cdot\mid x)
    \right)
    \le
    \delta \quad \text{and} \quad \mathbb{P}[Z_T \in \calY] =\Omega\!\left(
        \frac{1}{M^\lambda\kappa n}
    \right).
\end{equation*}
For example, if $Z_0 = \varnothing$, then \begin{equation*}
    T
    =
    O\!\left(
        \left(\frac{M}{m}\right)^{2\lambda}
        \kappa^4 n^2
        \log
        \frac{(M/m)^\lambda\kappa n}{\delta}
    \right).
\end{equation*}
In particular, if $m,M$ are constants independent of $n$ and $\lambda$ is fixed, the geometric chain's mixing time is of the same order as canonical balanced AOAR-VGB.
\end{theorem}

\begin{proof}
Let $f_{\mathrm{BAL}}$ be the canonical balanced AOAR edge weight and let $f_{\mathrm{BG}}$ be the canonical geometric balanced AOAR edge weight. For an edge $e=\mbra{p,c}$, let $p$ be the lower-depth endpoint and $c$ the higher-depth endpoint. If $c\notin\calY$, then by \cref{def:bg_aoar_edge_weights},
\begin{equation*}
    f_{\mathrm{BG}}(e)
    =
    \widehat V(x,p)^\lambda f_{\mathrm{BAL}}(e).
\end{equation*}
If $c\in\calY$, the leaf correction gives
\begin{equation*}
    f_{\mathrm{BG}}(e)=f_{\mathrm{BAL}}(e).
\end{equation*}
Therefore, using \cref{ass:bg_gate_bounds} and $m\le 1\le M$, every positive-weight edge in the component satisfies
\begin{equation*}
    m^\lambda f_{\mathrm{BAL}}(e)
    \le
    f_{\mathrm{BG}}(e)
    \le
    M^\lambda f_{\mathrm{BAL}}(e).
\end{equation*}
For any set $A$, write
\begin{equation*}
    Q_{\mathrm{BG}}(A,A^c):=\sum_{\substack{e=\mbra{u,v}\\u\in A,\ v\notin A}}f_{\mathrm{BG}}(e),
    \qquad
    \mathrm{vol}_{\mathrm{BG}}(A):=\sum_{u\in A}\sum_{v\sim u}f_{\mathrm{BG}}(u,v),
\end{equation*}
and define $Q_{\mathrm{BAL}}$ and $\mathrm{vol}_{\mathrm{BAL}}$ analogously.
The edge comparison gives
\begin{equation*}
    Q_{\mathrm{BG}}(A,A^c)\ge m^\lambda Q_{\mathrm{BAL}}(A,A^c),
    \qquad
    \mathrm{vol}_{\mathrm{BG}}(A)\le M^\lambda\mathrm{vol}_{\mathrm{BAL}}(A).
\end{equation*}
Thus, for any admissible $A$,
\begin{equation*}
    \Phi_{\mathrm{BG}}(A)
    =
    \frac{Q_{\mathrm{BG}}(A,A^c)}{\mathrm{vol}_{\mathrm{BG}}(A)}
    \ge
    \left(\frac{m}{M}\right)^\lambda
    \frac{Q_{\mathrm{BAL}}(A,A^c)}{\mathrm{vol}_{\mathrm{BAL}}(A)}
    =
    \left(\frac{m}{M}\right)^\lambda
    \Phi_{\mathrm{BAL}}(A).
\end{equation*}
Taking the infimum over admissible $A$ and applying \cref{thm:canonical_balanced_mixing} gives
\begin{equation*}
    \Phi(P_{\mathrm{BG}})
    =
    \Omega\!\left(
        \left(\frac{m}{M}\right)^\lambda
        \frac{1}{\kappa^2 n}
    \right).
\end{equation*}
Since the lazy geometric chain is reversible, Cheeger's inequality gives
\begin{equation*}
    \gamma(P_{\mathrm{BG}})
    =
    \Omega\!\left(
        \left(\frac{m}{M}\right)^{2\lambda}
        \frac{1}{\kappa^4 n^2}
    \right).
\end{equation*}

The exact leaf-conditioned law follows from \cref{thm:bg_aoar_stationarity}, since laziness does not change the stationary distribution. For the leaf mass, leaf-corrected edges have the same weight as in the canonical balanced graph. Writing $D_{\mathrm{BG}}$ and $D_{\mathrm{BAL}}$ for weighted degrees,
\begin{equation*}
    \mu_{\mathrm{BG}}(\calY)
    =
    \frac{\sum_{y\in\calY}D_{\mathrm{BAL}}(y)}
    {\sum_z D_{\mathrm{BG}}(z)}
    \ge
    \frac{\sum_{y\in\calY}D_{\mathrm{BAL}}(y)}
    {M^\lambda\sum_zD_{\mathrm{BAL}}(z)}
    =
    \frac{1}{M^\lambda}\mu_{\mathrm{BAL}}(\calY)
    =
    \Omega\!\left(
        \frac{1}{M^\lambda\kappa n}
    \right).
\end{equation*}
Similarly, the root degree is at least $m^\lambda$ times the balanced root degree, while the total degree is at most $M^\lambda$ times the balanced total degree, so
\begin{equation*}
    \mu_{\mathrm{BG}}(\varnothing)
    \ge
    \left(\frac{m}{M}\right)^\lambda\mu_{\mathrm{BAL}}(\varnothing)
    =
    \Omega\!\left(
        \left(\frac{m}{M}\right)^\lambda
        \frac{1}{\kappa^2 n}
    \right).
\end{equation*}

Finally, the standard start-state spectral bound for lazy reversible chains \citep{levin2017markov} gives
\begin{equation*}
    \left\|
        P_{\mathrm{BG}}^T(Z_0,\cdot)-\mu_{\mathrm{BG}}
    \right\|_{\mathrm{TV}}
    \le
    \frac12
    e^{-T\gamma(P_{\mathrm{BG}})}
    \mu_{\mathrm{BG}}(Z_0)^{-1/2}.
\end{equation*}
As in the proof of \cref{thm:canonical_balanced_mixing}, taking the full-chain TV error to be at most
\begin{equation*}
    \varepsilon:=\frac{\delta\,\mu_{\mathrm{BG}}(\calY)}{3}
\end{equation*}
implies
\begin{equation*}
    d_{\mathrm{TV}}
    \left(
        \mathcal L(Z_T\mid Z_T\in\calY),
        \mu_{\mathrm{BG}}(\cdot\mid\calY)
    \right)
    \le
    \delta \quad \text{and} \quad \mathbb{P}[Z_T \in \calY] \geq \mu_{\mathrm{BG}}(\calY)/2 = \Omega\!\left(
        \frac{1}{M^\lambda\kappa n}\right).
\end{equation*}
Substituting the bounds on \(\gamma(P_{\mathrm{BG}})\) and \(\mu_{\mathrm{BG}}(\calY)\), while retaining the start-state factor \(\mu_{\mathrm{BG}}(Z_0)\), yields the general mixing time. The lower bound on \(\mu_{\mathrm{BG}}(\varnothing)\) gives the root-start example. Since $\mu_{\mathrm{BG}}(\cdot\mid\calY)=\pi^*(\cdot\mid x)$, the claim follows.
\end{proof}

\paragraph{Forward/backward mass analysis.}

We now return from global mixing to the local behavior of the geometric
transition. The next proposition measures how the
forward-versus-backward move probabilities change under the exact verifier setup.

\begin{proposition}[Directional masses under the exact verifier]
\label{prop:bg_aoar_directional_masses}
Assume $\widehat V=V^*$ on all non-leaf states. Let $z$ be a non-leaf interior state of depth $1\le k\le n-2$ with $V^*(x,z)>0$, so that its one-step forward children are also non-leaf. Define the parent-value ratio
\begin{equation*}
    r_i^*(z)
    :=
    \frac{V^*(x,z^{-i})}{V^*(x,z)},
    \qquad i\in R(z).
\end{equation*}
Then the exact forward and backward masses are
\begin{equation*}
    F_\lambda^*(z)
    =
    s_k(n-k)V^*(x,z)^{1+\lambda},
\end{equation*}
and
\begin{equation*}
    B_\lambda^*(z)
    =
    s_{k-1}V^*(x,z)^{1+\lambda}
    \sum_{i\in R(z)} r_i^*(z)^\lambda.
\end{equation*}
Consequently,
\begin{equation*}
    \bbP(\mathrm{forward}\mid x,z,\mathrm{move})
    =
    \frac{s_k(n-k)}
    {s_k(n-k)+s_{k-1}\sum_{i\in R(z)}r_i^*(z)^\lambda},
\end{equation*}
and
\begin{equation*}
    \bbP(\mathrm{backward}\mid x,z,\mathrm{move})
    =
    \frac{s_{k-1}\sum_{i\in R(z)}r_i^*(z)^\lambda}
    {s_k(n-k)+s_{k-1}\sum_{i\in R(z)}r_i^*(z)^\lambda}.
\end{equation*}
\end{proposition}

\begin{proof}
For the forward mass, using the AOAR Bellman identity,
\begin{align*}
    F_\lambda^*(z)
    &=
    \sum_{j\notin R(z)}
    \sum_{a\in\calV}
    s_kV^*(x,z)^\lambda
    \base(Y_j=a\mid x,z)
    V^*(x,z^{j\leftarrow a})\\
    &=
    s_kV^*(x,z)^\lambda
    \sum_{j\notin R(z)}V^*(x,z)
    =
    s_k(n-k)V^*(x,z)^{1+\lambda}.
\end{align*}
For backward moves,
\begin{align*}
    B_\lambda^*(z)
    &=
    \sum_{i\in R(z)}
    s_{k-1}V^*(x,z)V^*(x,z^{-i})^\lambda\\
    &=
    s_{k-1}V^*(x,z)^{1+\lambda}
    \sum_{i\in R(z)}
    \left[
        \frac{V^*(x,z^{-i})}{V^*(x,z)}
    \right]^\lambda.
\end{align*}
Normalizing the forward and backward masses gives the probabilities in the proposition.
\end{proof}

\begin{remark}[Reference-target mismatch interpretation]
\label{rem:bg_reference_target_mismatch}
By the exact target conditional,
\begin{equation*}
    \pi^*(Y_i=z_i\mid x,z^{-i})
    =
    \base(Y_i=z_i\mid x,z^{-i})
    \frac{V^*(x,z)}{V^*(x,z^{-i})},
    \qquad
    r_i^*(z)
    =
    \frac{\base(Y_i=z_i\mid x,z^{-i})}
    {\pi^*(Y_i=z_i\mid x,z^{-i})}.
\end{equation*}
Thus $r_i^*(z)>1$ means $\base$ favors $z_i$ more than $\pi^*$ and increases
the remasking weight of coordinate $i$, while $r_i^*(z)<1$ means $\pi^*$ gives
stronger support and the remasking weight is reduced.

In other words, the chain preferentially erases tokens that look base-plausible
but target-weak, and keeps tokens that are target-supported.
\end{remark}

\begin{remark}[Penultimate-depth directional masses]
\label{rem:bg_leaf_boundary_directional_masses}
When $z$ is a penultimate-depth state with $k(z)=n-1$, any forward move reaches a leaf. The leaf correction resets the forward edge to $s_{n-1}U^*(x,y)$ rather than the non-leaf product edge. Hence, under the exact verifier,
\begin{equation*}
    F_\lambda^*(z)=s_{n-1}V^*(x,z),
\end{equation*}
whereas
\begin{equation*}
    B_\lambda^*(z)
    =
    s_{n-2}V^*(x,z)\sum_{i\in R(z)}V^*(x,z^{-i})^\lambda.
\end{equation*}
This boundary rule is intentional: it keeps the leaf degree proportional to $U^*(x,y)$, and therefore preserves the exact leaf-conditioned stationary law.
\end{remark}

\begin{remark}[Canonical depth anchor versus exact $1/2$ balance]
\label{rem:bg_canonical_anchor}
The linear balanced choice
\begin{equation*}
    s_k(n-k)=ks_{k-1}
\end{equation*}
does not give exact $1/2$--$1/2$ balance for $\lambda>0$ in the clean geometric chain. Under this choice, away from the leaf boundary,
\begin{equation*}
    \bbP(\mathrm{backward}\mid x,z,\mathrm{move})
    =
    \frac{\sum_{i\in R(z)}r_i^*(z)^\lambda}
    {k+\sum_{i\in R(z)}r_i^*(z)^\lambda},
    \qquad k=\abs{R(z)}.
\end{equation*}
Thus $s_k(n-k)=ks_{k-1}$ is best interpreted as the canonical \emph{depth-neutral anchor}: it recovers exact balance when $\lambda=0$, while the $\lambda>0$ correction changes the backward probability according to the verifier value of the remasked parent states.

More quantitatively, for small $\lambda$,
\begin{equation*}
    r_i^*(z)^\lambda
    =
    \exp\!\left(\lambda\log r_i^*(z)\right)
    =
    1+\lambda\log r_i^*(z)
    +O\!\left(\lambda^2\log^2 r_i^*(z)\right).
\end{equation*}
Define the first empirical log-moment of the parent-value ratios
\begin{equation*}
    \bar m_1(z)
    :=
    \frac{1}{k}\sum_{i\in R(z)}
    \log\frac{V^*(x,z^{-i})}{V^*(x,z)}.
\end{equation*}
Then
\begin{equation*}
    \sum_{i\in R(z)}r_i^*(z)^\lambda
    =
    k+\lambda k\bar m_1(z)+O\!\left(\lambda^2 k\bar m_2(z)\right),
\end{equation*}
where
\begin{equation*}
    \bar m_2(z)
    :=
    \frac{1}{k}\sum_{i\in R(z)}
    \log^2\frac{V^*(x,z^{-i})}{V^*(x,z)}.
\end{equation*}
Substituting this into the backward probability gives
\begin{equation*}
    \bbP(\mathrm{backward}\mid x,z,\mathrm{move})
    =
    \frac12
    +
    \frac{\lambda}{4}\bar m_1(z)
    +
    O\!\left(
        \lambda^2\bar m_2(z)
        +
        \lambda^2\bar m_1(z)^2
    \right).
\end{equation*}
Equivalently,
\begin{equation*}
    \bbP(\mathrm{forward}\mid x,z,\mathrm{move})
    =
    \frac12
    -
    \frac{\lambda}{4}\bar m_1(z)
    +
    O\!\left(
        \lambda^2\bar m_2(z)
        +
        \lambda^2\bar m_1(z)^2
    \right).
\end{equation*}
Therefore, the leading-order deviation from $1/2$ is not a combinatorial depth bias. It is the value-gated remasking signal induced by the geometric parent factor.
\end{remark}

\section{Momentum-Balanced AOAR-VGB}
\label{app:sec:momentum_balanced_aoar_vgb}

Balanced AOAR-VGB fixes the depth-allocation pathology of primitive AOAR-VGB by balancing forward and backward directions at each interior depth in the ideal exact-value case. However, the resulting walk can still be locally diffusive: after revealing a coordinate, the chain may immediately re-mask a coordinate, and after re-masking it may immediately reveal again. Such re-masking/revealing oscillations waste Markov steps and can be computationally inefficient, especially for long sequences where each transition may require evaluating many candidate child states.

This issue is closely related to the diffusive behavior of reversible random walks studied in the lifting literature. A lifting replaces each state by several copies and defines a larger Markov chain whose projection preserves the desired stationary law, while using non-reversible momentum to reduce unnecessary direction changes \citep{hayes2010liftings}. In the present AOAR setting, the two natural directions are downward reveal moves and upward re-masking moves. We therefore introduce a two-copy momentum lift of balanced AOAR-VGB. The construction is inspired by the momentum version of VGB in \citet{rohatgi2025taming} and by the flow-cancellation principle of \citet{hayes2010liftings}, but is adapted here to the any-order masked-state graph.

\subsection{Setup and Notation}
\label{app:subsec:momentum_balanced_setup}

Throughout this section, we use the balanced AOAR-VGB notation from \cref{app:subsec:balanced_aoar_vgb}. In particular, $K_{\mathrm{BAL}}$ denotes the self-loop-free balanced AOAR-VGB move kernel induced by the local weights in \cref{def:balanced_aoar_weights} with $\gamma=0$, and $\mu_{\mathrm{BAL}}$ denotes its stationary distribution from \cref{thm:balanced_weighted_graph}.

\paragraph{Directional neighborhoods and masses.}
Recall the AOAR child and parent neighborhoods
\begin{align*}
    C_{\mathrm{AOAR}}(z) &:= \mbra{z^{j\leftarrow a}: j\notin R(z),\ a\in\calV,\ M_{\mathrm{ref}}(x,z^{j\leftarrow a})>0},\\
    P_{\mathrm{AOAR}}(z) &:= \mbra{z^{-i}: i\in R(z)}.
\end{align*}
A move to $C_{\mathrm{AOAR}}(z)$ is called a \emph{downward} $(\downarrow)$ or \emph{forward} move, because it reveals one additional coordinate. A move to $P_{\mathrm{AOAR}}(z)$ is called an \emph{upward} $(\uparrow)$ or \emph{backward} move, because it re-masks one coordinate.

For the self-loop-free balanced kernel, define the directional masses
\begin{equation*}
    p_{\downarrow}(z) := \sum_{c\in C_{\mathrm{AOAR}}(z)}K_{\mathrm{BAL}}(z,c), \qquad
    p_{\uparrow}(z) := \sum_{p\in P_{\mathrm{AOAR}}(z)}K_{\mathrm{BAL}}(z,p).
\end{equation*}
Thus $p_{\downarrow}(z)+p_{\uparrow}(z)=1$. At the boundaries, these definitions give
\begin{equation*}
    p_{\downarrow}(\varnothing)=1,\qquad p_{\uparrow}(\varnothing)=0,\qquad
    p_{\downarrow}(y)=0,\qquad p_{\uparrow}(y)=1\quad (y\in\calY).
\end{equation*}
In the exact canonical balanced case, where $\beta_k=k$, $\alpha_k=n-k$ for $1\le k\le n-1$, and $\widehat V=V^*$, \cref{prop:balanced_directional_masses} gives
\begin{equation*}
    p_{\downarrow}(z)=p_{\uparrow}(z)=\frac12
\end{equation*}
for every interior reachable state. In practical plug-in samplers, however, $\widehat V$ may break this exact equality.

\paragraph{Lifted tuple state space.}
Momentum is represented by augmenting the masked state with a direction variable. We use tuple states
\begin{equation*}
    (z,\downarrow), \qquad (z,\uparrow).
\end{equation*}
As introduced above, the label $\downarrow$ means that the chain is currently in a downward, forward-revealing momentum mode, while $\uparrow$ means that it is currently in an upward, backtracking momentum mode.

Let $\calZ_{\mathrm{BAL},+}$ be a reachable component of $\calZ_{\mathrm{reach}}$ on which $W_{\mathrm{BAL}}(z)>0$ for every state $z$. We keep both direction labels for every state in this component:
\begin{equation*}
    \widetilde{\calZ}_{\mathrm{BAL}} := \calZ_{\mathrm{BAL},+}\times\mbra{\downarrow,\uparrow}.
\end{equation*}
Thus, states in this component, including its root and leaf boundary states, have two momentum copies. Some boundary copies are only used as switch states. For example, from $(\varnothing,\uparrow)$ the chain switches to $(\varnothing,\downarrow)$ before moving downward, while from $(y,\downarrow)$ a leaf switches to $(y,\uparrow)$ before moving upward.

We write
\begin{equation*}
    \textsf{proj}(z,d):=z,\qquad d\in\mbra{\downarrow,\uparrow},
\end{equation*}
for the projection from a lifted state to its current masked state. For a probability distribution $\widetilde\nu$ on $\widetilde{\calZ}_{\mathrm{BAL}}$, its current-state marginal is
\begin{equation*}
    \textsf{proj}_{\#}\widetilde\nu(z) = \widetilde\nu(z,\downarrow)+\widetilde\nu(z,\uparrow).
\end{equation*}

\paragraph{Stationary split.}
The goal of the momentum lift is to change path geometry without changing the stationary law after projection. Thus the lifted chain should have a stationary distribution on the two-copy space whose projection is exactly the balanced AOAR-VGB stationary law $\mu_{\mathrm{BAL}}$. We take the simplest such target: an equal split of the base stationary mass across the two momentum modes,
\begin{equation*}
    \widetilde\mu_{\mathrm{BAL}}(z,\downarrow):=\frac12\mu_{\mathrm{BAL}}(z), \qquad
    \widetilde\mu_{\mathrm{BAL}}(z,\uparrow):=\frac12\mu_{\mathrm{BAL}}(z).
\end{equation*}
This choice does not bias the chain toward either momentum direction at stationarity, and projecting out the momentum label immediately recovers the desired base stationary law:
\begin{equation*}
    \textsf{proj}_{\#}\widetilde\mu_{\mathrm{BAL}}(z)=\mu_{\mathrm{BAL}}(z).
\end{equation*}

\subsection{Momentum-Balanced AOAR-VGB}
\label{app:subsec:momentum_balanced_aoar_vgb}

We now define the flow-cancelled switch lift. When the chain is in downward mode, it keeps the child moves of the balanced AOAR kernel and replaces the total parent-move mass by a switch to upward mode. When the chain is in upward mode, it keeps the parent moves and replaces the total child-move mass by a switch to downward mode.

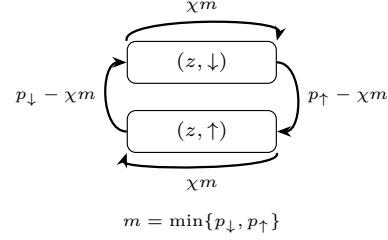
\begin{wrapfigure}[12]{r}{0.40\textwidth}
\vspace{-0.6em}
\centering
\resizebox{0.96\linewidth}{!}{
\begin{tikzpicture}[
    scale=0.78,
    every node/.style={font=\scriptsize},
    copy/.style={draw=black, fill=white, rounded corners=3pt, minimum width=1.85cm, minimum height=0.52cm},
    arr/.style={-{Stealth[length=1.8mm]}, thick, black}
]
    \node[copy] (down) at (0,0.55) {$(z,\downarrow)$};
    \node[copy] (up) at (0,-0.55) {$(z,\uparrow)$};
    \draw[arr] (down.east) to[out=0,in=0] node[right, font=\tiny] {$p_{\uparrow}-\chi m$} (up.east);
    \draw[arr] (up.west) to[out=180,in=180] node[left, font=\tiny] {$p_{\downarrow}-\chi m$} (down.west);
    \draw[arr] (down.north west) .. controls +(-0.1,0.38) and +(0.1,0.38) .. node[above, font=\tiny] {$\chi m$} (down.north east);
    \draw[arr] (up.south east) .. controls +(0.1,-0.28) and +(-0.1,-0.28) .. node[below, yshift=-0.4mm, font=\tiny] {$\chi m$} (up.south west);
    \node[align=center, font=\tiny] at (0,-2.00) {$m=\min\{p_{\downarrow},p_{\uparrow}\}$};
\end{tikzpicture}
}
\caption{Flow cancellation at state $z$.}
\label{fig:flow_cancelled_switch}
\end{wrapfigure}

For a fixed masked state $z$, the two lifted copies are $(z,\downarrow)$ and $(z,\uparrow)$. If both momentum modes assign positive probability to switching between these two copies, the smaller of the two switch probabilities is merely a symmetric back-and-forth exchange. We remove this common exchange and turn it into same-copy holding probability, leaving only the residual imbalance as an actual momentum switch.

This is the direct analogue of crossing-flow cancellation in lifted Markov chains: when there is crossing flow in both directions between two lifted copies, the common part can be cancelled without changing the projected stationary flow.

\begin{definition}[Flow-cancelled momentum-balanced AOAR-VGB kernel]
\label{def:flow_cancelled_momentum_kernel}
Fix a cancellation strength
\begin{equation*}
    \chi\in[0,1].
\end{equation*}
For each reachable state $z$, define
\begin{equation*}
    m(z):=\min\{p_{\downarrow}(z),p_{\uparrow}(z)\}.
\end{equation*}
The flow-cancelled lifted kernel $\widetilde K_{\mathrm{BAL}}^{\mathrm{fc},\chi}$ on $\widetilde{\calZ}_{\mathrm{BAL}}$ is defined as follows.

\paragraph{Downward copy.}
From $(z,\downarrow)$, the chain moves to a child while keeping downward momentum with probability
\begin{equation*}
    \widetilde K_{\mathrm{BAL}}^{\mathrm{fc},\chi}\bigl((z,\downarrow),(c,\downarrow)\bigr)
    := K_{\mathrm{BAL}}(z,c), \qquad c\in C_{\mathrm{AOAR}}(z).
    \tag{\textsc{Child Move}}
\end{equation*}
It switches to upward momentum with residual opposite-direction probability
\begin{equation*}
    \widetilde K_{\mathrm{BAL}}^{\mathrm{fc},\chi}\bigl((z,\downarrow),(z,\uparrow)\bigr)
    := p_{\uparrow}(z)-\chi m(z),
    \tag{\textsc{Switch Up}}
\end{equation*}
and it stays in the same lifted state with the cancelled crossing mass
\begin{equation*}
    \widetilde K_{\mathrm{BAL}}^{\mathrm{fc},\chi}\bigl((z,\downarrow),(z,\downarrow)\bigr)
    := \chi m(z).
    \tag{\textsc{Stay Down}}
\end{equation*}

\paragraph{Upward copy.}
From $(z,\uparrow)$, the chain moves to a parent while keeping upward momentum with probability
\begin{equation*}
    \widetilde K_{\mathrm{BAL}}^{\mathrm{fc},\chi}\bigl((z,\uparrow),(p,\uparrow)\bigr)
    := K_{\mathrm{BAL}}(z,p), \qquad p\in P_{\mathrm{AOAR}}(z).
    \tag{\textsc{Parent Move}}
\end{equation*}
It switches to downward momentum with residual opposite-direction probability
\begin{equation*}
    \widetilde K_{\mathrm{BAL}}^{\mathrm{fc},\chi}\bigl((z,\uparrow),(z,\downarrow)\bigr)
    := p_{\downarrow}(z)-\chi m(z),
    \tag{\textsc{Switch Down}}
\end{equation*}
and it stays in the same lifted state with the cancelled crossing mass
\begin{equation*}
    \widetilde K_{\mathrm{BAL}}^{\mathrm{fc},\chi}\bigl((z,\uparrow),(z,\uparrow)\bigr)
    := \chi m(z).
    \tag{\textsc{Stay Up}}
\end{equation*}
All other transition probabilities are zero.
\end{definition}

The parameter $\chi$ acts as a cancellation knob, analogous to a friction coefficient in underdamped Langevin dynamics: it controls how strongly the lifted chain suppresses direction changes and preserves momentum persistence. Larger $\chi$ removes more symmetric back-and-forth switching and keeps more probability in the current momentum copy. The case $\chi=0$ performs no cancellation. The case $\chi=1$ maximally cancels opposing switch flows. In particular, if $p_{\downarrow}(z)=p_{\uparrow}(z)=1/2$ as in the exact canonical balanced case with $\widehat V=V^*$, then the switch probabilities at $z$ become zero and the cancelled mass becomes a same-copy holding probability $1/2$.

\begin{remark}[Implementation as a categorical distribution]
\label{rem:flow_cancelled_categorical_implementation}
Recall that $w^{(k)}$ denotes the depth-rescaled balanced AOAR-VGB local weight from \cref{def:balanced_aoar_weights}. Let
\begin{equation*}
    F(z):=\sum_{c\in C_{\mathrm{AOAR}}(z)}w^{(k(z))}(z\to c), \qquad
    B(z):=\sum_{p\in P_{\mathrm{AOAR}}(z)}w^{(k(z))}(z\to p),
\end{equation*}
and
\begin{equation*}
    M(z):=\min\mbra{F(z),B(z)}.
\end{equation*}
Then
\begin{equation*}
    p_{\downarrow}(z)=\frac{F(z)}{F(z)+B(z)}, \qquad
    p_{\uparrow}(z)=\frac{B(z)}{F(z)+B(z)}, \qquad
    m(z)=\frac{M(z)}{F(z)+B(z)}.
\end{equation*}
Equivalently, the kernel can be implemented by unnormalized categorical weights. In downward mode, sample from
\begin{align*}
    \calA_{\downarrow}(z)
    :=
    &\ \mbra{c\in C_{\mathrm{AOAR}}(z):\ \text{weight }w^{(k(z))}(z\to c)}\\
    &\cup\ \mbra{\texttt{switch\_up}:\ \text{weight }B(z)-\chi M(z)}
     \cup\ \mbra{\texttt{stay\_down}:\ \text{weight }\chi M(z)}.
\end{align*}
If a child is sampled, the chain moves to $(c,\downarrow)$. If $\texttt{switch\_up}$ is sampled, it moves to $(z,\uparrow)$. If $\texttt{stay\_down}$ is sampled, it remains at $(z,\downarrow)$.

In upward mode, sample from
\begin{align*}
    \calA_{\uparrow}(z)
    :=
    &\ \mbra{p\in P_{\mathrm{AOAR}}(z):\ \text{weight }w^{(k(z))}(z\to p)}\\
    &\cup\ \mbra{\texttt{switch\_down}:\ \text{weight }F(z)-\chi M(z)}
     \cup\ \mbra{\texttt{stay\_up}:\ \text{weight }\chi M(z)}.
\end{align*}
If a parent is sampled, the chain moves to $(p,\uparrow)$. If $\texttt{switch\_down}$ is sampled, it moves to $(z,\downarrow)$. If $\texttt{stay\_up}$ is sampled, it remains at $(z,\uparrow)$.
\end{remark}

The sampler implementation is the singleton-block restriction of the momentum
MDM-VGB sampler in \cref{alg:momentum_mdm_sampler}: set $\calM=\mbra{1}$ and
use the AOAR local forward/backward weights. Thus each MDM child/parent block
is a single revealed or re-masked coordinate.

\begin{proposition}[Markov property]
\label{prop:flow_cancelled_markov}
For every $\chi\in[0,1]$, the kernel $\widetilde K_{\mathrm{BAL}}^{\mathrm{fc},\chi}$ is a Markov kernel on $\widetilde{\calZ}_{\mathrm{BAL}}$.
\end{proposition}

\begin{proof}
Since $m(z)\le p_{\downarrow}(z)$ and $m(z)\le p_{\uparrow}(z)$, all transition probabilities are nonnegative. For the downward copy,
\begin{align*}
    &\sum_{c\in C_{\mathrm{AOAR}}(z)}
    \widetilde K_{\mathrm{BAL}}^{\mathrm{fc},\chi}\bigl((z,\downarrow),(c,\downarrow)\bigr)
    +\widetilde K_{\mathrm{BAL}}^{\mathrm{fc},\chi}\bigl((z,\downarrow),(z,\uparrow)\bigr)
    +\widetilde K_{\mathrm{BAL}}^{\mathrm{fc},\chi}\bigl((z,\downarrow),(z,\downarrow)\bigr)\\
    &\qquad=
    p_{\downarrow}(z)+\bigl(p_{\uparrow}(z)-\chi m(z)\bigr)+\chi m(z)
    =
    p_{\downarrow}(z)+p_{\uparrow}(z)=1.
\end{align*}
The upward row is identical:
\begin{align*}
    &\sum_{p\in P_{\mathrm{AOAR}}(z)}
    \widetilde K_{\mathrm{BAL}}^{\mathrm{fc},\chi}\bigl((z,\uparrow),(p,\uparrow)\bigr)
    +\widetilde K_{\mathrm{BAL}}^{\mathrm{fc},\chi}\bigl((z,\uparrow),(z,\downarrow)\bigr)
    +\widetilde K_{\mathrm{BAL}}^{\mathrm{fc},\chi}\bigl((z,\uparrow),(z,\uparrow)\bigr)\\
    &\qquad=
    p_{\uparrow}(z)+\bigl(p_{\downarrow}(z)-\chi m(z)\bigr)+\chi m(z)=1.
\end{align*}
Thus every row sums to one.
\end{proof}

\begin{theorem}[Stationarity of flow-cancelled momentum-balanced AOAR-VGB]
\label{thm:flow_cancelled_stationary}
Assume terminal anchoring,
\begin{equation*}
    \widehat V(x,y)=\tau(x,y), \qquad y\in\calY.
\end{equation*}
Let $K_{\mathrm{BAL}}$ be the self-loop-free ($\gamma=0$) balanced AOAR-VGB move kernel from \cref{def:balanced_aoar_weights}, restricted to a reachable component where $W_{\mathrm{BAL}}(z)>0$ for every state $z$. Then, from \cref{thm:balanced_weighted_graph} $K_{\mathrm{BAL}}$ is reversible with respect to $\mu_{\mathrm{BAL}}$, and $\mu_{\mathrm{BAL}}(\cdot\mid\calY)=\pi^*(\cdot\mid x)$.

For every $\chi\in[0,1]$, the lifted measure
\begin{equation*}
    \widetilde\mu_{\mathrm{BAL}}(z,\downarrow)=\widetilde\mu_{\mathrm{BAL}}(z,\uparrow)=\frac12\mu_{\mathrm{BAL}}(z)
\end{equation*}
is stationary for $\widetilde K_{\mathrm{BAL}}^{\mathrm{fc},\chi}$. Moreover,
\begin{equation*}
    \textsf{proj}_{\#}\widetilde\mu_{\mathrm{BAL}}=\mu_{\mathrm{BAL}}.
\end{equation*}
Consequently, the flow-cancelled lifted chain has the exact tilted target as its leaf-conditioned current-state law:
\begin{equation*}
    \textsf{proj}_{\#}\widetilde\mu_{\mathrm{BAL}}(\cdot\mid\calY)=\pi^*(\cdot\mid x).
\end{equation*}
\end{theorem}

\begin{proof}
We verify stationarity by checking incoming flow into each lifted copy. First consider $(z,\downarrow)$. Incoming flow comes from three sources: parents of $z$ in downward mode, the switch from $(z,\uparrow)$, and the self-loop at $(z,\downarrow)$.

The incoming move-flow from parent states $u\in P_{\mathrm{AOAR}}(z)$ is
\begin{equation*}
    \frac12\sum_{u\in P_{\mathrm{AOAR}}(z)}\mu_{\mathrm{BAL}}(u)K_{\mathrm{BAL}}(u,z).
\end{equation*}
Using reversibility of $K_{\mathrm{BAL}}$,
\begin{align*}
    \frac12\sum_{u\in P_{\mathrm{AOAR}}(z)}\mu_{\mathrm{BAL}}(u)K_{\mathrm{BAL}}(u,z)
    &=
    \frac12\mu_{\mathrm{BAL}}(z)\sum_{u\in P_{\mathrm{AOAR}}(z)}K_{\mathrm{BAL}}(z,u)=
    \frac12\mu_{\mathrm{BAL}}(z)p_{\uparrow}(z).
\end{align*}
The incoming switch-flow from $(z,\uparrow)$ is
\begin{equation*}
    \frac12\mu_{\mathrm{BAL}}(z)\bigl(p_{\downarrow}(z)-\chi m(z)\bigr),
\end{equation*}
and the self-loop flow at $(z,\downarrow)$ is
\begin{equation*}
    \frac12\mu_{\mathrm{BAL}}(z)\chi m(z).
\end{equation*}
Therefore the total incoming flow into $(z,\downarrow)$ is
\begin{align*}
    &\frac12\mu_{\mathrm{BAL}}(z)p_{\uparrow}(z)
    +\frac12\mu_{\mathrm{BAL}}(z)\bigl(p_{\downarrow}(z)-\chi m(z)\bigr)
    +\frac12\mu_{\mathrm{BAL}}(z)\chi m(z)\\
    &\qquad=
    \frac12\mu_{\mathrm{BAL}}(z)\bigl(p_{\uparrow}(z)+p_{\downarrow}(z)\bigr)
    =
    \frac12\mu_{\mathrm{BAL}}(z)=\widetilde\mu_{\mathrm{BAL}}(z,\downarrow).
\end{align*}

The proof for $(z,\uparrow)$ is symmetric. Incoming move-flow comes from children $v\in C_{\mathrm{AOAR}}(z)$:
\begin{align*}
    \frac12\sum_{v\in C_{\mathrm{AOAR}}(z)}\mu_{\mathrm{BAL}}(v)K_{\mathrm{BAL}}(v,z)
    &=
    \frac12\mu_{\mathrm{BAL}}(z)\sum_{v\in C_{\mathrm{AOAR}}(z)}K_{\mathrm{BAL}}(z,v)=
    \frac12\mu_{\mathrm{BAL}}(z)p_{\downarrow}(z),
\end{align*}
where the first equality uses reversibility of $K_{\mathrm{BAL}}$. The incoming switch-flow from $(z,\downarrow)$ is
\begin{equation*}
    \frac12\mu_{\mathrm{BAL}}(z)\bigl(p_{\uparrow}(z)-\chi m(z)\bigr),
\end{equation*}
and the self-loop flow at $(z,\uparrow)$ is
\begin{equation*}
    \frac12\mu_{\mathrm{BAL}}(z)\chi m(z).
\end{equation*}
Thus the total incoming flow into $(z,\uparrow)$ is
\begin{equation*}
    \frac12\mu_{\mathrm{BAL}}(z)p_{\downarrow}(z)
    +\frac12\mu_{\mathrm{BAL}}(z)\bigl(p_{\uparrow}(z)-\chi m(z)\bigr)
    +\frac12\mu_{\mathrm{BAL}}(z)\chi m(z)
    =
    \frac12\mu_{\mathrm{BAL}}(z).
\end{equation*}
Therefore $\widetilde\mu_{\mathrm{BAL}}$ is stationary.

Finally,
\begin{equation*}
    \textsf{proj}_{\#}\widetilde\mu_{\mathrm{BAL}}(z)
    =
    \widetilde\mu_{\mathrm{BAL}}(z,\downarrow)+\widetilde\mu_{\mathrm{BAL}}(z,\uparrow)
    =
    \mu_{\mathrm{BAL}}(z).
\end{equation*}
By \cref{thm:balanced_weighted_graph}, terminal anchoring implies $\mu_{\mathrm{BAL}}(\cdot\mid\calY)=\pi^*(\cdot\mid x)$. Since the lifted current-state marginal is $\mu_{\mathrm{BAL}}$, the same leaf-conditioned law holds after projection.
\end{proof}

\subsection{Momentum Lift for Geometric Balanced AOAR-VGB}
\label{app:subsec:momentum_bg_aoar_vgb}

Let $K_{\mathrm{BG}}$ denote the non-lazy move kernel induced by \cref{def:bg_aoar_edge_weights}, and let $\mu_{\mathrm{BG}}$ be its stationary law from \cref{thm:bg_aoar_stationarity}. Define directional probabilities
\begin{equation*}
    p_{\downarrow}(z):=\sum_{c\in C_{\mathrm{AOAR}}(z)}K_{\mathrm{BG}}(z,c),
    \qquad
    p_{\uparrow}(z):=\sum_{p\in P_{\mathrm{AOAR}}(z)}K_{\mathrm{BG}}(z,p).
\end{equation*}
Then $p_{\downarrow}(z)+p_{\uparrow}(z)=1$ at non-isolated states. The same flow-cancelled lift used above applies to this geometric base chain.

\cref{fig:geometric_aoar_vgb_momentum_transition} illustrates the lifted transition. The downward copy keeps forward reveal moves, the upward copy keeps backward re-masking moves, and the residual opposite-direction mass switches the momentum label.

\begin{figure}[t]
    \centering
    \vspace{-0.5em}
    \includegraphics[width=0.95\linewidth]{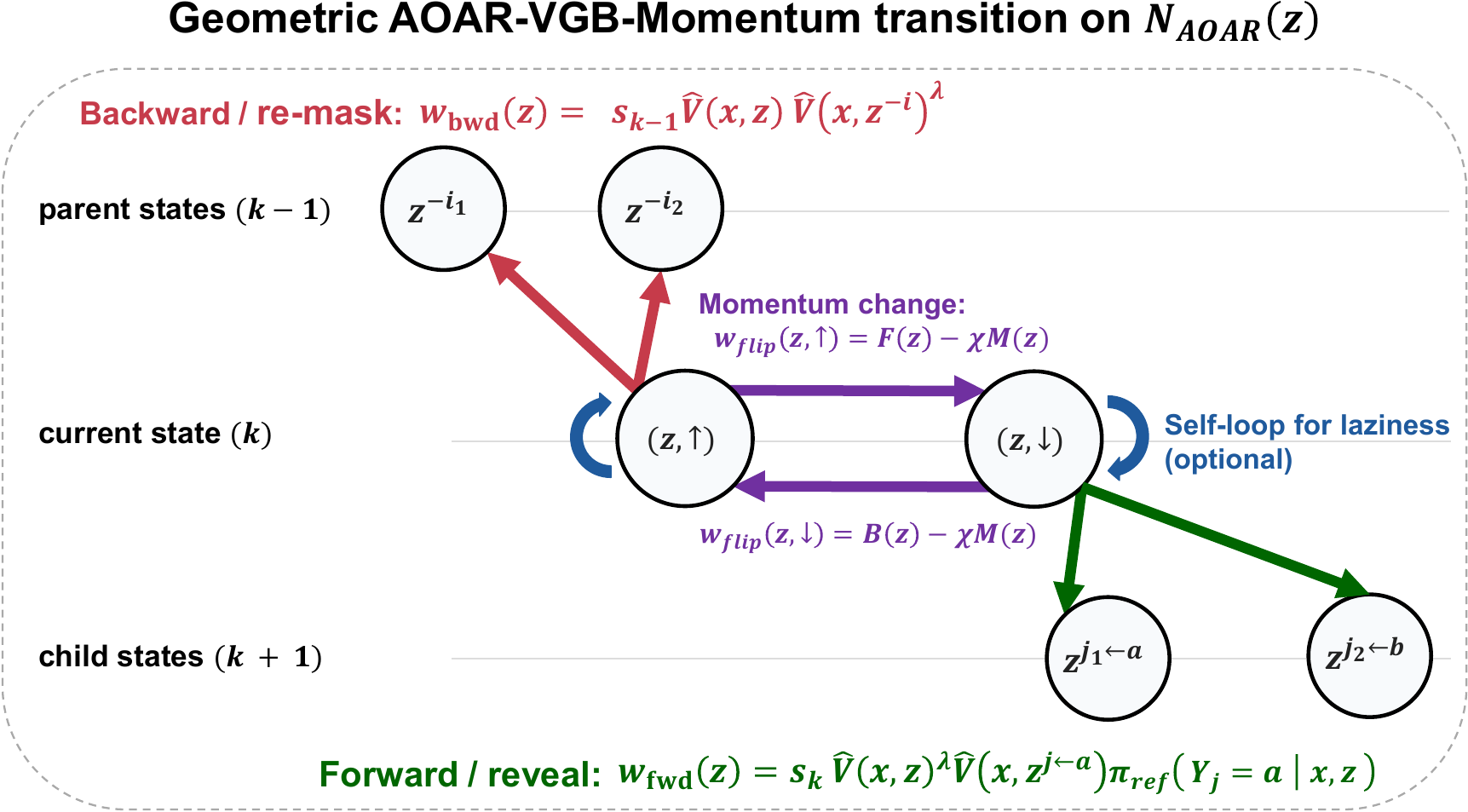}
    \caption{
        Illustration of the momentum lift for geometric balanced AOAR-VGB.
        Forward moves preserve downward momentum, backward moves preserve upward momentum, and residual opposite-direction mass switches between lifted copies; the cancelled common mass may be kept as an optional stay probability.
    }
    \label{fig:geometric_aoar_vgb_momentum_transition}
    \vspace{-0.5em}
\end{figure}

\begin{definition}[Flow-cancelled momentum geometric AOAR kernel]
\label{def:bg_aoar_momentum_kernel}
Fix $\chi\in[0,1]$ and set
\begin{equation*}
    m(z):=\min\mbra{p_{\downarrow}(z),p_{\uparrow}(z)}.
\end{equation*}
On the lifted state space
\begin{equation*}
    \widetilde{\calZ}_{\mathrm{BG}}:=\calZ_+(x)\times\mbra{\downarrow,\uparrow},
\end{equation*}
define the lifted kernel $\widetilde K_{\mathrm{BG}}^{\mathrm{fc},\chi}$ as follows.

\paragraph{Downward copy.}
From $(z,\downarrow)$, the chain moves to a child while keeping downward momentum with probability
\begin{equation*}
    \widetilde K_{\mathrm{BG}}^{\mathrm{fc},\chi}
    \bigl((z,\downarrow),(c,\downarrow)\bigr)
    :=
    K_{\mathrm{BG}}(z,c),\qquad c\in C_{\mathrm{AOAR}}(z).
    \tag{\textsc{Child Move}}
\end{equation*}
It switches to upward momentum with residual opposite-direction probability
\begin{equation*}
    \widetilde K_{\mathrm{BG}}^{\mathrm{fc},\chi}
    \bigl((z,\downarrow),(z,\uparrow)\bigr)
    :=
    p_{\uparrow}(z)-\chi m(z).
    \tag{\textsc{Switch Up}}
\end{equation*}
and it stays in the same lifted state with the cancelled crossing mass
\begin{equation*}
    \widetilde K_{\mathrm{BG}}^{\mathrm{fc},\chi}
    \bigl((z,\downarrow),(z,\downarrow)\bigr)
    :=
    \chi m(z).
    \tag{\textsc{Stay Down}}
\end{equation*}

\paragraph{Upward copy.}
From $(z,\uparrow)$, the chain moves to a parent while keeping upward momentum with probability
\begin{equation*}
    \widetilde K_{\mathrm{BG}}^{\mathrm{fc},\chi}
    \bigl((z,\uparrow),(p,\uparrow)\bigr)
    :=
    K_{\mathrm{BG}}(z,p),\qquad p\in P_{\mathrm{AOAR}}(z).
    \tag{\textsc{Parent Move}}
\end{equation*}
It switches to downward momentum with residual opposite-direction probability
\begin{equation*}
    \widetilde K_{\mathrm{BG}}^{\mathrm{fc},\chi}
    \bigl((z,\uparrow),(z,\downarrow)\bigr)
    :=
    p_{\downarrow}(z)-\chi m(z).
    \tag{\textsc{Switch Down}}
\end{equation*}
and it stays in the same lifted state with the cancelled crossing mass
\begin{equation*}
    \widetilde K_{\mathrm{BG}}^{\mathrm{fc},\chi}
    \bigl((z,\uparrow),(z,\uparrow)\bigr)
    :=
    \chi m(z).
    \tag{\textsc{Stay Up}}
\end{equation*}
All other transition probabilities are zero.
\end{definition}

\begin{theorem}[Momentum preserves the geometric AOAR leaf law]
\label{thm:bg_aoar_momentum_stationarity}
Define
\begin{equation*}
    \widetilde\mu_{\mathrm{BG}}(z,\downarrow)
    =
    \widetilde\mu_{\mathrm{BG}}(z,\uparrow)
    :=
    \frac12\mu_{\mathrm{BG}}(z).
\end{equation*}
Then $\widetilde\mu_{\mathrm{BG}}$ is stationary for $\widetilde K_{\mathrm{BG}}^{\mathrm{fc},\chi}$ for every $\chi\in[0,1]$. Moreover, its current-state projection is $\mu_{\mathrm{BG}}$, and therefore
\begin{equation*}
    \textsf{proj}_{\#}\widetilde\mu_{\mathrm{BG}}(\cdot\mid\calY)
    =
    \pi^*(\cdot\mid x).
\end{equation*}
\end{theorem}

\begin{proof}
We prove stationarity for the downward copy; the upward proof is symmetric. Incoming flow into $(z,\downarrow)$ has three parts. First, parent states $u\in P_{\mathrm{AOAR}}(z)$ in downward mode contribute
\begin{equation*}
    \frac12\sum_{u\in P_{\mathrm{AOAR}}(z)}
    \mu_{\mathrm{BG}}(u)K_{\mathrm{BG}}(u,z).
\end{equation*}
Since $K_{\mathrm{BG}}$ is reversible with respect to $\mu_{\mathrm{BG}}$, this equals
\begin{equation*}
    \frac12\mu_{\mathrm{BG}}(z)
    \sum_{u\in P_{\mathrm{AOAR}}(z)}K_{\mathrm{BG}}(z,u)
    =
    \frac12\mu_{\mathrm{BG}}(z)p_{\uparrow}(z).
\end{equation*}
Second, the switch from $(z,\uparrow)$ contributes
\begin{equation*}
    \frac12\mu_{\mathrm{BG}}(z)(p_{\downarrow}(z)-\chi m(z)).
\end{equation*}
Third, the self-loop at $(z,\downarrow)$ contributes
\begin{equation*}
    \frac12\mu_{\mathrm{BG}}(z)\chi m(z).
\end{equation*}
The total incoming flow is therefore
\begin{equation*}
    \frac12\mu_{\mathrm{BG}}(z)\mbra{p_{\uparrow}(z)+p_{\downarrow}(z)}
    =
    \frac12\mu_{\mathrm{BG}}(z)
    =
    \widetilde\mu_{\mathrm{BG}}(z,\downarrow).
\end{equation*}
The upward-copy calculation is identical with parents and children exchanged. Projection gives
\begin{equation*}
    \widetilde\mu_{\mathrm{BG}}(z,\downarrow)+
    \widetilde\mu_{\mathrm{BG}}(z,\uparrow)
    =
    \mu_{\mathrm{BG}}(z).
\end{equation*}
The leaf law follows from \cref{thm:bg_aoar_stationarity}.
\end{proof}

\subsection{Theoretical Analysis: Depth-Speed Gain with Momentum}
\label{app:subsec:momentum_balanced_mixing_speed_gain}

We now focus on the depth-wise behavior of Balanced AOAR-VGB (the $\lambda=0$, non-geometric chain) and its flow-cancelled momentum lift. We note that this theoretical analysis is restricted to this balanced chain, rather than the geometric variants above.

In AOAR-VGB without momentum, even when forward and backward moves are locally balanced, the reversible chain still behaves like a diffusive random walk along the depth axis, leading to a quadratic first-leaf hitting time. We show that, under a near-balance condition on the realized forward and backward masses, the no-momentum chain remains diffusive with first-leaf hitting time of order $n^2$, whereas the maximally flow-cancelled momentum lift can accelerate this depth traversal to order $n$ steps.

This is the quantity most aligned with practical VGB-style implementations, which often stop the Markov chain once it first reaches a leaf.

\begin{definition}[Depth process and hitting time]
\label{def:depth_process}
For a masked state $z$, let $k(z)=|R(z)|$. For the balanced chain with current masked state $Z_t^{\mathrm{BAL}}$, define
\begin{equation*}
    \tau_t^{\mathrm{BAL}}:=k(Z_t^{\mathrm{BAL}}), \qquad
    T_n^{\mathrm{BAL}}:=\inf\{t\ge 0:\tau_t^{\mathrm{BAL}}=n\}.
\end{equation*}
For a lifted momentum chain $\widetilde Z_t=(Z_t^{\mathrm{BAL,Mo}},D_t)$, define
\begin{equation*}
    \tau_t^{\mathrm{BAL,Mo}}:=k(Z_t^{\mathrm{BAL,Mo}}), \qquad
    T_n^{\mathrm{BAL,Mo}}:=\inf\{t\ge 0:\tau_t^{\mathrm{BAL,Mo}}=n\}.
\end{equation*}
\end{definition}

\begin{assumption}[Uniform directional near-balance]
\label{ass:directional_near_balance}
There exists $\varepsilon_n\in[0,1)$ such that for every reachable interior state $z$,
\begin{equation*}
    \abs{p_{\downarrow}(z)-p_{\uparrow}(z)}\le \varepsilon_n.
\end{equation*}
Equivalently, since $p_{\downarrow}(z)+p_{\uparrow}(z)=1$,
\begin{equation*}
    \frac{1-\varepsilon_n}{2}\le p_{\downarrow}(z),p_{\uparrow}(z)\le \frac{1+\varepsilon_n}{2}.
\end{equation*}
\end{assumption}

The exact canonical balanced case corresponds to $\varepsilon_n=0$. The near-balanced speedup regime is
\begin{equation*}
    \varepsilon_n=O(1/n).
\end{equation*}
Intuitively, this says that the plug-in verifier may break exact $1/2,1/2$ balance, but only by the same order as the inverse sequence length $1/n$.

\begin{lemma}[Biased reflecting walk hitting time]
\label{lem:biased_reflecting_walk_hitting}
Consider a birth-death chain on $\mbra{0,1,\dots,n}$ with absorbing state $n$, deterministic transition $0\to1$, and for $1\le k\le n-1$,
\begin{equation*}
    k\to k+1 \text{ with probability } p, \qquad
    k\to k-1 \text{ with probability } q:=1-p.
\end{equation*}
Let $(K_t)_{t\ge0}$ denote this birth-death chain, and define the hitting time of depth $n$ by
\begin{equation*}
    T_n^{(p)}:=\inf\mbra{t\ge0:K_t=n}.
\end{equation*}
Set
\begin{equation*}
    H_0^{(p)}:=\bbE\lbra{T_n^{(p)}\mid K_0=0}.
\end{equation*}
If $p\ne q$, then
\begin{equation*}
    H_0^{(p)}
    =
    \frac{n}{p-q}
    -
    \frac{2pq}{(p-q)^2}\sbra{1-\sbra{\frac{q}{p}}^n}.
\end{equation*}
If $p=q=1/2$, then
\begin{equation*}
    H_0^{(1/2)}=n^2.
\end{equation*}
Moreover, for every fixed $c>0$, if
\begin{equation*}
    \abs{p-\frac12}\le \frac{c}{2n},
\end{equation*}
then
\begin{equation*}
    H_0^{(p)}=\Theta_c(n^2).
\end{equation*}
\end{lemma}

\begin{proof}
Let
\begin{equation*}
    h_k:=\bbE\lbra{T_n^{(p)}\mid K_0=k}.
\end{equation*}
Then $h_n=0$, and the reflecting root gives $h_0=1+h_1$. For $1\le k\le n-1$,
\begin{equation*}
    h_k=1+p h_{k+1}+q h_{k-1}.
\end{equation*}
Define finite differences $\Delta_k:=h_k-h_{k-1}$ for $1\le k\le n$. The boundary equation gives $\Delta_1=-1$, and the interior recurrence is equivalent to
\begin{equation*}
    p\Delta_{k+1}-q\Delta_k=-1.
\end{equation*}
If $p\ne q$, set
\begin{equation*}
    r:=\frac{q}{p}, \qquad S_n(r):=\sum_{\ell=0}^{n-1}r^\ell.
\end{equation*}
Then the solution is
\begin{equation*}
    \Delta_k=\sbra{-1+\frac{1}{p-q}}r^{k-1}-\frac{1}{p-q}.
\end{equation*}
Since $h_n=0$,
\begin{equation*}
    -h_0=h_n-h_0=\sum_{k=1}^{n}\Delta_k
    =
    \sbra{-1+\frac{1}{p-q}}S_n(r)-\frac{n}{p-q}.
\end{equation*}
Rearranging and using $S_n(r)=\sbra{1-r^n}/\sbra{1-r}$ gives
\begin{equation*}
    h_0
    =
    S_n(r)+\frac{n-S_n(r)}{p-q}
    =
    \frac{n}{p-q}
    -
    \frac{2pq}{(p-q)^2}\sbra{1-\sbra{\frac{q}{p}}^n}.
\end{equation*}
If $p=q=1/2$, the recurrence becomes $h_k=1+\frac12h_{k+1}+\frac12h_{k-1}$, and the solution satisfying $h_n=0$ and $h_0=1+h_1$ is $h_k=n^2-k^2$. Thus $h_0=n^2$.

Finally suppose $\abs{p-1/2}\le c/(2n)$. Then $\abs{p-q}\le c/n$. If $p=q$, the claim follows from $h_0=n^2$. It remains to consider $p\ne q$. Recall that $r=q/p$; equivalently,
\begin{equation*}
    r=\frac{q}{p}=\frac{1-(p-q)}{1+(p-q)}
\end{equation*}
because $p+q=1$. For $n\ge 2c$, we have $\abs{p-q}\le 1/2$. The mean-value theorem then gives
\begin{equation*}
    \abs{\log r}
    =
    \abs{\log\sbra{1-(p-q)}-\log\sbra{1+(p-q)}}
    \le 4\abs{p-q}.
\end{equation*}
Therefore, for $0\le \ell\le n-1$,
\begin{equation*}
    e^{-4c}\le e^{-4\abs{p-q}n}\le r^\ell \le e^{4\abs{p-q}n}\le e^{4c}.
\end{equation*}
Summing these bounds gives
\begin{equation*}
    ne^{-4c}\le S_n(r)=\sum_{\ell=0}^{n-1}r^\ell\le ne^{4c},
\end{equation*}
so $S_n(r)=\Theta_c(n)$. The finitely many cases $n<2c$ are absorbed into the constants hidden by $\Theta_c$.
Moreover,
\begin{equation*}
    n-S_n(r)
    =
    \sum_{\ell=0}^{n-1}(1-r^\ell)
    =
    (1-r)\sum_{\ell=1}^{n-1}\sum_{j=0}^{\ell-1}r^j,
    \qquad
    \abs{1-r}=\Theta_c(\abs{p-q}).
\end{equation*}
The double sum is $\Theta_c(n^2)$, again because all powers $r^j$ stay within constant factors. Thus $\abs{n-S_n(r)}=\Theta_c(\abs{p-q}n^2)$, and $(n-S_n(r))/(p-q)=\Theta_c(n^2)$. Substituting into $h_0=S_n(r)+(n-S_n(r))/(p-q)$ gives $h_0=\Theta_c(n^2)$.
\end{proof}

\begin{proposition}[No-momentum remains diffusive under near-balance]
\label{prop:no_momentum_near_balance_quadratic}
Assume \cref{ass:directional_near_balance} with $\varepsilon_n\le c/n$ for a fixed constant $c>0$. Consider the self-loop-free no-momentum balanced AOAR-VGB depth process, started from depth $0$. Then
\begin{equation*}
    \bbE\lbra{T_n^{\mathrm{BAL}}\mid \tau_0^{\mathrm{BAL}}=0}=\Theta_c(n^2).
\end{equation*}
\end{proposition}

\begin{proof}
At every interior state, the probability of increasing the depth is between
\begin{equation*}
    p_-:=\frac{1-\varepsilon_n}{2}\qquad\text{and}\qquad p_+:=\frac{1+\varepsilon_n}{2}.
\end{equation*}
Because the conditional probability of increasing depth at each interior step is always in \([p_-,p_+]\), the no-momentum depth process can be coupled between two reflecting birth-death chains using common uniform random variables at each step. The favorable chain increases depth with probability \(p_+\) and reaches depth \(n\) no later than the AOAR depth process, while the unfavorable chain increases depth with probability \(p_-\) and reaches depth \(n\) no earlier. Therefore,
\begin{equation*}
    H_0^{(p_+)}
    \le
    \bbE\lbra{T_n^{\mathrm{BAL}}\mid \tau_0^{\mathrm{BAL}}=0}
    \le
    H_0^{(p_-)}.
\end{equation*}
Since $|p_\pm-1/2|\le c/(2n)$, \cref{lem:biased_reflecting_walk_hitting} gives $H_0^{(p_+)}=\Theta_c(n^2)$ and $H_0^{(p_-)}=\Theta_c(n^2)$. Hence $\bbE[T_n^{\mathrm{BAL}}\mid \tau_0^{\mathrm{BAL}}=0]=\Theta_c(n^2)$.
\end{proof}

\begin{theorem}[First-leaf hitting for near-maximal flow cancellation]
\label{thm:robust_flow_cancelled_hitting}
Assume \cref{ass:directional_near_balance}. Run the flow-cancelled momentum
chain $\widetilde K_{\mathrm{BAL}}^{\mathrm{fc},\chi}$ from the initial
lifted state $(\varnothing,\downarrow)$, where $\chi\in[0,1]$. Then
\begin{equation*}
    \bbE\lbra{
        T_n^{\mathrm{BAL,Mo}}
        \mid
        \widetilde Z_0=(\varnothing,\downarrow)
    }
    \le
    \sbra{
        \frac{4n}{1-\varepsilon_n}+1
    }
    \sbra{
        \frac{2-\chi(1-\varepsilon_n)}
             {1-\varepsilon_n}
    }^n .
\end{equation*}
In particular, if $\varepsilon_n\le \min\mbra{c/n,1/2}$ and
$1-\chi\le c_\chi/n$ for fixed constants $c,c_\chi>0$, then
\begin{equation*}
    \bbE\lbra{
        T_n^{\mathrm{BAL,Mo}}
        \mid
        \widetilde Z_0=(\varnothing,\downarrow)
    }
    =
    O_{c,c_\chi}(n).
\end{equation*}
For $\chi=1$, this recovers the maximally flow-cancelled bound.
\end{theorem}

\begin{proof}
Let $q_{\varepsilon}:=(1-\varepsilon_n)/2$. Under
\cref{ass:directional_near_balance}, at every interior state $z$,
$p_{\downarrow}(z),p_{\uparrow}(z)\ge q_{\varepsilon}$. From
$(z,\downarrow)$, the flow-cancelled kernel has downward move probability
$p_{\downarrow}(z)$, switch-up probability
$p_{\uparrow}(z)-\chi m(z)$, and same-copy holding probability
$\chi m(z)$, where
$m(z):=\min\{p_{\downarrow}(z),p_{\uparrow}(z)\}$. Thus, while in downward
mode, the probability of a downward depth increment is at least
$q_{\varepsilon}$.

We decompose the trajectory into downward attempts. A downward attempt starts
in downward mode. It succeeds if the chain reaches depth $n$ while remaining
in downward mode, and it fails if the chain switches from downward mode to
upward mode before reaching depth $n$.

Same-copy holding steps do not change either the depth or the momentum label.
For the success-probability calculation, condition them out and look only at
decisive steps, namely steps on which either a downward move or an upward switch
occurs. Given the current lifted state $(z,\downarrow)$,
\begin{align*}
    \bbP(\text{downward move}\mid \text{decisive},z,\downarrow)
    &=
    \frac{p_{\downarrow}(z)}
    {p_{\downarrow}(z)+p_{\uparrow}(z)-\chi m(z)}\\
    &=
    \frac{p_{\downarrow}(z)}
    {1-\chi m(z)}
    \ge
    \frac{q_\varepsilon}
    {1-\chi q_\varepsilon}
    :=
    \alpha_{\chi,\varepsilon}
    =
    \frac{1-\varepsilon_n}{2-\chi(1-\varepsilon_n)}.
\end{align*}
To succeed from depth $k$, an attempt only needs the next $n-k$ decisive
outcomes to be downward moves. Therefore, with
$s_{\chi,\varepsilon,n}:=\alpha_{\chi,\varepsilon}^{\,n}$,
\begin{equation*}
    \bbP(\text{attempt succeeds from depth }k)
    \ge
    \alpha_{\chi,\varepsilon}^{\,n-k}
    \ge
    \alpha_{\chi,\varepsilon}^{\,n}.
\end{equation*}
Thus every downward attempt succeeds with probability at least
$s_{\chi,\varepsilon,n}$, uniformly over the past.

Next, we bound the expected duration of one attempt cycle, counting Markov
transitions including same-copy holding steps. During the downward part, every
transition in downward mode increases the depth with probability at least
$q_\varepsilon$. To upper-bound the time cost, ignore early switch failures
and wait until $n$ downward depth increments have occurred. If $G_r$ is the
waiting time for the $r$-th such increment after the $(r-1)$-st one, then
$\bbP(G_r>t)\le (1-q_\varepsilon)^t$ and
$\bbE[G_r]\le 1/q_\varepsilon$. Hence the expected duration of the downward
part is at most $n/q_\varepsilon$.

If the attempt fails, the chain has switched to upward mode before reaching
depth $n$. We upper-bound the cost of preparing the next attempt by a
conservative reset: move upward until the root and then take the boundary
switch back to downward mode. Returning to the root requires at most $n$
upward depth decrements, each occurring with probability at least
$q_\varepsilon$. Thus the reset duration is at most $n/q_\varepsilon+1$
in expectation, and one complete attempt cycle has expected duration at most
$L_\varepsilon:=2n/q_\varepsilon+1$.

Let $N_{\mathrm{att}}$ be the number of attempts until the first successful
one. Since each attempt succeeds with conditional probability at least
$s_{\chi,\varepsilon,n}$, we have
$\bbP(N_{\mathrm{att}}>m)\le(1-s_{\chi,\varepsilon,n})^m$ and therefore
$\bbE[N_{\mathrm{att}}]\le 1/s_{\chi,\varepsilon,n}$.

Let $C_j$ denote the duration of the $j$-th attempt cycle. The cycle-length
bound above is uniform over the past, so
$\bbE[C_j\mid \calH_{j-1}]\le L_\varepsilon$ whenever the $j$-th cycle
starts. Since
$T_n^{\mathrm{BAL,Mo}}\le
\sum_{j\ge1}\mathbf 1\{N_{\mathrm{att}}\ge j\}C_j$,
we obtain
\begin{equation*}
    \bbE[T_n^{\mathrm{BAL,Mo}}]
    \le
    L_\varepsilon\sum_{j\ge1}\bbP(N_{\mathrm{att}}\ge j)
    =
    L_\varepsilon\bbE[N_{\mathrm{att}}]
    \le
    \left(\frac{2n}{q_\varepsilon}+1\right)
    \alpha_{\chi,\varepsilon}^{-n}.
\end{equation*}
Substituting $q_\varepsilon=(1-\varepsilon_n)/2$ and
$\alpha_{\chi,\varepsilon}=(1-\varepsilon_n)/(2-\chi(1-\varepsilon_n))$
gives
\begin{equation*}
    \bbE[T_n^{\mathrm{BAL,Mo}}]
    \le
    \sbra{
        \frac{4n}{1-\varepsilon_n}+1
    }
    \sbra{
        \frac{2-\chi(1-\varepsilon_n)}
             {1-\varepsilon_n}
    }^n .
\end{equation*}

Finally, suppose $\varepsilon_n\le \min\mbra{c/n,1/2}$ and
$1-\chi\le c_\chi/n$. Then
\begin{align*}
    \frac{2-\chi(1-\varepsilon_n)}
         {1-\varepsilon_n}
    &=
    1+
    \frac{
        1-\chi+\varepsilon_n(1+\chi)
    }{
        1-\varepsilon_n
    }=
    1+O_{c,c_\chi}\!\left(\frac1n\right).
\end{align*}
Therefore its $n$-th power is $O_{c,c_\chi}(1)$, while
$(4n)/(1-\varepsilon_n)+1=O_c(n)$. Hence
\begin{equation*}
    \bbE\lbra{
        T_n^{\mathrm{BAL,Mo}}
        \mid
        \widetilde Z_0=(\varnothing,\downarrow)
    }
    =
    O_{c,c_\chi}(n).
\end{equation*}
\end{proof}

\begin{corollary}[Depth-axis speedup with near-maximal momentum]
\label{cor:robust_depth_axis_speedup}
Assume \cref{ass:directional_near_balance} with
\begin{equation*}
    \varepsilon_n\le \min\left\{\frac{c}{n},\frac12\right\}
\end{equation*}
for a fixed constant $c>0$. Let the flow-cancelled momentum parameter satisfy
\begin{equation*}
    1-\chi\le \frac{c_\chi}{n}
\end{equation*}
for a fixed constant $c_\chi>0$. Then the no-momentum balanced AOAR-VGB depth
process satisfies
\begin{equation*}
    \bbE\lbra{T_n^{\mathrm{BAL}}\mid \tau_0^{\mathrm{BAL}}=0}
    =
    \Theta_c(n^2),
\end{equation*}
while the flow-cancelled momentum lift satisfies
\begin{equation*}
    \bbE\lbra{
        T_n^{\mathrm{BAL,Mo}}
        \mid
        \widetilde Z_0=(\varnothing,\downarrow)
    }
    =
    O_{c,c_\chi}(n).
\end{equation*}
Thus, under the same near-balance condition and near-maximal cancellation,
\begin{equation*}
    \frac{
        \bbE[T_n^{\mathrm{BAL}}\mid \tau_0^{\mathrm{BAL}}=0]
    }{
        \bbE[T_n^{\mathrm{BAL,Mo}}\mid
        \widetilde Z_0=(\varnothing,\downarrow)]
    }
    =
    \Omega_{c,c_\chi}(n).
\end{equation*}
The maximally flow-cancelled case $\chi=1$ is included by taking
$c_\chi=0$.
\end{corollary}

\begin{proof}
This follows immediately from \cref{prop:no_momentum_near_balance_quadratic} and \cref{thm:robust_flow_cancelled_hitting}.
\end{proof}

\begin{corollary}[Ratio form of the near-balance condition]
\label{cor:ratio_form_robust_speedup}
Suppose there exists $\rho_n\ge1$ such that for every reachable interior state $z$,
\begin{equation*}
    \frac{1}{\rho_n}\le \frac{p_{\downarrow}(z)}{p_{\uparrow}(z)}\le \rho_n.
\end{equation*}
Then
\begin{equation*}
    \abs{p_{\downarrow}(z)-p_{\uparrow}(z)}\le \frac{\rho_n-1}{\rho_n+1}.
\end{equation*}
Consequently, if there are fixed constants $c,c_\chi>0$ such that
\begin{equation*}
    \frac{\rho_n-1}{\rho_n+1}
    \le
    \min\left\{\frac{c}{n},\frac12\right\},
    \qquad
    1-\chi\le \frac{c_\chi}{n},
\end{equation*}
then the depth-axis speedup in
\cref{cor:robust_depth_axis_speedup} applies. In particular, the ratio condition
holds asymptotically when $\rho_n=1+O(1/n)$, provided the cancellation is
near-maximal.
\end{corollary}

\begin{proof}
Let $p=p_{\downarrow}(z)$ and $q=p_{\uparrow}(z)$. Since $p+q=1$ and $p/q\le\rho_n$, we have $p\le \rho_n q=\rho_n(1-p)$, so $p\le \rho_n/(1+\rho_n)$. Similarly, $q/p\le\rho_n$ gives $p\ge1/(1+\rho_n)$. Therefore,
\begin{equation*}
    \abs{p-\frac12}\le \frac{\rho_n-1}{2(\rho_n+1)}.
\end{equation*}
Since $|p-q|=|2p-1|$, we obtain $|p-q|\le(\rho_n-1)/(\rho_n+1)$. The displayed condition above gives \cref{ass:directional_near_balance} with $\varepsilon_n\le \min\mbra{c/n,1/2}$ and the near-maximal cancellation condition $1-\chi\le c_\chi/n$, so \cref{cor:robust_depth_axis_speedup} applies. If $\rho_n=1+O(1/n)$ and $1-\chi=O(1/n)$, then these conditions hold for all sufficiently large $n$.
\end{proof}

\begin{corollary}[Implication of multiplicative verifier accuracy]
\label{cor:kappa_for_linear_momentum}
Assume the canonical balanced coefficients $\beta_k=k$ and $\alpha_k=n-k$ for $1\le k\le n-1$. Suppose the plug-in verifier satisfies the multiplicative accuracy condition
\begin{equation*}
    \kappa^{-1}V^*(x,z)\le \widehat V(x,z)\le \kappa V^*(x,z)
\end{equation*}
on all relevant non-leaf states. Then, for every reachable interior state $z$,
\begin{equation*}
    \frac{1}{\kappa^2}\le \frac{p_{\downarrow}(z)}{p_{\uparrow}(z)}\le \kappa^2.
\end{equation*}
Therefore, by \cref{cor:ratio_form_robust_speedup}, the linear first-leaf
hitting guarantee is certified whenever, for fixed constants $c,c_\chi>0$,
\begin{equation*}
    \frac{\kappa^2-1}{\kappa^2+1}
    \le
    \min\left\{\frac{c}{n},\frac12\right\},
    \qquad
    1-\chi\le \frac{c_\chi}{n}.
\end{equation*}
Asymptotically, certifying this condition from $\kappa$ alone requires
$\kappa^2=1+O(1/n)$, equivalently $\kappa=1+O(1/n)$, together with
near-maximal cancellation.
\end{corollary}

\begin{proof}
Let $F(z)$ and $B(z)$ denote the plug-in forward and backward total masses at $z$. Under the exact value $V^*$, the canonical balanced choice gives $F^*(z)=B^*(z)$ by \cref{prop:balanced_directional_masses}. Under the multiplicative verifier condition, the plug-in forward total satisfies
\begin{equation*}
    \kappa^{-1}F^*(z)\le F(z)\le \kappa F^*(z),
\end{equation*}
because each child value $\widehat V(x,c)$ is within a factor $\kappa$ of $V^*(x,c)$. Similarly,
\begin{equation*}
    \kappa^{-1}B^*(z)\le B(z)\le \kappa B^*(z).
\end{equation*}
Since $F^*(z)=B^*(z)$, it follows that
\begin{equation*}
    \frac{1}{\kappa^2}\le \frac{F(z)}{B(z)}\le \kappa^2.
\end{equation*}
But
\begin{equation*}
    \frac{p_{\downarrow}(z)}{p_{\uparrow}(z)}
    =
    \frac{F(z)/(F(z)+B(z))}{B(z)/(F(z)+B(z))}
    =
    \frac{F(z)}{B(z)}.
\end{equation*}
Applying \cref{cor:ratio_form_robust_speedup} with $\rho_n=\kappa^2$ gives the stated sufficient condition. In particular, to certify the linear first-leaf hitting guarantee from this multiplicative bound alone, it is enough asymptotically that $\kappa^2=1+O(1/n)$, equivalently $\kappa=1+O(1/n)$, together with $1-\chi=O(1/n)$.
\end{proof}

\section{Balanced Parallel VGBs for Masked Diffusion}
\label{app:sec:balanced_mdm_vgb}

The previous AOAR sections study single-site reveal and re-mask moves. In this section, we extend the same VGB principle to parallel masked-diffusion style updates, motivated by discrete and masked diffusion language models \citep{austin2021structured,sahoo2024simple}. The resulting MDM-VGB family allows a move to reveal or re-mask multiple coordinates at once.

The construction proceeds in three steps. First, we define a leaf-corrected balanced MDM-VGB, which is the direct block-update analogue of Balanced AOAR-VGB. Second, we introduce a clean geometric balanced MDM-VGB whose backward remasking rule is value-gated while avoiding any explicit $M_{\mathrm{ref}}$ correction in the local implementation. Third, we add a unified flow-cancelled momentum framework that applies to either base kernel once its local forward and backward weights have been specified.

In particular, AOAR-VGB is recovered by the singleton-block restriction
$\calM=\mbra{1}$. Under this restriction, the MDM algorithms below also serve
as the implementation templates for the AOAR samplers.

\subsection{Balanced MDM-VGB}
\label{app:subsec:balanced_mdm_vgb}

\paragraph{Setup and notation.}
\label{app:subsubsec:balanced_mdm_setup}
We reuse the shared notation from \cref{app:sec:shared_setup_verifier_training,app:sec:aoar_vgb}. In particular, $x$ is the conditioning context, $\calY=\calV^n$ is the set of fully revealed sequences, $\calZ=\sbra{\calV\cup\mbra{\mask}}^n$ is the masked state space, $R(z)$ is the set of revealed coordinates, $k(z)=\abs{R(z)}$ is the depth of $z$, and $\calC(z)$ is the compatible completion set. Recall that
\begin{equation*}
    U^*(x,z)=M_{\mathrm{ref}}(x,z)V^*(x,z)
\end{equation*}
is the tilted completion mass, and recall the reachable masked-state set
\begin{equation*}
    \calZ_{\mathrm{reach}}:=\mbra{z\in\calZ: M_{\mathrm{ref}}(x,z)>0}.
\end{equation*}

Let $\calM\subseteq\mbra{1,\dots,n}$ be a nonempty set of admissible update sizes. We assume $1\in\calM$, so every non-root state has at least one possible backward update and every non-leaf state has at least one possible forward update.

For a reachable state $z$ and a block size $r\in\calM$, define
\begin{align*}
    \calB_r^+(z)&:=\mbra{A\subseteq\lbra{n}\setminus R(z): \abs{A}=r},\qquad
    \calB_r^-(z):=\mbra{A\subseteq R(z): \abs{A}=r}.
\end{align*}
For a forward block $A\in\calB_r^+(z)$ and assignment $a_A\in\calV^A$, write $z^{A\leftarrow a_A}$ for the state obtained by filling all coordinates in $A$ with $a_A$. For a backward block $A\in\calB_r^-(z)$, write $z^{-A}$ for the state obtained by re-masking all coordinates in $A$.

The MDM oracle assumption is that, for every reachable $z$, every $A\subseteq\lbra{n}\setminus R(z)$, and every $a_A\in\calV^A$, we can evaluate the joint reference conditional
\begin{equation*}
    \base(Y_A=a_A\mid x,z):=\bbP_{\base}\lbra{Y_A=a_A\mid x,\ Y\in\calC(z)}.
\end{equation*}
Equivalently,
\begin{equation*}
    \base(Y_A=a_A\mid x,z)=\frac{M_{\mathrm{ref}}(x,z^{A\leftarrow a_A})}{M_{\mathrm{ref}}(x,z)}.
\end{equation*}

\begin{remark}[Practical joint conditionals]
In a practical masked diffusion model, the exact joint conditional $\base(Y_A=a_A\mid x,z)$ may not be directly available. One may replace it by a factorized approximation, for example
\begin{equation*}
    \widetilde\base(Y_A=a_A\mid x,z):=\prod_{i\in A}\base(Y_i=a_i\mid x,z).
\end{equation*}
This factorized plug-in is a practical approximation: the exact stationarity statements below use the exact joint conditional.
\end{remark}

\paragraph{Block identities.}

We first collect the block analogues of the single-coordinate partition and Bellman identities.

\begin{lemma}[Block partition identity]
\label{lem:balanced_mdm_partition}
Fix a reachable state $z$ and a block $A\subseteq\lbra{n}\setminus R(z)$. Then
\begin{equation*}
    \calC(z)=\bigsqcup_{a_A\in\calV^A}\calC(z^{A\leftarrow a_A}).
\end{equation*}
Consequently,
\begin{equation*}
    \sum_{a_A\in\calV^A}U^*(x,z^{A\leftarrow a_A})=U^*(x,z),\qquad
    \sum_{a_A\in\calV^A}M_{\mathrm{ref}}(x,z^{A\leftarrow a_A})=M_{\mathrm{ref}}(x,z).
\end{equation*}
\end{lemma}

\begin{proof}
Every full sequence $y\in\calC(z)$ has a unique restriction $y_A\in\calV^A$. Hence $y$ belongs to exactly one completion set $\calC(z^{A\leftarrow a_A})$, namely the one with $a_A=y_A$. Summing $\base(y\mid x)\tau(x,y)$ over this disjoint partition gives the $U^*$ identity. Setting $\tau\equiv1$ gives the $M_{\mathrm{ref}}$ identity.
\end{proof}

\begin{proposition}[Block Bellman identity]
\label{prop:balanced_mdm_bellman}
For every reachable state $z$ and every block $A\subseteq\lbra{n}\setminus R(z)$,
\begin{equation*}
    V^*(x,z)=\sum_{a_A\in\calV^A}\base(Y_A=a_A\mid x,z)V^*(x,z^{A\leftarrow a_A}).
\end{equation*}
\end{proposition}

\begin{proof}
By the law of total expectation,
\begin{align*}
    V^*(x,z)
    &=\bbE_{\base}\lbra{\tau(x,Y)\mid x,\ Y\in\calC(z)}\\
    &=\sum_{a_A\in\calV^A}\bbP_{\base}\lbra{Y_A=a_A\mid x,z}\,
    \bbE_{\base}\lbra{\tau(x,Y)\mid x,\ Y\in\calC(z),\ Y_A=a_A}.
\end{align*}
The event $\mbra{Y\in\calC(z),Y_A=a_A}$ is exactly $\mbra{Y\in\calC(z^{A\leftarrow a_A})}$. Therefore the inner conditional expectation equals $V^*(x,z^{A\leftarrow a_A})$.
\end{proof}

\begin{proposition}[Exact block target conditional]
\label{prop:balanced_mdm_exact_conditional}
For every reachable state $z$ with $U^*(x,z)>0$, every block $A\subseteq\lbra{n}\setminus R(z)$, and every $a_A\in\calV^A$,
\begin{equation*}
    \pi^*(Y_A=a_A\mid x,z)
    =\frac{U^*(x,z^{A\leftarrow a_A})}{U^*(x,z)}
    =\base(Y_A=a_A\mid x,z)\frac{V^*(x,z^{A\leftarrow a_A})}{V^*(x,z)}.
\end{equation*}
\end{proposition}

\begin{proof}
The first equality follows from conditional probability under the tilted target:
\begin{equation*}
    \pi^*(Y_A=a_A\mid x,z)
    =\frac{\bbP_{\pi^*}(Y\in\calC(z^{A\leftarrow a_A})\mid x)}{\bbP_{\pi^*}(Y\in\calC(z)\mid x)}
    =\frac{U^*(x,z^{A\leftarrow a_A})}{U^*(x,z)}.
\end{equation*}
For the second equality, use $U^*(x,z)=M_{\mathrm{ref}}(x,z)V^*(x,z)$ and
\begin{equation*}
    M_{\mathrm{ref}}(x,z^{A\leftarrow a_A})=M_{\mathrm{ref}}(x,z)\base(Y_A=a_A\mid x,z).
\end{equation*}
\end{proof}

\paragraph{Leaf-corrected balanced MDM edge weights.}
For each update size $r\in\calM$, define the depth-size edge coefficient
\begin{equation*}
\label{eq:balanced_mdm_skr}
    s_{k,r}:=\frac{1}{\binom{n-r}{k}},\qquad 0\le k\le n-r.
\end{equation*}
Then
\begin{equation*}
\label{eq:balanced_mdm_skr_identity}
    \binom{n-k}{r}s_{k,r}=\binom{k}{r}s_{k-r,r}
\end{equation*}
whenever $r\le k\le n-r$. This is the multi-token analogue of the single-site balanced coefficient $s_k=1/\binom{n-1}{k}$.

For a reachable state $z$, define the MDM child, parent, and full neighborhoods by
\begin{align*}
    C_{\mathrm{MDM}}(z)&:=\mbra{z^{A\leftarrow a_A}: r\in\calM,\ A\in\calB_r^+(z),\ a_A\in\calV^A,\ z^{A\leftarrow a_A}\in\calZ_{\mathrm{reach}}},\\
    P_{\mathrm{MDM}}(z)&:=\mbra{z^{-A}: r\in\calM,\ A\in\calB_r^-(z)},\qquad
    N_{\mathrm{MDM}}(z):=C_{\mathrm{MDM}}(z)\cup P_{\mathrm{MDM}}(z).
\end{align*}

\begin{definition}[Leaf-corrected balanced MDM edge weights]
\label{def:balanced_mdm_edge_weights}
Consider a parent state $z$ of depth $k=k(z)$ and a child $u=z^{A\leftarrow a_A}$, where $A\in\calB_r^+(z)$ and $r=\abs{A}$. Thus $u$ has depth $k+r$. Define the symmetric edge weight
\begin{equation*}
    f_{\mathrm{MDM}}(z,u):=
    \begin{cases}
        s_{k,r}\,M_{\mathrm{ref}}(x,u)\widehat V(x,u), & u\notin\calY,\\[1mm]
        s_{n-r,r}\,M_{\mathrm{ref}}(x,u)\tau(x,u), & u\in\calY.
    \end{cases}
\end{equation*}
Equivalently, using $\widehat U(x,u):=M_{\mathrm{ref}}(x,u)\widehat V(x,u)$ and $U^*(x,y)=M_{\mathrm{ref}}(x,y)\tau(x,y)$ for leaves,
\begin{equation*}
    f_{\mathrm{MDM}}(z,u)=
    \begin{cases}
        s_{k,r}\widehat U(x,u), & u\notin\calY,\\
        s_{n-r,r}U^*(x,u), & u\in\calY.
    \end{cases}
\end{equation*}
The edge weight is understood symmetrically on the undirected block-update graph.
\end{definition}

\paragraph{Local implementation weights.}
The edge weights contain $M_{\mathrm{ref}}$, but the local implementation does not require estimating it. From a fixed current state $z$, all incident edge weights share the common factor $M_{\mathrm{ref}}(x,z)$, which cancels after normalization.

\begin{definition}[$\gamma$-held balanced MDM-VGB local kernel]
\label{def:balanced_mdm_local_weights}
Let $z$ be a reachable state of depth $k=k(z)$. We use terminal anchoring on leaves,
\begin{equation*}
    \widehat V(x,y)=\tau(x,y),\qquad y\in\calY.
\end{equation*}
Define the local implementation weights by
\begin{align*}
    \textsc{Forward:}\quad
    w_{\mathrm{MDM,fwd}}(z\to u)&:=s_{k,r}\base(Y_A=a_A\mid x,z)\widehat V(x,u),\\
    \textsc{Backward:}\quad
    w_{\mathrm{MDM,bwd}}(z\to u)&:=s_{k-r,r}\widehat V(x,z),
\end{align*}
where
\begin{equation*}
    \begin{aligned}
        &u=z^{A\leftarrow a_A},\quad A\in\calB_r^+(z) &&\text{for forward moves},\\
        &u=z^{-A},\quad A\in\calB_r^-(z) &&\text{for backward moves}.
    \end{aligned}
\end{equation*}
Let $w_{\mathrm{MDM}}(z\to u)$ denote the corresponding forward or backward weight, and set
\begin{equation*}
    W_{\mathrm{MDM}}(z):=
    \sum_{u\in C_{\mathrm{MDM}}(z)}w_{\mathrm{MDM,fwd}}(z\to u)
    +
    \sum_{u\in P_{\mathrm{MDM}}(z)}w_{\mathrm{MDM,bwd}}(z\to u).
\end{equation*}
The non-lazy move kernel is
\begin{equation*}
    K_{\mathrm{MDM}}(z,u):=\frac{w_{\mathrm{MDM}}(z\to u)}{W_{\mathrm{MDM}}(z)},\qquad u\in N_{\mathrm{MDM}}(z).
\end{equation*}
This formula is used only at states with \(W_{\mathrm{MDM}}(z)>0\).
The $\gamma$-held kernel is obtained by setting
\begin{equation*}
    P_{\mathrm{MDM}}^{(\gamma)}(z,z)=\frac{\gamma}{1+\gamma},\qquad
    P_{\mathrm{MDM}}^{(\gamma)}(z,u)=\frac{1}{1+\gamma}K_{\mathrm{MDM}}(z,u).
\end{equation*}
\end{definition}

\begin{theorem}[Weighted-graph representation and exact leaf law]
\label{thm:balanced_mdm_weighted_graph}
The local kernel in \cref{def:balanced_mdm_local_weights} is the weighted random walk on the graph with symmetric edge weights from \cref{def:balanced_mdm_edge_weights}. Therefore, its stationary distribution is
\begin{equation*}
    \mu_{\mathrm{MDM}}(z)=\frac{D_{\mathrm{MDM}}(z)}{\sum_{z'}D_{\mathrm{MDM}}(z')},\qquad
    D_{\mathrm{MDM}}(z):=\sum_{u\sim z}f_{\mathrm{MDM}}(z,u).
\end{equation*}
Moreover, the leaf-conditioned stationary law is exact:
\begin{equation*}
    \mu_{\mathrm{MDM}}(y\mid y\in\calY)=\pi^*(y\mid x).
\end{equation*}
\end{theorem}

\begin{proof}
Fix a state $z$ and a forward child $u=z^{A\leftarrow a_A}$ of size $r$. If $u\notin\calY$, then
\begin{align*}
    \frac{f_{\mathrm{MDM}}(z,u)}{M_{\mathrm{ref}}(x,z)}
    &=s_{k(z),r}\frac{M_{\mathrm{ref}}(x,u)}{M_{\mathrm{ref}}(x,z)}\widehat V(x,u)\\
    &=s_{k(z),r}\base(Y_A=a_A\mid x,z)\widehat V(x,u),
\end{align*}
which is the forward local weight. If $u\in\calY$, the same calculation gives
\begin{equation*}
    \frac{f_{\mathrm{MDM}}(z,u)}{M_{\mathrm{ref}}(x,z)}=s_{n-r,r}\base(Y_A=a_A\mid x,z)\tau(x,u).
\end{equation*}
For a backward move from $z$ to $u=z^{-A}$, the edge is the same edge with child endpoint $z$. Dividing its edge weight by $M_{\mathrm{ref}}(x,z)$ gives $s_{k(z)-r,r}\widehat V(x,z)$ if $z\notin\calY$, and $s_{n-r,r}\tau(x,z)$ if $z\in\calY$. Thus the local weights are proportional to incident edge weights by a state-dependent common factor, which cancels under normalization.

The stationary law of a random walk on a symmetric weighted graph is proportional to weighted degree. It remains to check the leaf law. For any leaf $y\in\calY$, every admissible backward block $A\subseteq\lbra{n}$ of size $r\in\calM$ gives a parent $y^{-A}$, and the corresponding leaf edge has weight
\begin{equation*}
    f_{\mathrm{MDM}}(y^{-A},y)=s_{n-r,r}U^*(x,y).
\end{equation*}
Therefore
\begin{equation*}
    D_{\mathrm{MDM}}(y)=U^*(x,y)\sum_{\substack{r\in\calM\\r\le n}}\binom{n}{r}s_{n-r,r}.
\end{equation*}
The multiplicative factor is constant over leaves. Since $U^*(x,y)=\base(y\mid x)\tau(x,y)$, conditioning on $\calY$ gives
\begin{equation*}
    \mu_{\mathrm{MDM}}(y\mid y\in\calY)=\frac{\base(y\mid x)\tau(x,y)}{\sum_{y'\in\calY}\base(y'\mid x)\tau(x,y')}=\pi^*(y\mid x).
\end{equation*}
\end{proof}

\begin{proposition}[Exact directional balance under exact values]
\label{prop:balanced_mdm_directional_balance}
Assume $\widehat V=V^*$ on all non-leaf states and terminal anchoring on leaves. Let $z$ be an interior state of depth $k$. For a fixed block size $r\in\calM$ with $r\le\min\mbra{k,n-k}$, the exact forward and backward masses of size $r$ satisfy
\begin{equation*}
    W_{\mathrm{fwd},r}^*(z)=W_{\mathrm{bwd},r}^*(z).
\end{equation*}
Consequently, if at an interior state only paired sizes
\begin{equation*}
    \calM_{\mathrm{pair}}(z):=\mbra{r\in\calM:r\le\min\mbra{k(z),n-k(z)}}
\end{equation*}
are used, then the exact move-conditioned forward and backward probabilities are both $1/2$.
\end{proposition}

\begin{proof}
Fix $r\le\min\mbra{k,n-k}$. Using \cref{prop:balanced_mdm_bellman}, for each $A\in\calB_r^+(z)$,
\begin{equation*}
    \sum_{a_A\in\calV^A}\base(Y_A=a_A\mid x,z)V^*(x,z^{A\leftarrow a_A})=V^*(x,z).
\end{equation*}
There are $\binom{n-k}{r}$ forward blocks of size $r$, so
\begin{equation*}
    W_{\mathrm{fwd},r}^*(z)=\binom{n-k}{r}s_{k,r}V^*(x,z).
\end{equation*}
For backward blocks of size $r$, there are $\binom{k}{r}$ choices, and each backward weight is $s_{k-r,r}V^*(x,z)$. Hence
\begin{equation*}
    W_{\mathrm{bwd},r}^*(z)=\binom{k}{r}s_{k-r,r}V^*(x,z).
\end{equation*}
The coefficient identity in \cref{eq:balanced_mdm_skr_identity} gives equality. Summing over paired sizes preserves equality.
\end{proof}

\subsection{Geometric Balanced MDM-VGB}
\label{app:subsec:balanced_geometric_mdm_vgb}

Balanced MDM-VGB gives the block-update analogue of the balanced AOAR walk. We now add the same value-gated remasking idea used in the geometric AOAR variant. The goal is to preserve the forward VGB ranking while making the backward block choice depend on the verifier value of the remasked parent.

\cref{fig:geometric_mdm_vgb_transition} shows the corresponding block-update transition. The only change from the AOAR case is that each reveal or re-mask candidate is now an admissible block $A\in\calM$, with joint reference conditionals used on forward block reveals.

\begin{figure}[t]
    \centering
    \vspace{-0.5em}
    \includegraphics[width=0.95\linewidth]{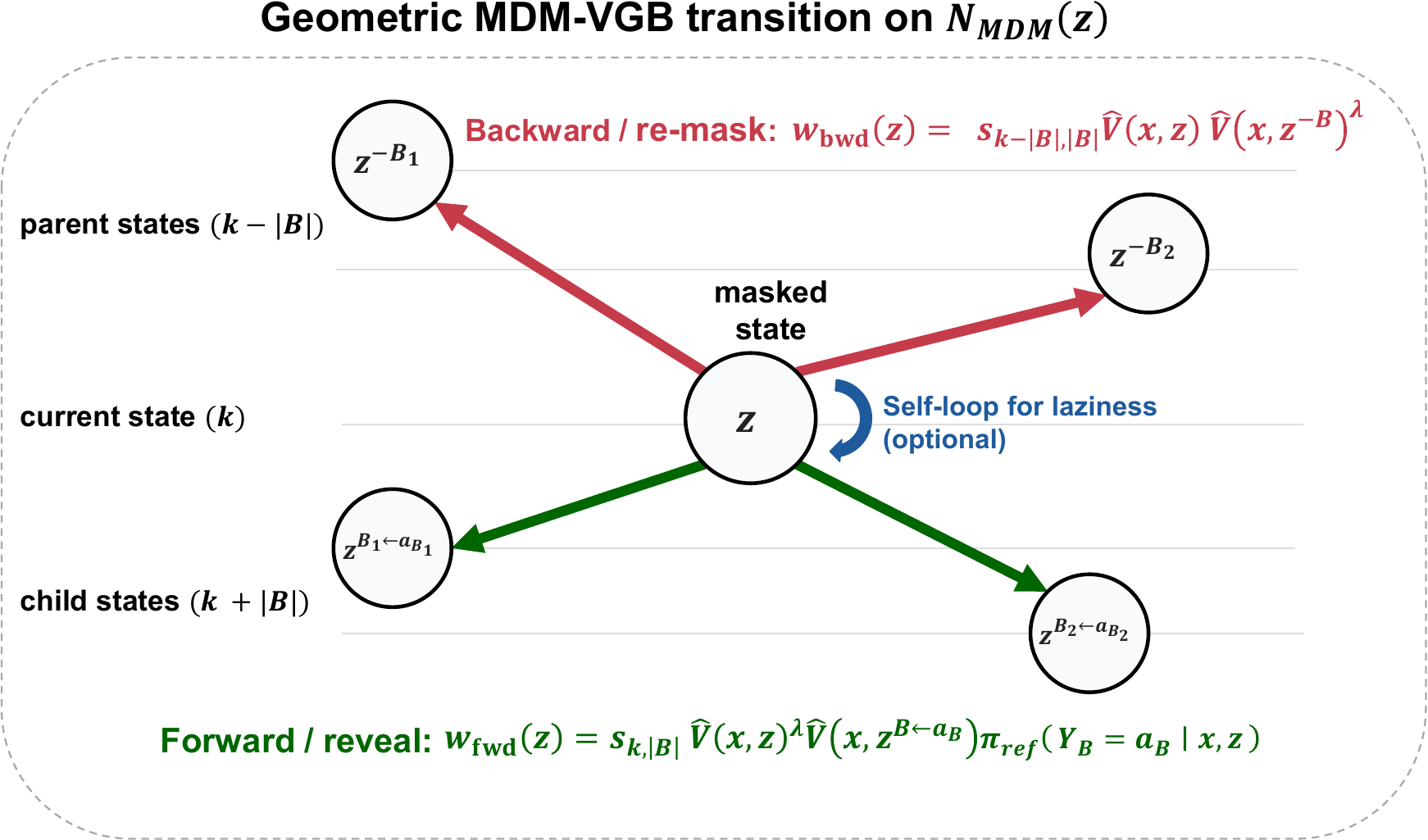}
    \caption{
        Illustration of one geometric balanced MDM-VGB transition.
        Forward moves reveal an admissible block using the joint reference conditional and the child verifier value, while backward moves re-mask an admissible block using the verifier value of the remasked parent.
        Setting $\lambda=0$ recovers balanced MDM-VGB.
    }
    \label{fig:geometric_mdm_vgb_transition}
    \vspace{-0.5em}
\end{figure}

\begin{definition}[Clean geometric balanced MDM edge weights]
\label{def:bg_mdm_edge_weights}
Fix $\lambda\ge0$. For every admissible size $r\in\calM$, let
\begin{equation*}
    s_{k,r}:=\binom{n-r}{k}^{-1},\qquad 0\le k\le n-r,
\end{equation*}
be the canonical balanced depth-size coefficient. Consider neighboring states $z,u$ with $\abs{k(u)-k(z)}=r$. Let $p(z,u)$ be the lower-depth endpoint and $c(z,u)$ be the higher-depth endpoint, so that
\begin{equation*}
    c(z,u)=p(z,u)^{A\leftarrow a_A}
\end{equation*}
for a unique block $A$ of size $r$ and assignment $a_A$.

If $c(z,u)\notin\calY$, define
\begin{equation*}
    f_{\mathrm{BGM}}(z,u)=f_{\mathrm{BGM}}(u,z):=
    s_{k(p(z,u)),r}\widehat U(x,c(z,u))\widehat V(x,p(z,u))^\lambda.
\end{equation*}
If $c(z,u)=y\in\calY$, define the leaf-corrected edge by
\begin{equation*}
    f_{\mathrm{BGM}}(z,u)=f_{\mathrm{BGM}}(u,z):=s_{n-r,r}U^*(x,y).
\end{equation*}
\end{definition}

\paragraph{Local implementation weights.}
Let $z$ have depth $k$. For a non-leaf forward child $u=z^{A\leftarrow a_A}\notin\calY$ with $\abs{A}=r$,
\begin{equation*}
\label{eq:bg_mdm_fwd_weight}
    w_{\mathrm{BGM,fwd}}(z\to u)=s_{k,r}\widehat V(x,z)^\lambda\base(Y_A=a_A\mid x,z)\widehat V(x,u).
\end{equation*}
The factor $\widehat V(x,z)^\lambda$ is common across all forward candidates from the same current state $z$, so the forward ranking remains
\begin{equation*}
    \base(Y_A=a_A\mid x,z)\widehat V(x,z^{A\leftarrow a_A}).
\end{equation*}
For a backward move from a non-leaf state $z$ to $u=z^{-A}$,
\begin{equation*}
\label{eq:bg_mdm_bwd_weight}
    w_{\mathrm{BGM,bwd}}(z\to z^{-A})=s_{k-r,r}\widehat V(x,z)\widehat V(x,z^{-A})^\lambda.
\end{equation*}
Thus backward moves prefer blocks whose remasked parent has high verifier value.

At the leaf boundary, if $y=z^{A\leftarrow a_A}\in\calY$, then
\begin{equation*}
\label{eq:bg_mdm_leaf_fwd_weight}
    w_{\mathrm{BGM,fwd}}^{\mathrm{leaf}}(z\to y)=s_{n-r,r}\base(Y_A=a_A\mid x,z)\tau(x,y),
\end{equation*}
and from a leaf $y$ to a parent $y^{-A}$,
\begin{equation*}
\label{eq:bg_mdm_leaf_bwd_weight}
    w_{\mathrm{BGM,bwd}}(y\to y^{-A})=s_{n-r,r}\tau(x,y).
\end{equation*}

\begin{remark}[$\lambda=0$ recovers balanced MDM-VGB]
When $\lambda=0$, the geometric gate disappears, and the local weights in \cref{eq:bg_mdm_fwd_weight,eq:bg_mdm_bwd_weight} reduce to the balanced MDM-VGB weights in \cref{def:balanced_mdm_local_weights}.
\end{remark}

\begin{theorem}[Stationarity and exact leaf law for clean geometric MDM]
\label{thm:bg_mdm_stationarity}
Let $P_{\mathrm{BGM}}$ be the weighted random walk on the block-update graph with edge weights from \cref{def:bg_mdm_edge_weights}. Then
\begin{equation*}
    \mu_{\mathrm{BGM}}(z)=\frac{D_{\mathrm{BGM}}(z)}{\sum_{z'}D_{\mathrm{BGM}}(z')},\qquad
    D_{\mathrm{BGM}}(z):=\sum_{u\sim z}f_{\mathrm{BGM}}(z,u),
\end{equation*}
is stationary. Moreover,
\begin{equation*}
    \mu_{\mathrm{BGM}}(y\mid y\in\calY)=\pi^*(y\mid x).
\end{equation*}
\end{theorem}

\begin{proof}
The stationarity proof is the standard detailed-balance proof for a symmetric weighted graph. For neighboring states $z\sim u$,
\begin{equation*}
    \mu_{\mathrm{BGM}}(z)P_{\mathrm{BGM}}(z,u)
    =
    \frac{D_{\mathrm{BGM}}(z)}{\sum_{z'}D_{\mathrm{BGM}}(z')}
    \frac{f_{\mathrm{BGM}}(z,u)}{D_{\mathrm{BGM}}(z)}
    =
    \frac{f_{\mathrm{BGM}}(z,u)}{\sum_{z'}D_{\mathrm{BGM}}(z')}.
\end{equation*}
This expression is symmetric in $z$ and $u$, hence detailed balance holds.

Now take a leaf $y\in\calY$. For every admissible block size $r\in\calM$ and every block $A\subseteq\lbra{n}$ with $\abs{A}=r$, the parent is $y^{-A}$ and the leaf edge has weight
\begin{equation*}
    f_{\mathrm{BGM}}(y^{-A},y)=s_{n-r,r}U^*(x,y).
\end{equation*}
Therefore
\begin{equation*}
    D_{\mathrm{BGM}}(y)=U^*(x,y)\sum_{\substack{r\in\calM\\r\le n}}\binom{n}{r}s_{n-r,r}.
\end{equation*}
The multiplier is independent of $y$, so conditioning on leaves yields
\begin{equation*}
    \mu_{\mathrm{BGM}}(y\mid y\in\calY)=\frac{U^*(x,y)}{\sum_{y'\in\calY}U^*(x,y')}=\pi^*(y\mid x).
\end{equation*}
\end{proof}

\begin{proposition}[Exact block directional masses]
\label{prop:bg_mdm_directional_masses}
Assume $\widehat V=V^*$ on non-leaf states. Let $z$ have depth $k$ with $V^*(x,z)>0$, and fix an admissible size $r$ with $r\le k$ and $k+r<n$, so that the size-$r$ forward children are non-leaf. For a backward block $A\subseteq R(z)$ with $\abs{A}=r$, define
\begin{equation*}
    r_A^*(z):=\frac{V^*(x,z^{-A})}{V^*(x,z)}.
\end{equation*}
Then the exact size-$r$ forward and backward masses are
\begin{equation*}
    F_{\lambda,r}^*(z)=s_{k,r}\binom{n-k}{r}V^*(x,z)^{1+\lambda},
\end{equation*}
and
\begin{equation*}
    B_{\lambda,r}^*(z)=s_{k-r,r}V^*(x,z)^{1+\lambda}
    \sum_{\substack{A\subseteq R(z)\\\abs{A}=r}} r_A^*(z)^\lambda.
\end{equation*}
\end{proposition}

\begin{proof}
For each forward block $A\in\calB_r^+(z)$, the block Bellman identity gives
\begin{equation*}
    \sum_{a_A\in\calV^A}\base(Y_A=a_A\mid x,z)V^*(x,z^{A\leftarrow a_A})=V^*(x,z).
\end{equation*}
There are $\binom{n-k}{r}$ forward blocks, hence
\begin{equation*}
    F_{\lambda,r}^*(z)=s_{k,r}V^*(x,z)^\lambda\binom{n-k}{r}V^*(x,z)=s_{k,r}\binom{n-k}{r}V^*(x,z)^{1+\lambda}.
\end{equation*}
For backward blocks,
\begin{align*}
    B_{\lambda,r}^*(z)
    &=\sum_{\substack{A\subseteq R(z)\\\abs{A}=r}}s_{k-r,r}V^*(x,z)V^*(x,z^{-A})^\lambda\\
    &=s_{k-r,r}V^*(x,z)^{1+\lambda}\sum_{\substack{A\subseteq R(z)\\\abs{A}=r}}
    \left[\frac{V^*(x,z^{-A})}{V^*(x,z)}\right]^\lambda.
\end{align*}
\end{proof}

\begin{remark}[Canonical MDM depth anchor]
\label{rem:bg_mdm_canonical_anchor}
For the balanced MDM chain, the canonical coefficient
\begin{equation*}
    s_{k,r}^{\mathrm{bal}}=\frac{1}{\binom{n-r}{k}}
\end{equation*}
satisfies
\begin{equation*}
    \binom{n-k}{r}s_{k,r}^{\mathrm{bal}}=\binom{k}{r}s_{k-r,r}^{\mathrm{bal}}.
\end{equation*}
For the clean geometric MDM chain with $\lambda>0$, this coefficient no longer gives exact $1/2$--$1/2$ balance in general. Instead, it is the canonical depth-neutral anchor: it recovers exact balance at $\lambda=0$, while the $\lambda>0$ correction changes backward probabilities according to the verifier values of remasked parent states.
\end{remark}

The sampler in \cref{alg:mdm_vgb_sampler} specializes to the AOAR sampler
when $\calM=\mbra{1}$, since each admissible block is a singleton coordinate.

\subsection{Momentum Framework for Balanced and Geometric MDM-VGB}
\label{app:subsec:momentum_bg_mdm_vgb}
\label{app:subsec:momentum_balanced_mdm_vgb}
\label{app:subsubsec:balanced_mdm_momentum}

We now add momentum after both base MDM kernels have been defined. This separation is useful because the same lifted transition only needs the local forward and backward weights of the chosen base kernel.

\cref{fig:geometric_mdm_vgb_momentum_transition} illustrates the momentum construction for the geometric MDM base kernel. The same diagram also applies to balanced MDM-VGB after setting $\lambda=0$ in the local block weights.

\begin{figure}[t]
    \centering
    \vspace{-0.5em}
    \includegraphics[width=0.95\linewidth]{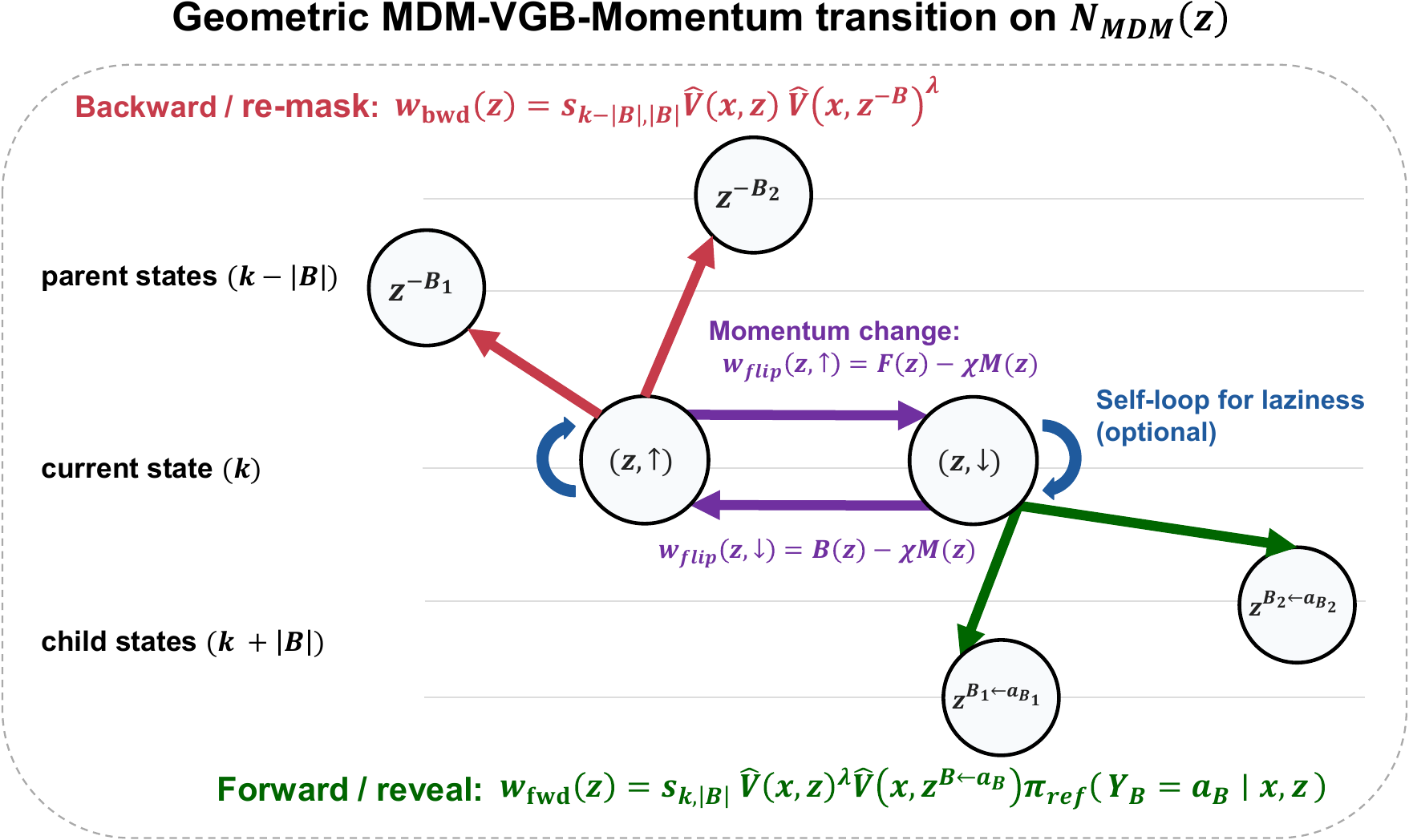}
    \caption{
        Illustration of the flow-cancelled momentum lift for geometric MDM-VGB.
        The downward copy keeps block reveal moves, the upward copy keeps block re-mask moves, and residual opposite-direction mass switches momentum; the cancelled common mass may be retained as an optional stay probability.
    }
    \label{fig:geometric_mdm_vgb_momentum_transition}
    \vspace{-0.5em}
\end{figure}

Let
\begin{equation*}
    \mathsf B\in\{\mathrm{MDM},\mathrm{BGM}\}
\end{equation*}
denote either the balanced MDM-VGB base kernel or the clean geometric balanced MDM-VGB base kernel. Let
\begin{equation*}
    w_{\mathsf B,\mathrm{fwd}}(z\to c),\qquad
    w_{\mathsf B,\mathrm{bwd}}(z\to p)
\end{equation*}
be its local forward and backward weights, as defined in \cref{def:balanced_mdm_local_weights} for $\mathsf B=\mathrm{MDM}$ and in \cref{eq:bg_mdm_fwd_weight,eq:bg_mdm_bwd_weight,eq:bg_mdm_leaf_fwd_weight,eq:bg_mdm_leaf_bwd_weight} for $\mathsf B=\mathrm{BGM}$.

Define the forward and backward total masses
\begin{equation*}
    F_{\mathsf B}(z):=\sum_{c\in C_{\mathrm{MDM}}(z)}w_{\mathsf B,\mathrm{fwd}}(z\to c),\qquad
    B_{\mathsf B}(z):=\sum_{p\in P_{\mathrm{MDM}}(z)}w_{\mathsf B,\mathrm{bwd}}(z\to p).
\end{equation*}
The non-lazy base move kernel is
\begin{equation*}
    K_{\mathsf B}(z,u):=\frac{w_{\mathsf B}(z\to u)}{F_{\mathsf B}(z)+B_{\mathsf B}(z)},\qquad u\in N_{\mathrm{MDM}}(z),
\end{equation*}
where $w_{\mathsf B}$ denotes the appropriate forward or backward local weight. Let $\mu_{\mathsf B}$ be its stationary law. By \cref{thm:balanced_mdm_weighted_graph} for $\mathsf B=\mathrm{MDM}$ and \cref{thm:bg_mdm_stationarity} for $\mathsf B=\mathrm{BGM}$, the projected leaf-conditioned law is exact:
\begin{equation*}
    \mu_{\mathsf B}(\cdot\mid\calY)=\pi^*(\cdot\mid x).
\end{equation*}
Let \(\calZ_{\mathsf B,+}\) be the positive-weight component of \(K_{\mathsf B}\), i.e., the component on which \(F_{\mathsf B}(z)+B_{\mathsf B}(z)>0\).

For \(z\in\calZ_{\mathsf B,+}\), define directional probabilities
\begin{equation*}
    p_{\downarrow}^{\mathsf B}(z):=\sum_{c\in C_{\mathrm{MDM}}(z)}K_{\mathsf B}(z,c)=\frac{F_{\mathsf B}(z)}{F_{\mathsf B}(z)+B_{\mathsf B}(z)},
\end{equation*}
and
\begin{equation*}
    p_{\uparrow}^{\mathsf B}(z):=\sum_{p\in P_{\mathrm{MDM}}(z)}K_{\mathsf B}(z,p)=\frac{B_{\mathsf B}(z)}{F_{\mathsf B}(z)+B_{\mathsf B}(z)}.
\end{equation*}
Thus $p_{\downarrow}^{\mathsf B}(z)+p_{\uparrow}^{\mathsf B}(z)=1$ whenever $F_{\mathsf B}(z)+B_{\mathsf B}(z)>0$.

\begin{definition}[Flow-cancelled momentum MDM kernel]
\label{def:generic_mdm_flow_cancelled_kernel}
\label{def:balanced_mdm_flow_cancelled_kernel}
Fix a base kernel $\mathsf B\in\{\mathrm{MDM},\mathrm{BGM}\}$ and a cancellation strength $\chi\in[0,1]$. Set
\begin{equation*}
    m_{\mathsf B}(z):=\min\mbra{p_{\downarrow}^{\mathsf B}(z),p_{\uparrow}^{\mathsf B}(z)}.
\end{equation*}
On the lifted state space
\begin{equation*}
    \widetilde{\calZ}_{\mathsf B}:=\calZ_{\mathsf B,+}\times\mbra{\downarrow,\uparrow},
\end{equation*}
define $\widetilde K_{\mathsf B}^{\mathrm{fc},\chi}$ as follows.

\paragraph{Downward copy.}
From $(z,\downarrow)$, the chain moves to a child while keeping downward momentum with probability
\begin{equation*}
    \widetilde K_{\mathsf B}^{\mathrm{fc},\chi}\bigl((z,\downarrow),(c,\downarrow)\bigr):=K_{\mathsf B}(z,c),\qquad c\in C_{\mathrm{MDM}}(z).
    \tag{\textsc{Child Move}}
\end{equation*}
It switches to upward momentum with residual opposite-direction probability
\begin{equation*}
    \widetilde K_{\mathsf B}^{\mathrm{fc},\chi}\bigl((z,\downarrow),(z,\uparrow)\bigr):=p_{\uparrow}^{\mathsf B}(z)-\chi m_{\mathsf B}(z),
    \tag{\textsc{Switch Up}}
\end{equation*}
and stays in the same lifted state with cancelled crossing mass
\begin{equation*}
    \widetilde K_{\mathsf B}^{\mathrm{fc},\chi}\bigl((z,\downarrow),(z,\downarrow)\bigr):=\chi m_{\mathsf B}(z).
    \tag{\textsc{Stay Down, optional}}
\end{equation*}

\paragraph{Upward copy.}
From $(z,\uparrow)$, the chain moves to a parent while keeping upward momentum with probability
\begin{equation*}
    \widetilde K_{\mathsf B}^{\mathrm{fc},\chi}\bigl((z,\uparrow),(p,\uparrow)\bigr):=K_{\mathsf B}(z,p),\qquad p\in P_{\mathrm{MDM}}(z).
    \tag{\textsc{Parent Move}}
\end{equation*}
It switches to downward momentum with residual opposite-direction probability
\begin{equation*}
    \widetilde K_{\mathsf B}^{\mathrm{fc},\chi}\bigl((z,\uparrow),(z,\downarrow)\bigr):=p_{\downarrow}^{\mathsf B}(z)-\chi m_{\mathsf B}(z),
    \tag{\textsc{Switch Down}}
\end{equation*}
and stays in the same lifted state with cancelled crossing mass
\begin{equation*}
    \widetilde K_{\mathsf B}^{\mathrm{fc},\chi}\bigl((z,\uparrow),(z,\uparrow)\bigr):=\chi m_{\mathsf B}(z).
    \tag{\textsc{Stay Up, optional}}
\end{equation*}
All other transition probabilities are zero.
\end{definition}

\begin{proposition}[Markov property]
\label{prop:generic_mdm_momentum_markov}
\label{prop:balanced_mdm_momentum_markov}
For every $\chi\in[0,1]$ and every $\mathsf B\in\{\mathrm{MDM},\mathrm{BGM}\}$, $\widetilde K_{\mathsf B}^{\mathrm{fc},\chi}$ is a Markov kernel.
\end{proposition}

\begin{proof}
All transition probabilities are nonnegative because $m_{\mathsf B}(z)\le p_{\downarrow}^{\mathsf B}(z)$ and $m_{\mathsf B}(z)\le p_{\uparrow}^{\mathsf B}(z)$. For the downward copy,
\begin{align*}
    &\sum_{c\in C_{\mathrm{MDM}}(z)}\widetilde K_{\mathsf B}^{\mathrm{fc},\chi}\bigl((z,\downarrow),(c,\downarrow)\bigr)
    +\widetilde K_{\mathsf B}^{\mathrm{fc},\chi}\bigl((z,\downarrow),(z,\uparrow)\bigr)
    +\widetilde K_{\mathsf B}^{\mathrm{fc},\chi}\bigl((z,\downarrow),(z,\downarrow)\bigr)\\
    &\qquad =p_{\downarrow}^{\mathsf B}(z)+\sbra{p_{\uparrow}^{\mathsf B}(z)-\chi m_{\mathsf B}(z)}+\chi m_{\mathsf B}(z)=1.
\end{align*}
The upward row is identical.
\end{proof}

\begin{theorem}[Momentum preserves the MDM leaf law]
\label{thm:generic_mdm_momentum_stationarity}
\label{thm:balanced_mdm_momentum_stationarity}
Fix $\mathsf B\in\{\mathrm{MDM},\mathrm{BGM}\}$ and $\chi\in[0,1]$. Define
\begin{equation*}
    \widetilde\mu_{\mathsf B}(z,\downarrow)=\widetilde\mu_{\mathsf B}(z,\uparrow):=\frac12\mu_{\mathsf B}(z).
\end{equation*}
Then $\widetilde\mu_{\mathsf B}$ is stationary for $\widetilde K_{\mathsf B}^{\mathrm{fc},\chi}$. Moreover,
\begin{equation*}
    \textsf{proj}_{\#}\widetilde\mu_{\mathsf B}=\mu_{\mathsf B},
\end{equation*}
and hence
\begin{equation*}
    \textsf{proj}_{\#}\widetilde\mu_{\mathsf B}(\cdot\mid\calY)=\pi^*(\cdot\mid x).
\end{equation*}
\end{theorem}

\begin{proof}
We verify stationarity for the downward copy; the upward copy is symmetric. Incoming flow into $(z,\downarrow)$ comes from three sources: parents of $z$ in downward mode, the switch from $(z,\uparrow)$, and the self-loop at $(z,\downarrow)$.

The incoming move-flow from parents is
\begin{align*}
    \frac12\sum_{u\in P_{\mathrm{MDM}}(z)}\mu_{\mathsf B}(u)K_{\mathsf B}(u,z)
    &=\frac12\mu_{\mathsf B}(z)\sum_{u\in P_{\mathrm{MDM}}(z)}K_{\mathsf B}(z,u)\\
    &=\frac12\mu_{\mathsf B}(z)p_{\uparrow}^{\mathsf B}(z),
\end{align*}
where reversibility of $K_{\mathsf B}$ was used. The switch-flow from $(z,\uparrow)$ is
\begin{equation*}
    \frac12\mu_{\mathsf B}(z)\sbra{p_{\downarrow}^{\mathsf B}(z)-\chi m_{\mathsf B}(z)},
\end{equation*}
and the self-loop flow is
\begin{equation*}
    \frac12\mu_{\mathsf B}(z)\chi m_{\mathsf B}(z).
\end{equation*}
The total incoming flow is therefore
\begin{equation*}
    \frac12\mu_{\mathsf B}(z)\sbra{p_{\uparrow}^{\mathsf B}(z)+p_{\downarrow}^{\mathsf B}(z)}
    =\frac12\mu_{\mathsf B}(z)=\widetilde\mu_{\mathsf B}(z,\downarrow).
\end{equation*}
The upward-copy argument is identical after exchanging parents and children. Projection gives
\begin{equation*}
    \widetilde\mu_{\mathsf B}(z,\downarrow)+\widetilde\mu_{\mathsf B}(z,\uparrow)=\mu_{\mathsf B}(z).
\end{equation*}
The leaf law follows from the corresponding base-chain leaf law.
\end{proof}

The sampler in \cref{alg:momentum_mdm_sampler} also specializes to the
momentum AOAR sampler when $\calM=\mbra{1}$.

\section{Experimental Setting and Supplementary Results}
\label{app:sec:experiment_settings_and_supplementary_results}
\label{app:exp_details}

This section gives the full task definitions, data sources, backbone models, reward functions, root-start and leaf-start protocols, and hyperparameter grids used in our experiments. Across tasks, the reference model is denoted by $\base$ and the masked-state verifier by $\widehat V(x,z)$, following the notation in \cref{app:sec:key_notation_and_remarks}. For hard-constraint tasks, the terminal reward $\tau(x,y)$ is binary. For QM9 and DNA, we first define a real-valued terminal score $R(y)$ and then use the corresponding Pass@95 indicator as the terminal reward when a binary reward is needed. Protein reports Success@1\AA{} and motif RMSD, while its learned verifier uses the smoothed training target defined below. Unless otherwise stated, VGB uses terminal anchoring: at a complete leaf $y$, $\widehat V(x,y)$ is replaced by the corresponding exact terminal reward or score. Experiments were run in an Ubuntu PyTorch environment on NVIDIA RTX PRO 6000 Blackwell 96GB and NVIDIA H200 141GB GPUs.

\paragraph{Common evaluation protocol.}
For root-start sampling, every method starts from the fully masked state
$\varnothing$ and produces a completed sample under a fixed step-limit
budget. For leaf-start repair, every method starts from the same set of low-reward or invalid completed leaves $y_0$. All comparisons use the same reference model $\base$ and the same task-specific terminal score/reward convention. In matched-budget root-start comparisons, \textsc{BoN} uses $N$ independent full rollouts, while MDM-VGB variants use up to an $n\times N$ step budget before forward-only completion.

\paragraph{Adjusted NFE.}
Raw NFE counts realized reference-model forward calls. For tasks with learned verifiers, we additionally report FLOP-adjusted NFE to account for verifier scoring cost. Let $N_{\rm base}$ be the number of reference-model forward calls, $N_{\rm verif}$ the number of verifier evaluations, and $C_{\rm base}$ and $C_{\rm verif}$ the FLOPs of one base-model and verifier forward pass. We define
\begin{equation*}
    \mathrm{NFE}_{\rm adj}
    =
    N_{\rm base}
    +
    N_{\rm verif}\frac{C_{\rm verif}}{C_{\rm base}}.
\end{equation*}
Thus each verifier query is charged as its equivalent fraction of a base-model forward call. For task-provided exact checkers, this adjustment is not applied unless otherwise stated.

\paragraph{Pass@$\alpha$ normalization.}
For tasks with continuous rewards, let $R(y)$ denote the target task-specific score. To make success rates comparable across tasks, we convert the continuous score into a rare-event indicator using a fixed threshold from an independent \textsc{Base} calibration set. Specifically, for $\alpha\in(0,1)$, let $q_\alpha^{\rm Base}(R)$ be the $\alpha$-quantile of $R(y)$ on this calibration set, where $y\sim\base$. We define
\begin{equation*}
    \mathrm{Pass@\alpha}(y)
:=
    \begin{cases}
        \mathbf 1\{R(y)\ge q_{\alpha}^{\rm Base}(R)\},
        & \text{if higher is better},\\
        \mathbf 1\{R(y)\le q_{1-\alpha}^{\rm Base}(R)\},
        & \text{if lower is better}.
    \end{cases}
\end{equation*}
The reported $\mathrm{Pass@95}$ metric corresponds to $\alpha=0.95$.

\subsection{Task Descriptions}
\label{app:task_descriptions}

\paragraph{Dyck suffix-edit repair.}
The Dyck language is the canonical language of well-balanced brackets studied in formal-language theory \citep{chomsky1963algebraic}. It is simple enough to serve as a controlled toy problem, but it still requires hierarchical dependency tracking: for example, \texttt{[()]} is valid, whereas \texttt{([)]} has the right local symbols in the wrong nesting order. We use this structure to illustrate the difference between prefix-tree backtracking and any-order re-masking.

The terminal reward is binary: reward $1$ means that all opening brackets are properly paired with closing brackets of the same type and in the correct nested order; reward $0$ means that this well-balanced-bracket condition fails.

Each repair example is initialized from the fixed corrupted leaf template
\begin{equation*}
    \texttt{B((([(((([(((}\textcolor{red!70!black}{\texttt{XXXXXXXX}}\texttt{)))]))))])))E}.
\end{equation*}
Here \texttt{B} and \texttt{E} are special begin/end tokens. The prefix before the marked block is fixed and used as the prompt, while the
\textcolor{red!70!black}{$\mathsf{X}$}'s denote the initially corrupted span of
length eight. We generate 10,000 corrupted leaves by sampling random bracket strings for this span and retaining only examples with reward $0$. Thus every example starts from an invalid fully revealed leaf.

In the main suffix-edit protocol, all bracket positions from the first corrupted token through the token immediately before \texttt{E} are editable. The suffix is not manually restored: if a sampler re-masks or changes a suffix token, it must also repair that token before the exact Dyck checker accepts the sequence. Equivalently, the reward is the exact bracket-validity checker implemented as membership in the accepted terminal set. The base model is a 12.9M-parameter BERT-style masked language model, and the process verifier is a 3.24M-parameter token-state Transformer.

\paragraph{Letter avoidance.}
Letter is a simple lexical positive-control task. The fixed task prompt is
\begin{center}
\begin{tikzpicture}
\node[
    draw=black,
    rounded corners=3pt,
    fill=black!4,
    inner sep=6pt,
    text width=0.90\linewidth,
    align=left
] {
\textbf{Task prompt.}\\[0.25em]
\small\ttfamily
Generate a one-sentence story without using the letter `e':
};
\end{tikzpicture}
\end{center}
This is a lipogram-style constraint, where a text is deliberately written while avoiding a specified letter. The terminal reward is exact:
\begin{equation*}
    \tau(x,y)
    =
    \begin{cases}
        1.0, & y\text{ is non-empty and contains no `e'},\\
        0.0, & \text{otherwise}.
    \end{cases}
\end{equation*}
All Letter experiments use the Qwen3-0.6B-diffusion-mdlm-v0.1 reference model \citep{qwen3_06b_diffusion_mdlm}. For partial configurations, we use a heuristic process verifier that checks whether the revealed tokens currently contain the forbidden character. The primary metric is the avoidance rate. We additionally report a Qwen3-32B \citep{qwen2025qwen3} judge score on outputs satisfying the lexical constraint. The judge prompt below is used with the task prompt above and each candidate answer; the reported Ovl. score is the returned \texttt{overall} field.
\begin{center}
\begin{tikzpicture}
\node[
    draw=black,
    rounded corners=3pt,
    fill=black!4,
    inner sep=7pt,
    text width=0.94\linewidth,
    align=left
] {
\textbf{Qwen3-32B overall-score prompt:}\\[0.25em]
\rule{\linewidth}{0.35pt}\\[0.35em]
\small\ttfamily
You are evaluating a candidate answer to a writing task.\\
Score the candidate on four dimensions from 1 to 10, where 10 is best.\\
Dimensions:\\
1. naturalness: Does the text read like fluent natural English?\\
2. coherence: Is it internally coherent as a one-sentence story?\\
3. semantic\_plausibility: Does it make semantic sense rather than feeling\\
\hspace*{1.2em}like token salad or broken text?\\
4. overall: Overall quality as an answer to the task prompt.\\
Use the task prompt as context when judging overall quality.\\
Return ONLY a JSON object with integer fields naturalness, coherence,\\
semantic\_plausibility, overall.
};
\end{tikzpicture}
\end{center}

\paragraph{Sudoku.}
Sudoku is a structured constraint-satisfaction benchmark. Each puzzle is a flattened $9\times 9$ grid with fixed clues and editable empty cells. The output is a completed board. The terminal checker is exact: a board solves the puzzle if every row, column, and $3\times3$ box is a permutation of $\mbra{1,\ldots,9}$. The primary metric is exact solve rate. We also report the normalized mean violation count, shown as a percentage:
\begin{equation*}
    \mathrm{Viol}(y)
=
    \frac{1}{27}
    \sum_{U\in\mathcal U}
    \left(9-\left|\mathrm{unique}(y_U)\right|\right),
\end{equation*}
where $\mathcal U$ is the set of the $27$ Sudoku units: $9$ rows, $9$ columns, and $9$ boxes. Thus, for a fully filled board with fixed clues preserved,
\begin{equation*}
    \mathrm{Solve}(y)=1
    \quad\Longleftrightarrow\quad
    \mathrm{Viol}(y)=0.
\end{equation*}
We use the Sudoku data provided by \citet{kim2025fine}, with 48,000 training puzzles and 2,000 test puzzles. Our base model is a 28.6M-parameter DiT-MDM, i.e., a masked discrete diffusion model \citep{sahoo2024simple} with a Diffusion Transformer backbone \citep{peebles2023scalable}, trained on this data. For partial configurations, we use a heuristic process verifier that returns the indicator that no Sudoku constraint has yet been violated.

\paragraph{QM9 molecule generation.}
QM9 is a small-molecule generation benchmark over organic molecules with up to nine heavy atoms \citep{ramakrishnan2014quantum}. We use a split with 127,190 training molecules and 6,695 test molecules, and train a QM9 MDLM 92.4M base model with the MDLM objective using the UDLM-QM9 architecture. The terminal score \(R_{\rm QM9}\) uses QED \cite{bickerton2012quantifying}, a quantitative estimate of drug-likeness:
\begin{equation*}
R_{\rm QM9}(y) =
    \begin{cases}
    \mathrm{QED}(y), & y\text{ is a valid molecule},\\
    0, & y\text{ is invalid}.
    \end{cases}
\end{equation*}
Since higher QED is better, the terminal reward used for rare-success guidance and evaluation is $\mathrm{Pass@95}$ from the common definition above with $R(y)=R_{\rm QM9}(y)$. We also report mean QED over valid molecules, validity, and unique Pass@95 where appropriate.

\paragraph{DNA enhancer design.}
DNA is a regulatory-sequence design task. The reference model is D3LM \citep{yang2026d3lmdiscretednadiffusion}, a DNA diffusion language model initialized from Nucleotide Transformer \citep{dallatorre2025nucleotide}. The terminal scorer is DeepSTARR \citep{dealmeida2022deepstarr}, which predicts developmental and housekeeping enhancer activity from a $249$ bp DNA sequence. The terminal score uses the developmental enhancer activity:
\begin{equation*}
R_{\rm DNA}(y)=s_{\rm Dev}(y).
\end{equation*}
For verifier training, we use a smoothed target rather than the sparse Pass@95 indicator. Let
\begin{equation*}
    z_{\rm Dev}(y)
=
    \mathrm{clip}\left(
    \frac{s_{\rm Dev}(y)-\mu_{\rm Dev}}{\sigma_{\rm Dev}},
    -3,3
    \right),
    \qquad
    \tilde r_{\rm train}(y)=\exp(\eta z_{\rm Dev}(y)).
\end{equation*}
Here \(\mu_{\rm Dev}\) and \(\sigma_{\rm Dev}\) are computed from base-model calibration samples, and we use \(\eta=1\). This target is used only for process-verifier regression. Since higher developmental enhancer activity is better, the evaluation terminal reward is the Pass@95 indicator \(\mathbf 1\{s_{\rm Dev}(y)\ge q_{0.95}^{\rm Base}(s_{\rm Dev})\}\).

\paragraph{Protein motif scaffolding.}
Protein is a motif-scaffolding sequence-design task. The reference model is EvoDiff OADM 38M \citep{alamdari2023protein}. We use official EvoDiff motif-scaffolding task metadata whenever available: PDB code, motif ranges, fixed motif residues, scaffold length range, and generated motif positions. In our task, all reported Protein experiments use the 1YCR motif-scaffolding target. The generated sequence is folded with OmegaFold \citep{wu2022omegafold}. We compute motif \(C_\alpha\) root-mean-square deviation (RMSD) against the reference motif:
\begin{equation*}
r_{\rm motif}(y) =
    \mathrm{RMSD}_{C_\alpha}
    \left(
        \mathrm{Fold}(y)_{\mathcal M},
        \mathrm{PDB}_{\mathcal M}
    \right).
\end{equation*}
Here $\mathcal M$ denotes the motif residue positions, $\mathrm{Fold}(y)$ is the OmegaFold-predicted structure, $\mathrm{PDB}_{\mathcal M}$ contains the reference motif $C_\alpha$ coordinates, and RMSD is computed after optimal rigid alignment.
The hard success metric is
\begin{equation*}
    \text{Success@1\AA{}}(y)
=
    \mathbf 1\{r_{\rm motif}(y)\le 1.0\text{\AA}\}.
\end{equation*}
For verifier training, we use a smoothed RMSD-derived target
\begin{equation*}
    \tau_{\rm train}(y)
    =
    \exp\left[
    -\left(
    \frac{r_{\rm motif}(y)}{4.0\text{\AA}}
    \right)^2
    \right],
\end{equation*}
where \(r_{\rm motif}(y)\) is the motif \(C_\alpha\) RMSD. In implementation, the verifier regresses \(\log \tau_{\rm train}(y)\) for numerical stability, and we exponentiate its output at inference time to obtain the positive value used by MDM-VGB. This verifier-training target is distinct from the reported Protein metrics: Success@1\AA{} and mean motif RMSD.

\subsection{Leaf-start repair protocol}
\label{app:root_leaf_protocols}

For leaf-start repair, we choose bad but repairable leaves. This avoids turning the repair experiment into full regeneration. For the $\alpha$-sweep in \cref{fig:leaf_start_repair}, $\mathrm{Pass@\alpha}$ is defined by the same base-calibration rule as in the common evaluation protocol: higher-is-better scores use \(q_\alpha^{\rm Base}\), while lower-is-better scores use \(q_{1-\alpha}^{\rm Base}\). Each starting group contains completed samples with $\mathrm{Pass@95}(y_0)=0$ but $\mathrm{Pass@\alpha}(y_0)=1$. Thus larger $\alpha$ selects easier near-miss leaves that are closer to the rare-success boundary. The main-text leaf-start table entries average over \(\alpha\in\{60,65,\ldots,90\}\). Unless otherwise stated, the bad-leaf pools are summarized in \cref{tab:leaf_start_bad_leaf_pools}.
\begin{table}[h!]
\centering
\caption{Leaf-start bad-leaf pools for the $\alpha$-sweep.}
\label{tab:leaf_start_bad_leaf_pools}
\vspace{5pt}
\small
\setlength{\tabcolsep}{5pt}
\renewcommand{\arraystretch}{1.08}
\begin{tabular}{@{}ll@{}}
\toprule
Task & Bad-leaf pool \\
\midrule
QM9 & $y_0$ valid, $q_{\alpha}^{\rm Base}(\mathrm{QED})\le \mathrm{QED}(y_0)<q_{0.95}^{\rm Base}(\mathrm{QED})$ \\
DNA & $q_{\alpha}^{\rm Base}(R_{\rm DNA})\le R_{\rm DNA}(y_0)<q_{0.95}^{\rm Base}(R_{\rm DNA})$ \\
Protein & $q_{0.05}^{\rm Base}(r_{\rm motif})< r_{\rm motif}(y_0)\le q_{1-\alpha}^{\rm Base}(r_{\rm motif})$ \\
\bottomrule
\end{tabular}
\end{table}

\subsection{Verifier Architecture}
\label{app:verifier_training_details}

We use task-specific process verifiers. Letter and Sudoku use heuristic process verifiers with no learned parameters, while Dyck, QM9, DNA, and Protein use learned lightweight Transformer verifiers. The learned verifiers are trained by regression to the task-specific targets described in the task descriptions above.
\begin{itemize}[leftmargin=15pt]
    \item Dyck: a learned token-state Transformer process verifier ($3.24$M parameters).
    \item Letter: a heuristic process verifier that checks whether the revealed tokens currently contain the forbidden letter; it has no learned parameters.
    \item Sudoku: a heuristic process verifier that checks whether the revealed entries already violate row, column, box, or fixed-clue constraints; it has no learned parameters.
    \item QM9: a learned token Transformer process verifier over SMILES tokens ($3.24$M parameters).
    \item DNA: a learned Transformer process verifier over D3LM tokens ($0.97$M parameters).
    \item Protein: a learned Transformer process verifier over amino-acid tokens with motif/editable/mask features ($1.94$M parameters).
\end{itemize}

We summarize the main inference hyperparameters in \cref{tab:hyperparameter_grid}.

\begin{table}[h!]
\centering
\setlength{\tabcolsep}{3.0pt}
\renewcommand{\arraystretch}{1.12}
\begin{tabular}{l c c c c}
\toprule
\textbf{Task}
& \textbf{Length}
& $\boldsymbol{\abs{\mathcal V}}$
& $\boldsymbol{L_f,L_b}$
& $\boldsymbol{K}$ \\
\midrule
Dyck & $34$ & $8$ & $(8,8)$ & $8$ \\
Letter & $32$ & $151{,}670$ & $(8,8)$ & $8$ \\
Sudoku & $89$ & $11$ & $(8,8)$ & $8$ \\
QM9 & $32$ & $35$ & $(8,32)$ & $8$ \\
DNA & $48$ & $4107$ & $(8,8)$ & $32$ \\
Protein (1YCR) & $73$ & $20$ & $(8,8)$ & $20$ \\
\bottomrule
\end{tabular}
\vspace{5pt}
\caption{
Main inference hyperparameters. $L_f$ and $L_b$ are the numbers of shortlisted forward and backward coordinates or blocks. $K$ is the number of token candidates per forward coordinate.
}
\label{tab:hyperparameter_grid}
\end{table}

\subsection{Additional Ablations}
\label{app:ablation_details}

The main text focuses on geometric re-masking and verifier amortization. Here we collect additional ablations on shortlisting and the block size.

\paragraph{Shortlisting budget.}
We ablate the local shortlisting budget for \textsc{MDM-VGB-Momentum} by increasing $L_f=L_b$ while holding the task-specific token budget $K$ fixed. On QM9, increasing the shortlist from $1$ to $4$ substantially improves Pass@95 while also reducing realized search cost, and larger budgets give smaller additional gains. On DNA, Pass@95 peaks around $L_f=L_b=4$, whereas $L_f=L_b=8$ substantially increases adjusted NFE without improving quality. Thus the gains are not explained by simply evaluating an arbitrarily large local neighborhood; a moderate shortlist is sufficient for the verifier to expose useful moves.

\begin{figure}[h!]
\centering
\includegraphics[width=0.72\linewidth]{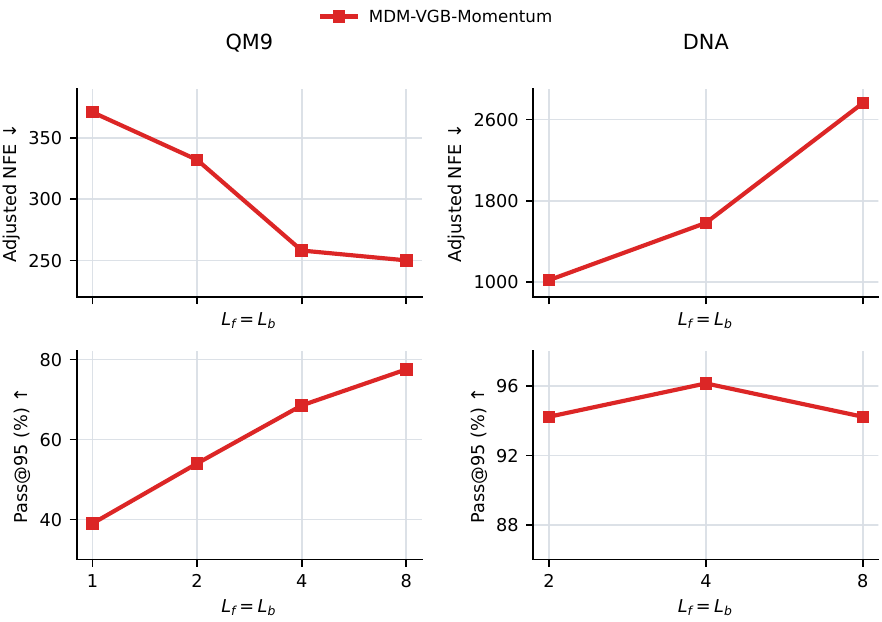}
\caption{
Shortlisting-budget ablation for \textsc{MDM-VGB-Momentum} on QM9 and DNA. Increasing $L_f=L_b$ helps up to a moderate budget, after which quality saturates while adjusted NFE can continue to grow.
}
\label{fig:shortlisting_budget_qm9_dna}
\end{figure}

\paragraph{Block size / MDM precision.}
We next study how MDM block size affects the finite-budget reward--compute tradeoff for \textsc{MDM-VGB-Momentum}. Singleton AOAR updates give the finest local control, while larger MDM blocks can reduce the number of sequential decisions by revealing or re-masking multiple coordinates at once. This creates a precision--parallelism tradeoff: moderate block sizes may amortize verifier and proposal costs, but overly coarse blocks can make the local edit decisions too blunt.

\begin{figure}[h!]
\centering
\includegraphics[width=0.62\linewidth]{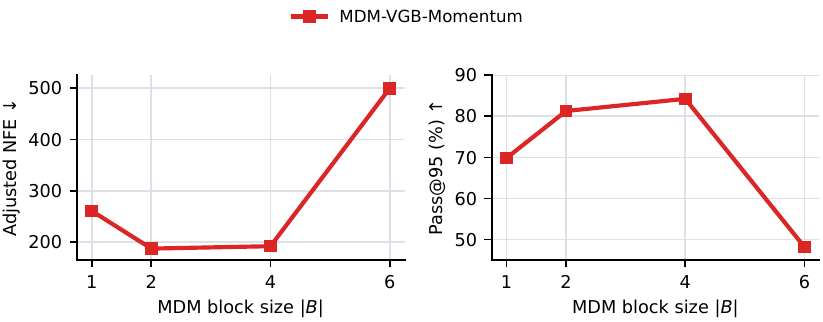}
\caption{
QM9 root-start \textsc{MDM-VGB-Momentum} block-size ablation under a maximum budget corresponding to $N=16$, with $L_f=L_b=4$, $K=4$, and $\lambda=4$. The left panel reports adjusted NFE, including verifier FLOPs, and the right panel reports Pass@95. Moderate block sizes improve the compute--quality tradeoff, whereas overly large blocks increase cost and degrade reward satisfaction.
}
\label{fig:qm9_mdm_block_sweep}
\end{figure}

\end{document}